\documentclass[10pt,journal,compsoc]{IEEEtran}

\ifCLASSOPTIONcompsoc
  \usepackage[nocompress]{cite}
\else
  \usepackage{cite}
\fi

\usepackage{amssymb}
\usepackage{amsmath}
\usepackage{amsthm}
\usepackage{cases}
\usepackage{algorithm}
\usepackage{algorithmic}
\usepackage{tabularx}
\usepackage{multirow}
\usepackage{graphicx}
\usepackage{subfigure}
\usepackage{longtable}
\usepackage{makecell}
\usepackage{booktabs}
\usepackage{color}

\newcommand{\Tr}{^{\rm \top}}

\newcommand{\tr}{{\rm Tr}}

\renewcommand{\a}{{\bf a}}

\newcommand{\f}{{\bf f}}

\newcommand{\x}{{\bf x}}

\newcommand{\A}{{\bf A}}
\newcommand{\B}{{\bf B}}
\newcommand{\C}{{\bf C}}
\newcommand{\D}{{\bf D}}

\newcommand{\E}{{\bf E}}

\newcommand{\F}{{\bf F}}
\newcommand{\G}{{\bf G}}

\renewcommand{\H}{{\bf H}}
\newcommand{\I}{{\bf I}}

\newcommand{\K}{{\bf K}}
\renewcommand{\L}{{\bf L}}
\newcommand{\M}{{\bf M}}

\renewcommand{\P}{{\bf P}}
\newcommand{\Q}{{\bf Q}}
\newcommand{\R}{{\bf R}}
\renewcommand{\S}{{\bf S}}

\newcommand{\T}{{\bf T}}

\newcommand{\U}{{\bf U}}

\newcommand{\V}{{\bf V}}

\newcommand{\W}{{\bf W}}

\newcommand{\Ocal}[1]{{\mathcal{O}\left( #1  \right)}}

\newcommand{\X}{{\bf X}}

\newcommand{\Y}{{\bf Y}}

\newcommand{\bSigma}{\boldsymbol{\Sigma}}

\newcommand{\bmu}{\boldsymbol{\mu}}

\newcommand{\1}{{\bf 1}}
\newcommand{\0}{{\bf 0}}

\newcommand{\lrincir}[1]{\left( #1 \right)}

\newcommand{\EE}{\mathop{\mathbb{E}}}
\newcommand{\RR}{\mathbb{R}}

\newtheorem{Definition}{\bf Definition}
\newtheorem{Theorem}{\bf Theorem}

\newtheorem{Proposition}{\it Proposition}

\begin{document}

\title{Multi-View Spectral Clustering with High-Order Optimal Neighborhood Laplacian Matrix}

\author{Weixuan Liang$^\ast$, Sihang Zhou$^\ast$, Jian Xiong, Xinwang Liu$^\dagger$, Siwei Wang, En Zhu$^\dagger$, Zhiping Cai, and Xin Xu
\IEEEcompsocitemizethanks{
\IEEEcompsocthanksitem $\ast$ equal contribution
\IEEEcompsocthanksitem $^\dagger$ corresponding author
\IEEEcompsocthanksitem Weixuan Liang, Xinwang Liu,  Siwei Wang, En Zhu, and Zhiping Cai are with School of Computer, National University of Defense Technology, Changsha, Hunan, 410073, China. E-mail: \{xinwangliu@nudt.edu.cn, enzhu@nudt.edu.cn\}.
\IEEEcompsocthanksitem  Sihang Zhou and Xin Xu are with College of Intelligence Science and Technology, National University of Defense Technology, Changsha, Hunan, 410073, China.
\IEEEcompsocthanksitem Jian Xiong is with School of Business Administration, Southwestern University of Finance and Economics, Chengdu, Sichuan, 611130, China.
}
}

\markboth{SUBMITTED TO IEEE TRANSACTIONS ON KNOWLEDGE AND DATA ENGINEERING, February~2020}%
{Shell \MakeLowercase{\textit{et al.}}: Bare Demo of IEEEtran.cls for IEEE Journals}
\IEEEcompsoctitleabstractindextext{%

\begin{abstract}
Multi-view spectral clustering can effectively reveal the intrinsic cluster structure among data by performing clustering on the learned optimal embedding across views. Though demonstrating promising performance in various applications, most of existing methods usually linearly combine a group of pre-specified first-order Laplacian matrices to construct the optimal Laplacian matrix, which may result in limited representation capability and insufficient information exploitation. Also, storing and implementing complex operations on the $n\times n$ Laplacian matrices incurs intensive storage and computation complexity. To address these issues, this paper first proposes a multi-view spectral clustering algorithm that learns a high-order optimal neighborhood Laplacian matrix, and then extends it to the late fusion version for accurate and efficient multi-view clustering. Specifically, our proposed algorithm generates the optimal Laplacian matrix by searching the neighborhood of the linear combination of both the first-order and high-order base Laplacian matrices simultaneously. By this way, the representative capacity of the learned optimal Laplacian matrix is enhanced, which is helpful to better utilize the hidden high-order connection information among data, leading to improved clustering performance. We design an efficient algorithm with proved convergence to solve the resultant optimization problem. Extensive experimental results on nine datasets demonstrate the superiority of our algorithm against state-of-the-art methods, which verifies the effectiveness and advantages of the proposed algorithm.
\end{abstract}

}
\maketitle

\IEEEpeerreviewmaketitle

\section{Introduction}
\IEEEPARstart{S}{pectral} clustering is an important technique that optimally learns the low-dimensional intrinsic embedding from the noisy high-dimensional data for clustering. In recent years, in the era of big data, integrate the diverse and complementary information from multiple views to further improve the effectiveness of the algorithm is becoming an increasingly attractive hotspot in this field \cite{de2005spectral,yang2018multi,YuYWD2020,WangWLG18}. As information represented by different views can be heterogeneous and biased, how to fully exploit the multi-view information and subtly fuse them to acquire a better overall vision of the whole sample set is one of the vital topics in the field of multi-view spectral clustering (MSC). According to the information fusion mechanism, the existing literature of MSC can be roughly dividing into three categories. The {first} category of method adopts a co-training mechanism to force the clustering results of different views to be consistent with each other \cite{kumar2011co,huang2015spectral,WangLY16}. The {second} category of method holds that the affinity matrix of each view is a perturbation of the optimal affinity matrix. Then, by conducting low-rank or sparse optimization, these algorithms extract an optimal consensus affinity matrix from all views \cite{nie2017self,zhan2019multiview,tang2018learning,zhou2019incremental}. By assuming that the optimal Laplacian matrix is a linear aggregation of the base Laplacian matrices, the {third} category of method optimizes the combination coefficients of the base Laplacian matrices by minimizing the normalized cut of the combined matrix \cite{xia2010multiview}. 

Although the mentioned methods have largely improved the clustering accuracy in multi-view circumstances, the storage and computational cost on the $n \times n$ Laplacian matrices limits the efficiency of these methods. To further improve the efficiency of the existing literature, a large number of methods are proposed. Zhou et al. reduce the complexity of spectral clustering by employing random Fourier features to construct the base kernels and do low-rank SVD decompositions accordingly \cite{zhou2019incremental}. Semertzidis et al. propose an efficient spectral clustering method for large-scale data sets in which a set of pairwise constraints are given to increase clustering accuracy and reduce clustering complexity \cite{Semertzidis2015}. The work in \cite{Li2015large} adopts a novel bipartite graph, which records only the similarity between the data ($n$ samples) and the salient point set ($p$ samples) for spectral clustering, thus largely reduces the memory and computational complexity. Chen et al. \cite{chen2011} and Cai et al. \cite{Deng} propose landmark-based spectral clustering and spectral dimensionality reduction, in which they adopt a representative point-based strategy to construct the similarity graph to accelerate the procedure of spectral clustering. The Nystr\"om approach \cite{alaoui2014fast} samples $m (\ll n)$ columns from the affinity matrix, and then forms a low-rank approximation of the full matrix by using the correlations between the sampled columns and the remaining $n-m$ columns. As only a portion of the full matrix is computed and stored, the Nystr\"om approach can reduce the time and space complexity significantly. Fowlkes et al. \cite{Fowlkes:2004} and Li et al. \cite{Li:2010} successfully applied this to spectral clustering and propose the algorithms which can scale to very large data sets.

Existing algorithms have achieved various improvements in multi-view spectral clustering. But we observe that these algorithms bear the following drawbacks. {First}, the algorithms in the third category share a common assumption that the optimal Laplacian matrix lies in the linear space spanned by the base Laplacian matrices. This assumption, on the one hand, simplifies the optimization procedure. But on the other hand, as it is uncovered in recent work that it might over-reduce the feasible set of the optimal Laplacian matrix and could result in limited representation capability of the learned matrix~\cite{bach2009exploring,cortes2009learning,liu2017optimal,li2018discriminative}. {Second}, existing algorithms do not sufficiently consider the high-order affinity information, which is important to reveal hidden neighborhood relations among samples. Both factors could adversely affect the learned Laplacian matrix, leading to unsatisfying clustering performance.

In this paper, we propose an optimal neighborhood multi-view spectral clustering algorithm (termed optimal neighborhood multi-view spectral clustering, ONMSC) to address both issues. Specifically, in our proposed algorithm, instead of restricting the optimal Laplacian matrix being a linear combination of base Laplacian matrices, we allow the optimal matrix to lie in the neighborhood of the latter. In this way, our algorithm effectively enlarges the region from which an optimal Laplacian matrix can be chosen and consequently improves its representative capacity. Moreover, we further enforce the learned optimal Laplacian matrix to be in the neighborhood of the linear combination of both the first-order and high-order base Laplacian matrices. As a consequence, the constructed optimal Laplacian matrix will be able to exploit both the first-order and high-order connection information. After that, we carefully instantiate an optimization objective function and develop an efficient algorithm with proved convergence to solve the resulting optimization problem. 

Unlike the early fusion methods that merge the base affinity matrices or Laplacian matrices, late fusion multi-view clustering \cite{Tao:2017,Zhou:2020,liu2018late,Wang:2019} generates base partitions from each view independently and integrates them into a consensus one. \cite{Tao:2017} adopts the low-rank and sparse decomposition to maintain consistency and get rid of the adverse effects of noises across views for better clustering performance. \cite{Zhou:2020} clusters instances from easy to difficult by a self-paced clustering ensemble method to enhance the stability of the corresponding algorithm. \cite{liu2018late,Wang:2019} learn an optimal consensus partition by maximally aligning the consensus partition with the weighted base partitions. The mentioned multi-view algorithms reduce the computational complexity of each iteration from $\mathcal{O}(n^3)$ to $\mathcal{O}(n^2)$ \cite{Tao:2017,Zhou:2020} or $\mathcal{O}(n)$ \cite{liu2018late,Wang:2019} while keeping comparable clustering accuracy. Inspired by the recent development of late fusion multi-view clustering \cite{liu2018late,Wang:2019}, we further extend the proposed algorithm into the late fusion fashion (denoted as ONMSC-LF) for efficient computation.

The contributions of this paper are summarized as follows:
\begin{itemize}
	\item  We discover that the current linear combination-based multi-view spectral clustering framework could: 1) limit the representation capacity of the learned Laplacian matrix; and 2) insufficiently explore the high-order neighborhood information among data. 
	\item We propose a high-order optimal Laplacian matrix construction method to solve the above problems. In our proposed method, both first-order and high-order affinity information is fully explored and the searching space of the optimal Laplacian matrix is considerably enlarged. 
	
	\item We also extend the proposed algorithm with a late fusion fashion and a Nystr\"om sample technique to further improve the efficiency of the proposed algorithm. As a consequence, the computational complexity and the storage complexity has drastically reduced from $\Ocal{n^3}$ to $\Ocal{n}$ per iteration.
\end{itemize}

The notations that are used throughout the paper are summarized in Table \ref{notations}. The rest of the paper is organized as follows. The related work of spectral clustering, high-order Laplacian matrix and late fusion multi-view clustering is reviewed in Section \ref{Pre}. The proposed optimal neighborhood multi-view spectral clustering algorithm and its late fusion version are described in Section \ref{prop}. The experimental results are reported in Section \ref{exper}. Finally, the paper is  concluded in Section \ref{con}.

\begin{table}[H]
	\centering
	\caption{Summary of notations}\label{notations}
	\vspace{-10pt}
	\begin{tabular}{c|c}
		\toprule
		$\X$ & A dataset of $n$ samples\tabularnewline
		$\x_i$ & The $i$-th sample of $\X$\tabularnewline
		$n$ & The number of samples in $\X$\tabularnewline
		$k$ & The number of clusters\tabularnewline
		$O$ & The largest order number\tabularnewline
		$v$ & The number of views\tabularnewline
		$N$ & The neighbor number of affinity matrix\tabularnewline
		$\A^{(o)}$ & $o$-th order affinity matrix\tabularnewline
		$\D^{(o)}$ & The degree matrix of $\A^{(o)}$\tabularnewline 
		$\L^{(o)}$ & The Laplacian matrix of $\A^{(o)}$\tabularnewline
		$\H^{(o)}_p$ & The $o$-th order cluster indicating matrix of $p$-th view\tabularnewline
		$\W^{(o)}_p$ & The rotation matrix of $\H^{(o)}_p$\tabularnewline
		$\bmu$ & Combination coefficients of multi-view\tabularnewline
		$\M$ & Correlation measuring matrix of multi-view\tabularnewline 
		$\F$ & The average cluster indicating matrix\tabularnewline 
		$\L^*$ & The optimal Laplacian matrix\tabularnewline 
		$\H^*$ & The optimal cluster indicating matrix\tabularnewline 
		$\G$ & The normalized affinity matrix\tabularnewline
		$\lambda_1$ & The average view balancing parameters\tabularnewline
		$\lambda_2$ & The diversity balancing parameters\tabularnewline
		\bottomrule
	\end{tabular}
\end{table}

\section{Preliminaries} \label{Pre}
In this section, we first briefly introduce some important notations about spectral clustering and then revisit the 1) \textit{linear Laplacian matrix combination-based multi-view spectral clustering} and 2) \textit{late fusion alignment maximization based multi-view clustering}. Finally, we introduce the definition of the high-order Laplacian matrix in our paper.
\subsection{Spectral Clustering}
Spectral clustering is a powerful unsupervised machine learning algorithm, especially for linear inseparable data. Denote the given data matrix as $\mathbf{X}=[\mathbf{x}_1,...,\mathbf{x}_n]^\top\in \mathbb{R}^{n\times d}$, where $n$ is the sample number and $d$ is the feature dimension. Given a kernel function $\kappa(\cdot,\cdot)$, the affinity matrix $\mathbf{A}$ can be constructed in a K-NN fasion. In particular, in the affinity matrix, $x_i$ and $x_j$ are linked if at least one of them is among the $k$ nearest neighbors of the other in the measurement of $\kappa(\cdot,\cdot)$. The $j$-th element of the $i$-th row of $\A$ is:
\begin{equation}\label{affine}
\A_{ij}=\begin{cases}
\kappa(\x_i,\x_j),&\text{if}\,\x_i\,\text{and}\,\x_j\,\text{are linked};\\
0,&~~~~~~~~~~~~~~~~~~\text{otherwies}.
\end{cases}
\end{equation}
Denoting the $i$-th diagonal element in the degree matrix $\D\in\RR^{n\times n}$ as 
\begin{equation}\label{degree}
\D_{ii}=\sum_{j=1}^{n} \A_{ij},
\end{equation}
and the definition of the corresponding first-order normalized graph Laplacian matrix is:
\begin{equation}\label{laplacian}
\L^{(1)}~\triangleq~\I_n~-~\D^{-\frac{1}{2}}\A\D^{-\frac{1}{2}},
\end{equation}
where $\mathbf{I}_n$ is a $n\times n$ identity matrix.
Let $\H\in\RR^{n\times k}$ denotes the cluster indicating matrix, where $k$ is the number of classes. The object function of the normalized spectral clustering\cite{ng2002spectral} is:

\begin{equation}\label{SC object function}
\min_{\H^\top\H=\I}\tr\lrincir{\H\Tr \L^{(1)}\H}.
\end{equation}
The optimal $\mathbf{H}$ can be easily acquired by conducting singular value decomposition (SVD) w.r.t. $\mathbf{L}^{(1)}$ and take the eigenvectors corresponding to the smallest $k$ eigenvalues. Finally, the categorical assignment can be acquired by doing $k$-means on the learned optimal cluster indicating matrix $\mathbf{H}$.

\subsection{Multi-view Spectral Clustering with Linear Laplacian Matrix Combination}
For multi-view data, let $v$ be be the number of views, $\A_1,...,\A_v\in\RR^{n\times n}$ be the affinity matrix of each view and $\L^{(1)}_1,...,\L^{(1)}_v\in\RR^{n\times n}$ be the corresponding first-order normalized Laplacian matrices. To exploit the complementary information from different views, \cite{xia2010multiview} linearly aggregates the base Laplacian matrices from each view and learns an optimal matrix which is the most suitable for clustering. The formulation of the algorithm is:

\begin{equation}\label{RMSC}
\begin{aligned}
&\min_{\H^\top\H=\I_c,\bmu}\tr\lrincir{\H\Tr \L^{(1)}_{\bmu} \H}, \\
\text{s.t.}~\L^{(1)}_{\bmu}&=\sum^{v}_{p = 1}\mu^{r}_p\L^{(1)}_p,\|\bmu\|_1 = 1,\bmu \geq 0,
\end{aligned}
\end{equation}
where $\mu_p$ is the combination weight of the $p$-th view, $\L^{(1)}_{\bmu}$ is the optimal Laplacian matrix for learning, and $r (\in \mathbb{N}^+)$ is a hyper-parameter to balance the contribution of each view. 
Although good performance has been achieved by the above method, recent literatures show that this method over reduces the feasible set of the optimal Laplacian matrix, which may lead to a less representative solution and yield even worse performance than using a single view\cite{liu2017optimal}. 

\subsection{Multi-view Clustering via Late Fusion}
In multiple kernel clustering and multi-view spectral clustering, recording and doing complex operations on the $n \times n$ kernel or Laplacian matrices are storage and computational expensive. To solve the problem, Wang et al.\cite{Wang:2019} adopt a late fusion fashion and propose an efficient multi-view clustering algorithm. In their method, to reduce storage and computational cost, the authors use the light-weighted cluster indicating matrix obtained by the kernel $k$-means algorithm to represent the categorical information from each view. Then, by maximally aligning the linear combination of the rotated cluster indicating matrix with the optimal data partition matrix, the information of each view is efficiently and effectively fused. Let $\H_p (p\in[v])$ be the cluster indicating matrix of the $p$-th view, the formulation of the late-fusion based multi-view clustering is as follow:
\begin{equation}
\begin{aligned}
&\max_{\H^*,\{\W_p\}^v_{p=1},\bmu}\tr({\H^*}^\top\S)+\lambda\tr({\H^*}^\top\F), \\
&\;\;\;\;\text{s.t.}~{\H^*}^\top\H^*=\I_k,{\W}^\top\W=\I_k,\\
&\sum_{p=1}^v{\mu_p}^2=1,\mu_p\geq0,\S=\sum_{p=1}^v\mu_p\H_p\W_p,
\end{aligned}
\end{equation}
where $\{\W_p\}^m_{p=1}\in \RR^{k\times k}$ is a set of rotation matrices, and $\S=\sum_{p=1}^v\mu_p\H_p\W_p$ is the linear combination of the rotated cluster indicating matrices; $\F$ denotes the average partition matrix, which we can obtain by performing spectral clustering on the average affinity matrix $\frac{1}{v}\sum_{p=1}^{v}\A_p$; $\lambda$ is a trade-off parameter to prevent $\H^*$ from being too far way from prior average partition.

\subsection{High-Order Laplacian Matrix}
First-order and second-order connections are essential concepts in graph analyzing \cite{tang2015line}. Specifically, in graph embedding, the first-order connection refers to the local pairwise proximity between vertices in a graph. Comparatively, the second-order connection holds that vertices which have similar affinity network structure are also similar to each other. An example can be found in Fig. \ref{highorder_graph}. In this figure, the circles indicate samples in the dataset and the edges indicate the first-order connection between the corresponding samples. As we can see, sample 5 and 6 are not connected, they are with a low similarity w.r.t. the definition of first-order connection. However, since both sample 5 and sample 6 are connected with sample 7, 8, 9 and 10, they share an identical neighborhood network structure. As a consequence, w.r.t. the definition of second-order connection, sample 5 and sample 6 are similar with each other.

Moreover, in recent literatures, because of the popularity of graph convolutional neural networks \cite{defferrard2016convolutional}, higher-order connection information has attracted the attention of researches. In these papers, the order of connections has been explained as the receptive field of different convolutional filters. Specifically, the definition of second-order proximity in \cite{tang2015line} is as follows:
\begin{Definition}[Second-order Proximity]
	The second-order proximity between a pair of vertices $(u,v)$ in a network is the similarity between their neighborhood network structures.
\end{Definition}
According to the above definition, denote $\a_j$ as the $j$-th column of first-order affinity matrix $\A$, the mathematical definition of the second-order affinity matrix $\A^{(2)}$ is:
\begin{equation}\label{high-order affine}
\A^{(2)}_{ij}\triangleq\a^\top_i\a_j,\forall i,j\in [n].
\end{equation}
Consequently, the corresponding second-order normalized Laplacian matrix can be written as:
\begin{equation}
\L^{(2)}\triangleq \I_n - \lrincir{\D^{(2)}}^{-\frac{1}{2}}\A^{(2)}\lrincir{\D^{(2)}}^{-\frac{1}{2}},
\end{equation}
where $\D^{(2)}_{ii}=\sum^{n}_{j=1}\A^{(2)}_{ij}$. According to this definition, we can readily calculate a $o$-order proximity via $\A^{(o)}=\A^{(o-1)}\A$.
As shown by existing literature \cite{tang2015line}, first-order connection in the real world data is usually not sufficient to preserve the global data structure. However, existing methods in this regard do not sufficiently consider the high-order information, which is crucial to improve the learning performance, especially in unsupervised scenario.

\begin{figure}[!t]
	\setlength{\abovecaptionskip}{0pt}
	\setlength{\belowcaptionskip}{0pt}
	\centering 
	\includegraphics[width=0.7\columnwidth]{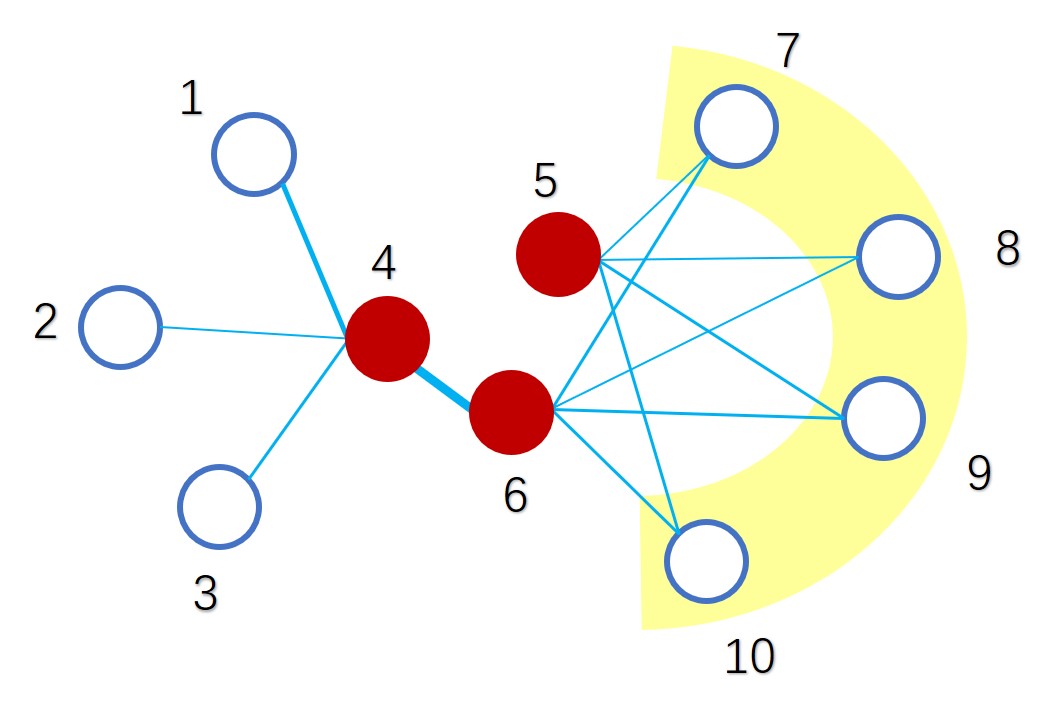}
	\caption{A toy example of the graph of an affinity matrix. Sample 4 and 6 should be placed closely in the low-dimensional space as they are connected through a strong tie. Sample 5 and 6 should also be placed closely as they share similar neighbors.}
	\label{highorder_graph}
\end{figure}

\section{The Proposed Algorithm}\label{prop}
In this section, to explore better representation capacity and more comprehensively exploit both the first-order and high-order affinity information in data, we first propose a novel multi-view spectral clustering algorithm with optimal neighborhood Laplacian matrix. Then, to improve both the storage and computational efficiency, we extend the proposed algorithm to a late fusion version.
\subsection{Multi-View Spectral Clustering with High-Order Optimal Neighborhood Laplacian Matrix}
To better capture the high-order affinity information and search the optimal Laplacian matrix in a larger space, we propose the following formulation:
\begin{equation}\label{Laplacian1}
\begin{aligned}
\min_{\H,\bmu,\L^*} &\tr \lrincir{\H^\top\L^*\H} + \sum^{O}_{o=1}\|\L^*-\L^{(o)}_{\bmu}\|^2_\text{F} + \alpha {\bmu}^\top\M\bmu, \\
\text{s.t.}~&\L^{(o)}_{\bmu} = \sum^v_{p=1}\mu_p\L^{(o)}_p(o\in [O]),\|\bmu\|_1=1,\bmu\geq 0,\\
&\L^*\succeq 0,\L^*_{mn}\leq 0(m\neq n),\H^\top\H=\I_k,
\end{aligned}
\end{equation}
where $\L^*$ is the optimal Laplacian matrix for learning, $\L^{(o)}_{\bmu}$ is the linear combination of the $o$-order base Laplacian matrices, $O$ is the largest order number, $\alpha$ is an importance balancing coefficient, and $\M$ is the correlation measuring matrix which records the centered kernel alignment value \cite{cortes2012algorithms} between affinity matrices. Specifically, denote the $o$-order affinity matrix of the $p$-th view as $\A^{(o)}_p$, the definition of $\M$ is:
\[
\M_{pq} = \sum_{o=1}^O\frac{\tr\lrincir{\A^{(o)}_p\A^{(o)}_q}}{\|\A^{(o)}_p\|_\text{F}\|\A^{(o)}_q\|_\text{F}}.
\]
In the objective function of Eq. \eqref{Laplacian1}, the first term is the spectral clustering term which encourages the learned optimal Laplacian matrix to perform well in clustering. In the second term, we restrict $\L^*$ to be in the neighborhood of the linearly combined multi-order based Laplacian matrices by minimizing the difference between $\L^*$ and $\L^{(o)}_{\bmu}$'s at the same time. The third term is the diversity inducing term which tries to introduce more diverse information for optimal Laplacian matrix construction by minimizing the overall pair-wise correlation between the base affinity matrices \cite{liu2016multiple}.

In Eq. \eqref{Laplacian1}, the PSD and non-positive constraints are added to guarantee that the learned matrix $\L^*$ to be a Laplacian matrix. However, these constraints also make the corresponding optimization problem hard and inefficient to solve. To tackle the problem, we take advantage of the original definition of the Laplacian matrix, and propose the following formulation:
\begin{equation}\label{laplacian2}
\begin{aligned}
&\min_{\H,\bmu,\P,\boldsymbol{\Lambda}}\tr\lrincir{\H^\top(\I_n-\P\boldsymbol{\Lambda}\P^\top)\H} \\
+ \sum^{O}_{o=1}&\|(\I_n-\P\boldsymbol{\Lambda}\P^\top)-\L^{(o)}_{\bmu}\|^2_\text{F} + \alpha{\bmu}^\top\M{\bmu},\\
\text{s.t.}~&\L^{(o)}_{\bmu} = \sum^v_{p=1}\mu_p\L^{(o)}_p(o\in [O]),\|\bmu\|_1=1,{\bmu}\geq 0,\\
\P\in&\RR^{n\times k},\P^\top\P=\I_k,0\leq \boldsymbol{\Lambda}_{ii}\leq 1,\H^\top\H=\I_k,
\end{aligned}
\end{equation}
where $\boldsymbol{\Lambda}\in \RR^{k\times k}$ is a diagonal matrix. In the new formulation, we use $\P\boldsymbol{\Lambda}\P^\top$ to represent a low-rank normalized affinity matrix and $\I_n-\P\boldsymbol{\Lambda}\P^\top$ to represent the corresponding Laplacian matrix. Notably, the constraint $0\leq\boldsymbol{\Lambda}_{ii}\leq 1$ is added to make sure that the optimization process is stable. The optimization procedure of Eq. \eqref{laplacian2} is listed in the appendix, please check Appendix A for details.

\subsection{Late fusion-based Multi-View Spectral Clustering with Optimal Neighborhood Laplacian Matrix}
To improve the efficiency of the proposed algorithm, we further propose the late fusion version of optimal neighborhood multi-view spectral clustering. In this version, we use more compact and light-weighted data partition matrices instead of the heavy-weighted Laplacian matrices to present the multi-level affinity information from different views to reduce the cost both in storage and computation.
The formulation of the late fusion method is as follows:
\begin{equation}
\begin{aligned}
\max_{\H^*,\{\W^{(o)}_p\}^{v,O}_{p,o=1},\bmu}&\tr({\H^*}^\top\S)+\lambda_1\tr({\H^*}^\top\F)-\lambda_2\tr\lrincir{{\bmu}^\top\M\bmu} \\
\text{s.t.}~{\H^*}^\top&\H^*=\I_k,{\W_p^{(o)}}^\top\W_p^{(o)}=\I_k,\\
\M_{pq}&=\sum_{o=1}^O\frac{\tr\lrincir{{\H^{(o)}_p}^\top\H^{(o)}_q}}{\|\H^{(o)}_p\|_\text{F}\|\H^{(o)}_q\|_\text{F}},\\
\sum^v_{p=1}\mu_p=1,&\mu_p\geq0,\S=\sum_{o=1}^O\sum_{p=1}^v\mu_p\H^{(o)}_p\W^{(o)}_p,\\
(\forall~p,q&\in[v]~\text{and}~o\in[O])
\end{aligned}
\label{objective function}
\end{equation}
where $\H^{(o)}_p$ denotes the cluster indicating matrix of $o$-order affinity matrix of the $p$-th view, $\M$ is the correlation measuring matrix, and $\F$ is the cluster indicating matrix of average first-order affinity matrix $\frac{1}{v}\sum_{p=1}^v \A^{(1)}_p$. Correspondingly, the second term in the objective function is a generalization term which is added to avoid bad local maximal solution. $\lambda_1$ and $\lambda_2$ are hyper-parameters.

\subsection{Optimization Algorithm}
In this part, we design an efficient three-step alternative optimization algorithm to solve the problem in Eq. \eqref{objective function}:

{\bfseries i) Update $\H^*$}. Given $\{\W^{(o)}_p\}^{v,O}_{p,o=1}$ and $\bmu$, the optimization problem in Eq. \eqref{objective function} w.r.t. $\H^*$ reduces to:
\begin{equation}
\max_{\H^*}\tr({\H^*}^\top\C)~~\text{s.t.}~{\H^*}^\top\H^* = \I_k,
\label{update_H}
\end{equation}
where $\C = \sum^{O}_{o=1}\sum^{v}_{p=1}\mu_p\H^{(o)}_p\W^{(o)}_p+\lambda_1\F$. The optimization of Eq. \eqref{update_H} could be easily solved by doing singular value decomposition(SVD) over the given matrix $\C$. Suppose that the matrix $\C$ in Eq. \eqref{update_H} has the economic rank-$k$ singular value decomposition form as $\C_k = \U_k\bSigma_k{\V_k}^\top$, where $\U_k\in\RR^{n\times k}$, $\bSigma_k\in\RR^{k\times k}$, $\V_k\in\RR^{k\times k}$. Through Theorem \ref{theorem1}, we can find that Eq. \eqref{update_H} has a closed-form solution, i.e $\H^*=\U_k{\V_k}^\top$.

\begin{Theorem}\label{theorem1}
	Suppose that the matrix $\C$ in Eq. \eqref{update_H} has the economic rank-$k$ singular value decomposition form as $\C_k = \U_k\bSigma_k{\V_k}^\top$, where $\U_k\in\RR^{n\times k}$, $\bSigma_k\in\RR^{k\times k}$, $\V_k\in\RR^{k\times k}$. The
	optimization in Eq. \eqref{update_H} has a closed-form solution as follows:
	\begin{equation}
	\H^*=\U_k{\V_k}^\top.
	\label{theo1}
	\end{equation}
\end{Theorem}

\begin{proof}
	By taking the the normal singular value decomposition $\C = \U\bSigma\V^\top$, the Eq. \eqref{update_H} could be rewritten as,
	\begin{equation*}
	\tr({\H^*}^\top\U\bSigma\V^\top)=\tr(\V^\top{\H^*}^\top\U\bSigma).
	\end{equation*}
	Due to the singular values of each $\H_p^{(o)}$ are non-negative, after rotating by $\W_p^{(o)}$, the singular values of $\H_p^{(o)}\W_p^{(o)}$ are the same with $\H_p^{(o)}$. Since $\C$ is the linear combination of $\H_p^{(o)}\W_p^{(o)}$'s and the singular values of $\F$ are also non-negative, we can obtain that the singular values of $\C$ are non-negative.	Considering that $\Q=\V^\top{\H^*}^\top\U$, we have $\Q{\Q}^\top=\V^\top{\H^*}^\top\U\U^\top\H^*\V=\I_k$. Therefore we can take $\tr(\V^\top{\H^*}^\top\U\bSigma)=\tr(\Q\bSigma)\leq\sum_{i=1}^k\sigma_i$. Hence, in order to
	maximize the value of Eq. \eqref{update_H}, the solution should be given as Eq. \eqref{theo1}.
\end{proof}

{\bfseries ii) Update $\{\W^{(o)}_p\}^{v,O}_{p,o=1}$}. Given $\H^*$ and $\bmu$, for each single $\W_p^{(o)}$, the optimization problem in Eq. \eqref{objective function} reduces to:
\begin{equation}
\max_{\W_p^{(o)}}\tr\lrincir{{\W_p^{(o)}}^\top \A}~\text{s.t.}~{\W_p^{(o)}}^\top \W_p^{(o)} = \I_k,
\label{update_w}
\end{equation}
where $\A=\mu_p{\H^{(o)}_p}^\top\H^*$. The solution of Eq. \eqref{update_w} is similar with that of Eq. \eqref{update_H}, it also has a closed form solution as shown in Theorem \ref{theorem1}.

{\bfseries iii) Update $\bmu$}. Given $\H^*$ and $\{\W^{(o)}_p\}^{v,O}_{p,o=1}$, the optimization problem in Eq. \eqref{objective function} w.r.t $\bmu$ reduces to:
\begin{equation}\label{update_mu}
\begin{aligned}
    &\min_{\bmu}\bmu^\top\M\bmu - \f^\top\bmu, \\
    \text{s.t.}~&\|\bmu\|_1=1, 0\leq\mu_p\leq1 (p \in [v]),
\end{aligned}
\end{equation}
where $\f_p=\frac{\lambda_1}{\lambda_2}\tr\left[{\H^*}^\top\lrincir{\sum^{O}_{o=1}\H_p^{(o)}\W_p^{(o)}}\right](p\in [v])$. Since $\M$ is PSD\cite{cortes2012algorithms}, the above function is a standard convex quadratic programming(QP) problem, its global optimal solution can be easily solved by the optimization toolbox of MATLAB.

In sum, our algorithm for solving Eq. \eqref{objective function} is outlined in Algorithm \ref{alg:algorithm1}, where $obj_{(t)}$ denotes the objective value at the $t$-th iteration.
\begin{algorithm}[htb]
	\caption{Late Fusion-based Optimal Neighborhood Multi-view Spectral Clustering} 
	\label{alg:algorithm1} 
	\begin{algorithmic}[1]
		\renewcommand{\algorithmicrequire}{\textbf{Input:}}
		\renewcommand{\algorithmicensure}{\textbf{Output:}}
		\REQUIRE ~Data from $v$ views $\{\X_{(1)},...,\X_{(v)}\}$, number of cluster $k$, parameter $\lambda_1$, $\lambda_2$ and the neighbor number $N$\\
		\ENSURE ~The learned optimal cluster indicating matrix $\H^*$\\
		~\\
		\STATE Construct first-order and high-order affinity matrices and the corresponding Laplacian matrices $\L^{(o)}_p$ of each view. Obtain the cluster indicating matrix $\H^{(o)}_p$ by standard spectral clustering of $\L^{(o)}_p$. Initialize $\H^*$ as $\0_{n\times k}$, $\bmu$ as $\1_v/v$, $\W^{(o)}_p$ as $\I_{k \times k}$, and $t$ as 1.  
		\REPEAT
		\STATE 	Calculate
		$$\H^*_{(t-1)}=\sum^v_{p=1}\sum^O_{o=1}\bmu_{p_{(t-1)}}\H^{(o)}_{p_{(t-1)}}\W^{(o)}_{p_{(t-1)}}.$$
		\STATE  Calculate $\W^{(o)}_{p_{(t)}}$ by optimizing Eq. \eqref{update_w}.
		\STATE  Calculate $\bmu_{(t)}$ by optimizing Eq. \eqref{update_mu}.
		\STATE 	Calculate $\H^*_{(t)}$ by optimizing Eq. \eqref{update_H}.
		\STATE  $t=t+1$.
		\UNTIL $|Obj_{(t)}-Obj_{(t-1)}|/|Obj_{(t)}|<10^{-4}$.
	\end{algorithmic}
\end{algorithm}

\subsection{Algorithmic Discussion} \label{Algorithmic Discussion}
{\bfseries Convergence}. In each optimization iteration of Algorithm \ref{alg:algorithm1}, two eigenvalue decomposition problems and one convex quadratic programming problem are solved. Since all these three sub-problems have optimal solution, the objective of Algorithm \ref{alg:algorithm1} is guaranteed to be monotonically increased when optimizing one variable with others fixed at each step. Moreover, the objective function is upper bounded by Proposition \ref{upper bound}. As a result, since the objective value rises monotonically while is upper bounded, our proposed optimization algorithm is guaranteed to converge to a local optimum of problem Eq. \eqref{objective function}.

\begin{Proposition}\label{upper bound}
	The value of the objective function in Eq.~\eqref{objective function} is upper bounded by $\frac{1}{2}(1+O^2v^2+2\lambda_1)k$.
\end{Proposition}
The proof of Proposition \ref{upper bound} can be found in Appendix B.

{\bfseries Computational Complexity}. 
The computation of the proposed algorithm mainly includes four parts, i.e., an initialization procedure and three-step alternative optimization procedures. Among these steps, in the initialization procedure, calculating the corresponding cluster indicating matrices $\mathbf{H}_p^{(o)}$ requires conducting SVD on $n\times n$ Laplacian matrices, which incurs $\Ocal{Ovn^3}$ complexity. 
In the optimization procedure, updating the optimal cluster indicating $\H^*$ requires solving a SVD problem on a $n \times k$ matrix, which takes $\Ocal{nk^2}$ time.
Updating the $\{\W^{(o)}_p\}_{p=1, o=1}^{v, O}$ requires solving $Ov$ SVD problems over $k \times k$ matrices. It will incur $\Ocal{Ovk^3}$ computational complexity.
Finally, updating $\bmu$ requires solving a standard Quadratic Programming with Linear Constraints (QPLC) whose complexity is $\Ocal{\epsilon^{-1}v}$, where $\epsilon$ is the precision of the result. Let $T$ be the iteration number, the overall complexity of our algorithm is $\Ocal{Ovn^3+T(Ovk^3+nk^2+\epsilon^{-1}v)}$. Considering that $\epsilon^{-1},~k,~O,~\text{and}~v$ far less than $n$, the complexity is basically $\Ocal{n^3+Tn}$.

\subsection{Scale to Large-Scale Dataset}
In the previous section, by analyzing the computational complexity of algorithm \ref{alg:algorithm1}, we find the computational bottleneck mainly lies in the computation of the cluster indicating matrices of each view. This makes our algorithm hard to extend to large-scale datasets. In this section, we further adopt an efficient Nystr\"om algorithm to make the proposed algorithm more suitable for large scale datasets.

Specifically, we follow the suggestion of Li et al.\cite{Li:2010} to combine the Nystr\"om algorithm with randomized SVD and propose an algorithm to efficiently learn the spectral embedding of large-scale datasets. Since calculating the eigenvectors corresponding to the smallest $k$ eigenvalues of $\L$ is equivalent with calculating the eigenvectors corresponding to the largest eigenvalues of the normalized affinity matrix $\G=\D^{-\frac{1}{2}}\A\D^{-\frac{1}{2}}$, for the application convenience of the Nystr\"om method, we take $\G$ instead of $\L$ for calculation.

{\bf Nystr\"om Method}. The Nystr\"om method \cite{Williams2001} has been commonly used to constuct low-rank matrix approximation. Given a symmetric matrix $\G\in\RR^{n\times n}$, this algorithm first samples $m(\ll n)$ columns from $\G$ (denote the columns selected as $\mathbf{E} \in \mathbb{R}^{n\times m}$). Let $\R$ be the $m \times m$ matrix consisting of the intersection of these $m$ columns with the corresponding $m$ rows of $\G$. The rows and columns of $\G$ can be rearranged such that $\E$ and $\G$ are written as:
\[
\E= \left[ \begin{array}{c}
\R \\
\R^{'}
\end{array} 
\right]~~\text{and}~~ 
\G= \left[ \begin{array}{cc}
\R & {\R^{'}}^\top \\
\R^{'} & \R^{''}
\end{array} 
\right],
\]
where $\R^{'}\in\RR^{(n-m)\times m}$ and $\R^{''}\in\RR^{(n-m)\times (n-m)}$. Assume that the SVD of $\R$ is $\U\boldsymbol{\Lambda}\U^\top$, where $\boldsymbol{\Lambda}=\text{diag}(\sigma_1,...,\sigma_m)$ is the diagonal matrix containing the singular values of $\R$ in non-increasing order. For $k\leq m$, the rank-$k$ Nystr\"om approximation is
\[
\widetilde{\G}=\E\R^+_k\E^\top,
\]
where $\R^+_k=\sum_{i=1}^{k}\sigma^{-1}_i\U^{(i)}\U^{(i)\top}$, and $\U^{(i)}$ is the $i$-th column of $\U$. The time complexity of the approximation is $\Ocal{nmk+m^3}$, which is much smaller than the $\Ocal{n^3}$ complexity of doing standard SVD on $\G$ When $m\ll n$.

{\bf Randomized SVD}. When the number of sampled columns $m$ is large, the complexity $\Ocal{nmk+m^3}$ of the Nystr\"om method is still unacceptable. To tackle the problem, Halko et al. \cite{randomSVD2011} propose a class of simple but efficient randomized algorithm. Our adopted algorithm includes two stages. In the first stage, we generate a $m\times(k+s)$ standard Gaussian random matrix $\boldsymbol{\Omega}$ (i.e., each entry of $\boldsymbol{\Omega}$ is an independent Gaussian random variable with zero mean and unit variance). And then, we form the matrix $\Y = \R \boldsymbol{\Omega}$ and construct a matrix $\Q$ whose columns form an orthogonal basis for the range of $\Y$ (by QR decomposition). As a consequence, we find an approximate basis $\Q$ for the range of $\R$, such that $\R\approx \Q\Q^\top\R$. The number of columns in $\boldsymbol{\Omega}$ is often set to be larger than the required rank $k$ by an over-sampling parameter $s$. Typically, $s$ is a small number such as 5 or 10. It enables $\Y = \R \boldsymbol{\Omega}$ to have a better chance to span the $k$-dimensional subspace of $\R$.
In the second stage, $\R$ is restricted to the obtained subspace from $\Y$, leading to the reduced matrix $\B=\Q^\top\R\Q$. A standard SVD is computed on $\B$ to obtain $\B=\V\boldsymbol{\Lambda}\V^\top$. The SVD of $\R$ can then be approximated as
\[
\R\approx\Q\B\Q^\top = (\Q\V)\boldsymbol{\Lambda}(\Q\V)^\top.
\]
The operations of the randomized SVD are shown in Algorithm \ref{alg:algorithm_RSVD}. The proposed fast spectral embedding via Nystr\"om and randomized SVD are reported in Algorithm \ref{alg:algorithm2}.
\begin{algorithm}[htb] 
	\caption{Randomized SVD}
	\label{alg:algorithm_RSVD} 
	\begin{algorithmic}[1] 
		\renewcommand{\algorithmicrequire}{\textbf{Input:}}
		\renewcommand{\algorithmicensure}{\textbf{Output:}}
	    \REQUIRE symmetric matrix $\R\in\RR^{m\times m}$, rank $k$, over-sampling parameter $s$\\
	    \ENSURE $\U$, $\boldsymbol{\Lambda}$\\
	    \STATE $\boldsymbol{\Omega} \leftarrow$ a $m\times(k+s)$ standard Gaussian random matrix.
	    \STATE $\Y \leftarrow$ $\R\boldsymbol{\Omega}$.
	    \STATE Find an orthogonal matrix $\Q$ (by QR decomposition) such that $\Y=\Q\Q^\top\Y$.
	    \STATE $\B \leftarrow \Q^\top\R\Q$.
	    \STATE Perform SVD on $\B$ to obtain $\V\boldsymbol{\Lambda}\V^\top$.
	    \STATE $\U \leftarrow \Q\V$.
	\end{algorithmic}
\end{algorithm}
\begin{algorithm}[htb] 
	\caption{Fast Spectral Embedding via Nystr\"om and Randomized SVD}
	\label{alg:algorithm2} 
	\begin{algorithmic}[1] 
		\renewcommand{\algorithmicrequire}{\textbf{Input:}}
		\renewcommand{\algorithmicensure}{\textbf{Output:}}
		\REQUIRE~the $p$-th view $\X_p(p\in[v])\in\RR^{n\times d}$, number of sampled columns $m$, over-sampling parameter $s$, number of cluster $k$  \\
		\ENSURE~the cluster indicating matrix $\H\in\RR^{n\times k}$ of the $o$-th order affinity matrix, the $k$ largest approximate eigenvalues $\boldsymbol{\Lambda}_k$. \\
		~\\
		\STATE  Construct $o$-th order affinity matrix $\A$ by Eq. \eqref{affine} and Eq. \eqref{high-order affine}.
		\STATE Compute the degree matrix $\D$ by Eq. \eqref{degree}.
		\STATE $\G \leftarrow \lrincir{\D}^{-\frac{1}{2}}\A\lrincir{\D}^{-\frac{1}{2}}$.
		\STATE 	$\E\leftarrow$ $m$ columns of $\G$ sampled uniformly at random without replacement.
		\STATE  $\R\leftarrow$ the intersection of the $m$ columns sampled in the step 4 with the corresponding $m$ rows of $\G$.
		\STATE $[\widetilde{\U},\boldsymbol{\Lambda}]\leftarrow$ randsvd($\R,k,s$) using Algorithm \ref{alg:algorithm_RSVD}.
		\STATE  $\U\leftarrow\E\widetilde{\U}\boldsymbol{\Lambda}^+$.
		\STATE $\H\leftarrow\sqrt{\frac{m}{n}}\U$.
	\end{algorithmic}
\end{algorithm}
By Algorithm \ref{alg:algorithm2}, matrix $\hat{\G}=\lrincir{\sqrt{\frac{m}{n}}\U}\lrincir{\frac{n}{m}\boldsymbol{\Lambda}}\lrincir{\sqrt{\frac{m}{n}}{\U^\top}}$ can be regarded as an approximation of the input matrix $\G$, and $\H$ is the cluster indicating matrix. 
According to the conclusion in Theorem \ref{error} of \cite{Li:2010}, the approximation error of Algorithm \ref{alg:algorithm2} is upper bounded.

\begin{Theorem}[\cite{Li:2010}]\label{error}
	For the $\G = \lrincir{\D}^{-\frac{1}{2}}\A\lrincir{\D}^{-\frac{1}{2}}$ and $\hat{\G}=\lrincir{\sqrt{\frac{m}{n}}\U}\lrincir{\frac{n}{m}\boldsymbol{\Lambda}}\lrincir{\sqrt{\frac{m}{n}}{\U^\top}}$ obtained by Algorithm \ref{alg:algorithm2},
	\begin{equation}
	\begin{aligned}
	\EE&\|\G-\hat{\G}\|_\text{F} \\
	\leq &\frac{2(k+s)}{\sqrt{s-1}}\|\G-{\G}_k\|_\text{F}+\lrincir{1+\frac{4(k+s)}{\sqrt{m(s-1)}}}n\G^*_{ii},
	\end{aligned}
	\end{equation}
	where $\G_k$ is the best rank-k approximation of $\G$, $\G^*_{ii}=\max_i\G_{ii}$. 
\end{Theorem}

Note that the approximate eigenvectors $\H$ obtained by algorithm \ref{alg:algorithm2} may not be orthogonal. According to \cite{li2011time}, we orthogonalize $\H$ by the Algorithm \ref{alg:algorithm3}.

\begin{algorithm}[htb] 
	\caption{Orthogonalize $\H$}
	\label{alg:algorithm3} 
	\begin{algorithmic}[1]
		\renewcommand{\algorithmicrequire}{\textbf{Input:}}
		\renewcommand{\algorithmicensure}{\textbf{Output:}}
		\REQUIRE~$\H\in\RR^{n\times k},\boldsymbol{\Lambda}\in\RR^{k\times k}$\\
		\ENSURE~orthogonal $\widetilde{\H}, \widetilde{\boldsymbol{\Lambda}}$\\
		~\\
		\STATE $\T\leftarrow \H^\top\H$.
		\STATE eigen-decomposition: $\T = \V\boldsymbol{\Sigma}\V^\top$.
		\STATE $\K\leftarrow \boldsymbol{\Sigma}^\frac{1}{2}\V^\top\boldsymbol{\Lambda}\V\boldsymbol{\Sigma}^\frac{1}{2}$.
		\STATE eigen-decomposition: $\K = \widetilde{\V}\widetilde{\boldsymbol{\Lambda}}\widetilde{\V}^\top$.
		\STATE $\widetilde{\H}\leftarrow\H\V\boldsymbol{\Sigma}^{-\frac{1}{2}}\widetilde{\V}$.
	\end{algorithmic}
\end{algorithm}

The following proposition shows that, after performing Algorithm \ref{alg:algorithm3}, the $\H$ has an orthogonalization version $\widetilde{\H}$. 

\begin{Proposition}
	In Algorithm \ref{alg:algorithm3}, $\H\boldsymbol{\Lambda}\H^\top=\widetilde{\H}\widetilde{\boldsymbol{\Lambda}}\widetilde{\H}^\top$, and $\widetilde{\H}^\top\widetilde{\H}=\I$.
\end{Proposition}


\section{Experiments} \label{exper}
\subsection{Datasets and Experimental Settings}
We evaluate the clustering performance of the proposed algorithm on 9 popular datasets from various applications, including natural language processing, protein subcellular localization, and image recognition. The detailed information of these datasets is listed in Table \ref{Table:datasets1}. From this table, we observe that the number of samples, views, and categories of these datasets range from 165 to 60,000, 2 to 69, and 3 to 102, respectively. For these datasets, all affinity matrices are pre-computed with carefully designed similarity function and are publicly available from websites\footnote{http://mlg.ucd.ie/datasets/bbc.html}\footnote{http://mkl.ucsd.edu/dataset/protein-fold-prediction}\footnote{http://www.robots.ox.ac.uk/~vgg/data/}.

\begin{table}[htbp]
	\centering
	\caption{Information of benchmark datasets. In this table, the number of samples, views, and categories of these datasets are illustrated.}\label{Table:datasets1}
	\vspace{-10pt}
	\scalebox{0.95}{
		\begin{tabular}{c||c|c|c}
			\hline
			Datasets &\#  Samples &\#  Views &\# Clusters  \\
			\hline\hline
			{YALE}             &165    &12   &  15  \\
			{BBCSport}         & 554   & 2   &  5   \\
			{ProteinFold}      & 694   & 12  &  27  \\
			{Flower17}         & 1360  & 7   &  17  \\
			{UCI-Digit}        & 2000  & 3   &  10  \\
			{Mfeat}            & 2000  & 12  &  10  \\
			{Nonpl}            & 2732  & 69  &  3   \\
			{Flower102}        & 8189  & 4   &  102 \\
			{MNIST}            &60000  & 3   &  10  \\
			\hline
	\end{tabular}}
\end{table}

In our experiments, the MATLAB implementation of all the compared algorithms is downloaded from the authors' websites. The hyper-parameters are set according to the suggestions of the corresponding literature. Especially, to all the compared spectral clustering algorithms, the optimal neighbor numbers are carefully searched in the range of $[0.1s, 0.2s, \dots, s]$, where $s=n/k$ is the average sample number in each category. As to our proposed method, the parameter $\lambda_1$ and $\lambda_2$ are chosen in the range of $[2^{-15},2^{-12}, \dots, 2^{15}]$. $K$-means clustering is adopted on the final representation to assign an appropriate label for each sample. In the experiment, to reduce the effect of randomness caused by $k$-means, we repeat the clustering process for 50 times with random initialization and report the result with the smallest $k$-means distortion. The clustering performance is evaluated in terms of three widely used criteria, including clustering accuracy (ACC), normalized mutual information (NMI), and purity. All our experiments are conducted on a desktop computer with 3.6GHz Intel Core i7 CPU, 64GB RAM, and MATLAB 2018a (64bit).

\subsection{Ablation Study}
In our first experiment, we study the effectiveness of each proposed component, i.e., the neighborhood learning mechanism (NLM) and the high-level connection information (HCI) by careful ablation study. Also, the optimal order number of the high-order Laplacian matrix is exploited. Specifically, six algorithms are designed and tested. The average clustering performance on all eight datasets are listed in Table \ref{table_A}.

\noindent\textbf{Effectiveness of the designed algorithm.}
Among the compared algorithms, the baseline method (BL) indicates a classic linear Laplacian matrix combination with matrix-induced regularization \cite{liu2016multiple}. For high-level connection information extraction, the order number of the Laplacian matrix is fixed as $2$. As we can see from Table \ref{table_A}, both the neighborhood learning mechanism and the high-level connection information is capable of improving the spectral clustering performance of the corresponding algorithm. Specifically, HCI and NLM improve the ACC of the baseline algorithm for $2.04\%$ and $3.63\%$ on average, respectively. Moreover, by combining these two designs, the resultant algorithm can improve $4.64\%$ over the baseline algorithm in terms of ACC.

\noindent\textbf{The optimal order-number.} We also test the effect of different order-number of the high-order Laplacian matrices. In this part, the second-, third-, forth- and fifth-order algorithms are compared. As one can see in Table~\ref{table_High}, the second- and the third-order algorithms provide comparably good performance. However, as the orders of the Laplacian matrices keep get higher, the range of neighborhood also gets larger and the discriminative capacity of the corresponding algorithms start to decrease a little bit. As a consequence, for the sake of the clustering performance and the computational efficiency, the order number of our proposed algorithm is fixed as two in all our following experiments.

\begin{table}[t]
\small
\setlength{\abovecaptionskip}{0pt}
\setlength{\belowcaptionskip}{10pt}
\renewcommand{\arraystretch}{1.0}
\centering
\caption{Ablation study. Average clustering performance on eight datasets of four algorithms. In the compared algorithms, BL indicates the baseline method, NLM indicates neighborhood learning mechanism, HCI indicates high-level connection information.}
\label{table_A}
\begin{tabular}{|c|c|c|c|c|}
\hline
{Methods}     & BL           &  BL+HCI         & BL+NLM            & BL+NLM+HCI    \\ \hline
{ACC ($\%$)}     &   64.00  &    66.04        &    67.63    &    \textbf{68.64} \\ \hline
{NMI ($\%$)}     &  63.58   &    64.53        &    66.42    &    \textbf{67.65}  \\ \hline
{Purity ($\%$)}  & 68.32    &    70.01        &    71.78    &    \textbf{72.29}  \\ \hline
\end{tabular}
\vspace{-10pt}
\end{table}

\begin{table}[t]
\vspace{-1pt}
\small
\setlength{\abovecaptionskip}{0pt}
\setlength{\belowcaptionskip}{10pt}
\renewcommand{\arraystretch}{1.0}
\centering
\caption{Average clustering performance comparison with different Laplacian matrix order number.}
\label{table_High}
\begin{tabular}{|c|c|c|c|c|}
\hline
{Methods}     & $2$nd-order           &  $3$rd-order         &  $4$th-order            & $5$th-order    \\ \hline
{ACC ($\%$)}     &   68.64           &   \textbf{68.94}&   67.58     &  65.76   \\ \hline
{NMI ($\%$)}     &  \textbf{67.65}   &    67.04        &   66.23     &  64.94    \\ \hline
{Purity ($\%$)}  &  \textbf{72.29}   &    72.21        &   71.87     &  69.88    \\ \hline
\end{tabular}
\vspace{-10pt}
\end{table}

\subsection{Comparison with state-of-the-art algorithms}
To verify the effectiveness of the proposed algorithm, we further compare it with six state-of-the-art multi-view spectral clustering algorithms and three multiple kernel clustering algorithms. Among these methods, (1) average multi-view spectral clustering (\textbf{A-MVSC}) uniformly weights Laplacian matrices from each view to generate a new Laplacian matrix for clustering (2) Single best spectral clustering (\textbf{SB-SC}) performs spectral clustering on every single view separately and reports the best performance. (3) Co-regularized Spectral Clustering  (\textbf{Co-reg}) \cite{kumar2011co} is a representative of the co-training methods. (4) Auto-weighted Multiple Graph Learning (\textbf{AMGL}) \cite{nie2016parameter} is a linear combination-based method. (5) Multi-view Learning with Adaptive Neighbors (\textbf{MLAN}) \cite{nie2017multi}, and (6) Robust Multi-view Spectral Clustering \textbf{RMSC} \cite{xia2014robust} are consensus Laplacian construction methods. Also, since the affinity matrices in each view can be treated as kernels, three multiple kernel clustering algorithms, i.e., (7) \textbf{ONKC} \cite{liu2017optimal}, (8) \textbf{MKKM-MR} \cite{liu2016multiple}, and (9) \textbf{RMKKM} \cite{du2015robust}, are also included for a more comprehensive comparison. The early fusion and late fusion version of our proposed algorithm refers to ONMSC-EF and ONMSC-LF in Table \ref{TableTotalResult}, respectively.

The ACC, NMI, and purity, of the algorithms mentioned above are reported in Table \ref{TableTotalResult}. As seen, in all of the eight datasets, both the proposed ONMSC-LF and ONMSC-EF show superior performance gains over the state-of-the-art algorithms w.r.t. all the three metrics. Also, the proposed algorithm significantly outperforms existing linear combination based algorithms, including RMKKM, MKKM-RM, and AMGL with comparable computational consumption. This validates the effectiveness of optimal neighborhood spectral clustering and the high-order information again.

\begin{table*}[htbp]
\vspace{-5pt}
	\centering
	\caption{ACC, NMI, purity comparison of different clustering algorithms on eight benchmark datasets. In this table, the boldface indicates the best performance among all the compared algorithms.}\label{TableTotalResult}
	\vspace{-10pt}
	\begin{tabular}{|c|c|c|c|c|c|c|c|c|c|c|c|}
		\hline
		\multirow{2}{*}{Datasets}  & \multirow{2}{*}{A-MVSC}   & \multirow{2}{*}{SB-SC}   & {RMKKM} &   {MKKM-MR}       & {ONKC}         & {Co-reg}                       & {AMGL}                         & {RMSC}                          & MLAN                         & \multirow{2}{*}{ONMSC-EF} & \multirow{2}{*}{ONMSC-LF} \\
		&  & & \scriptsize{\cite{du2015robust}}& \scriptsize{\cite{liu2016multiple}}   & \scriptsize{\cite{liu2017optimal}}&  \scriptsize{\cite{kumar2011co}} & \scriptsize{\cite{nie2016parameter}} & \scriptsize{\cite{xia2014robust}} & \scriptsize{\cite{nie2017multi}}  & & \\
		\hline
		\hline
		\multicolumn{12}{|c|}{ACC(\%)}\tabularnewline
		\hline
		\hline
		BBCSports    	&66.18 &76.65 &63.79 &66.18 &68.20 &85.66 &86.39 &86.03 &70.58 &95.77  
		&\bf 97.61      	 \\
		\hline
		ProteinFold     &30.69 &34.58 &30.98 &36.46 &37.90 &34.87 &36.88 &33.00 &28.38 &\bf 41.21
		&40.48   	 \\
		\hline
		Flower17     	&51.02 &42.05 &48.38 &60.00 &60.88 &52.72 &56.32 &53.90 &53.38 &66.39
		&\bf 67.5  		 \\
		\hline
		UCI-Digit 		&88.75 &75.40 &40.45 &90.40 &91.05 &84.80 &92.85 &90.40 &97.15 &97.6
		&\bf 97.85		 \\
		\hline
		MFeat 			&95.20 &86.00 &65.30 &83.20 &97.05 &84.30 &84.35 &84.15 &96.55 &\bf 98.1
		& 97.00	 \\
		\hline
		Nonpl 			&49.37 &57.50 &62.77 &56.59 &59.57 &55.27 &56.91 &60.65 &44.98 &65.84 
		&\bf 68.41		 \\
		\hline
		Flower102 		&27.29 &33.12 &28.17 &39.91 &41.56 &37.26 &33.34 &32.97 &24.19 &43.31
		&\bf 44.47		 \\
		\hline
		YALE            &46.06 &44.24 &58.79 &60.00 &61.21 &51.52 &60.0  &56.36 &57.58 &64.85
		&\bf 66.06      \\
		\hline
		\hline
		\multicolumn{12}{|c|}{NMI(\%)}\tabularnewline
		\hline
		\hline
		BBCSport	&53.92	&59.38	&39.62	&53.93	&54.64	&71.27	&73.7	&73.89	&65.34	&87.19
		&\bf 92.00 \\
		\hline
		ProteinFold	&40.95	&42.33	&38.78	&45.32	&46.93	&43.34	&44.18	&43.91	&27.86	&49.33
		&\bf 49.5 \\
		\hline
		Flower17	&50.18	&45.14	&50.73	&57.11	&58.58	&52.13	&56.97	&53.89	&55.38	&65.54
		&\bf 66.35 \\
		\hline
		UCI-Digit	&80.59	&68.38	&46.87	&83.22	&83.96	&73.51	&86.65	&81.8	&93.4	&94.39
		&\bf 94.86 \\
		\hline
		MFeat	&89.83	&75.78	&62.67	&78.12	&93.07	&80.99	&81.57	&81.69	&92.89   &\bf 95.51
		&93.43 \\
		\hline
		Nonpl	&16.55	&15.26	&17.34	&15.51	&24.04	&12.55	&15.19	&20.35	&6.14 &\bf 25.35
		&24.44 \\
		\hline
		Flower102	&46.32	&48.99	&48.17	&57.27	&59.13	&54.18	&51.63	&53.36	&34.94	&60.12
		&\bf 60.74 \\
		\hline
		YALE        &49.04  &50.42  &59.70  &61.29  &62.27  &57.01  &61.2   &59.11  &57.07  &64.40
		&\bf 64.46   \\
		\hline
		\hline
		\multicolumn{12}{|c|}{Purity(\%)}\tabularnewline
		\hline
		\hline
		BBCSport	&77.2	&79.59	&67.83	&77.21	&77.76	&85.66	&86.39	&86.03	&74.44	&95.77
		&\bf 97.61	\\
		\hline
		ProteinFold	&37.17	&41.21	&36.6	&42.65	&45.24	&40.78	&42.07	&42.36	&31.84	&47.98
		&\bf 49.56	\\
		\hline
		Flower17	&51.98	&44.63	&51.54	&61.03	&61.69	&56.47	&58.16	&53.24	&55.07	&68.52
		&\bf 69.7	\\
		\hline
		UCI-Digit	&88.75	&76.1	&44.2	&90.4	&91.05	&77.75	&92.85	&82.9	&97.15	&97.6
		&\bf 97.85	\\
		\hline
		MFeat	&95.2	&86	&66.25	&83.2	&97.05	&84.3	&84.35	&84.1	&96.55	&\bf 98.1
		&97	\\
		\hline
		Nonpl	&72.18	&71.12	&71.71	&63.91	&75.34	&66.07	&69.94	&70.5	&60.35	&\bf 76.13
		&75.66	\\
		\hline
		Flower102	&32.27	&38.78	&27.61	&33.86	&47.64	&44.08	&39.71	&40.24	&31.15	&50.78
		&\bf 51.75	\\
		\hline
		YALE    &48.48  &47.88  &59.39 &60.21  &61.82  &54.55  &60.61  &57.58  &58.18  &65.45
		&\bf 66.67 \\
		\hline
	\end{tabular}
	\vspace{-10pt}
\end{table*}

\subsection{Running Time and Memory Consumption}
To compare the computational efficiency of the proposed algorithms, we record the running time and maximum memory consumption of various algorithms on the benchmark datasets and report them in Fig. \ref{Time and Memory}. The result in figure \ref{Time and Memory} is obtained by dividing by the result of A-MVSC. As we can see in Fig. \ref{Time and Memory}-(a), without the acceleration of the late-fusion mechanism and the Nystr\"om algorithm, the computational time of the early fusion version of our algorithm is comparable with most of the compared state-of-the-art algorithms. Also, with the enhancement of the mentioned techniques, the late fusion version of our proposed algorithm achieves $\Ocal{n}$ time complexity.

Moreover, from Fig. \ref{Time and Memory}-(b), we can see that the proposed ONMSC-LF has the smallest memory consumption compared to other algorithms. The results in Fig. \ref{Time and Memory} well demonstrate the efficiency of the proposed algorithm both in computation and storage.

\begin{figure}[htbp]
	\setlength{\abovecaptionskip}{0pt}
	\setlength{\belowcaptionskip}{0pt}
	\centering 
	\subfigure[Illustration of the relative running time of the compared algorithms. ]{\includegraphics[width=1\columnwidth]{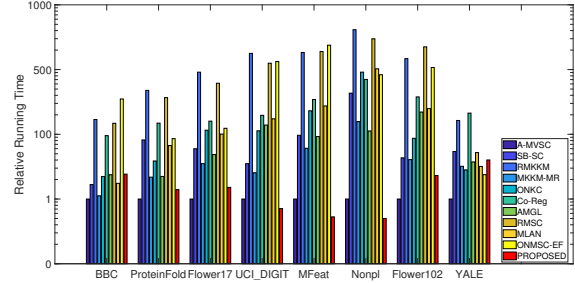}\label{time}}
	\subfigure[Illustration of the relative maximum memory consumption of the compared algorithms.]{\includegraphics[width=1\columnwidth]{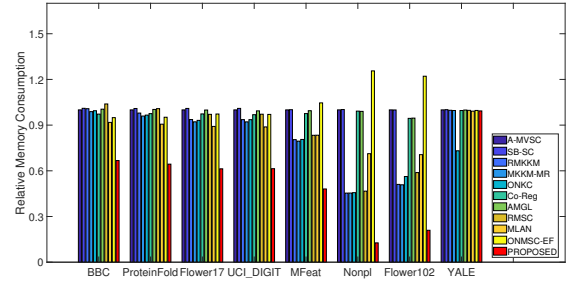}\label{memory}}
	\vspace{-2pt}
	\caption{The running time and maximum memory consumption comparison of different algorithms on eight benchmark datasets. Both the running time and the memory cost of the compared algorithms are divided by the corresponding results of A-MVSC.}
	\label{Time and Memory}
\end{figure}

\subsection{Visualization of Algorithm Performance}
To illustrate the performance of the compared algorithms more intuitively, we further adopt t-distributed stochastic neighbor embedding(t-SNE)\cite{van:2008} to visualized the acquired cluster structure of six representative datasets. Specifically, only the results of our proposed algorithm and the second best results are illustrated. As we can see from Fig. \ref{Visualization}, the proposed algorithm better reveals the underlying local geometric structure of in all the six datasets, validating its superior performance.

\begin{figure*}[htbp]
	\setlength{\abovecaptionskip}{0pt}
	\setlength{\belowcaptionskip}{0pt}
	\centering 
	\subfigure[BBCsport~(proposed)]{\includegraphics[width=0.5\columnwidth]{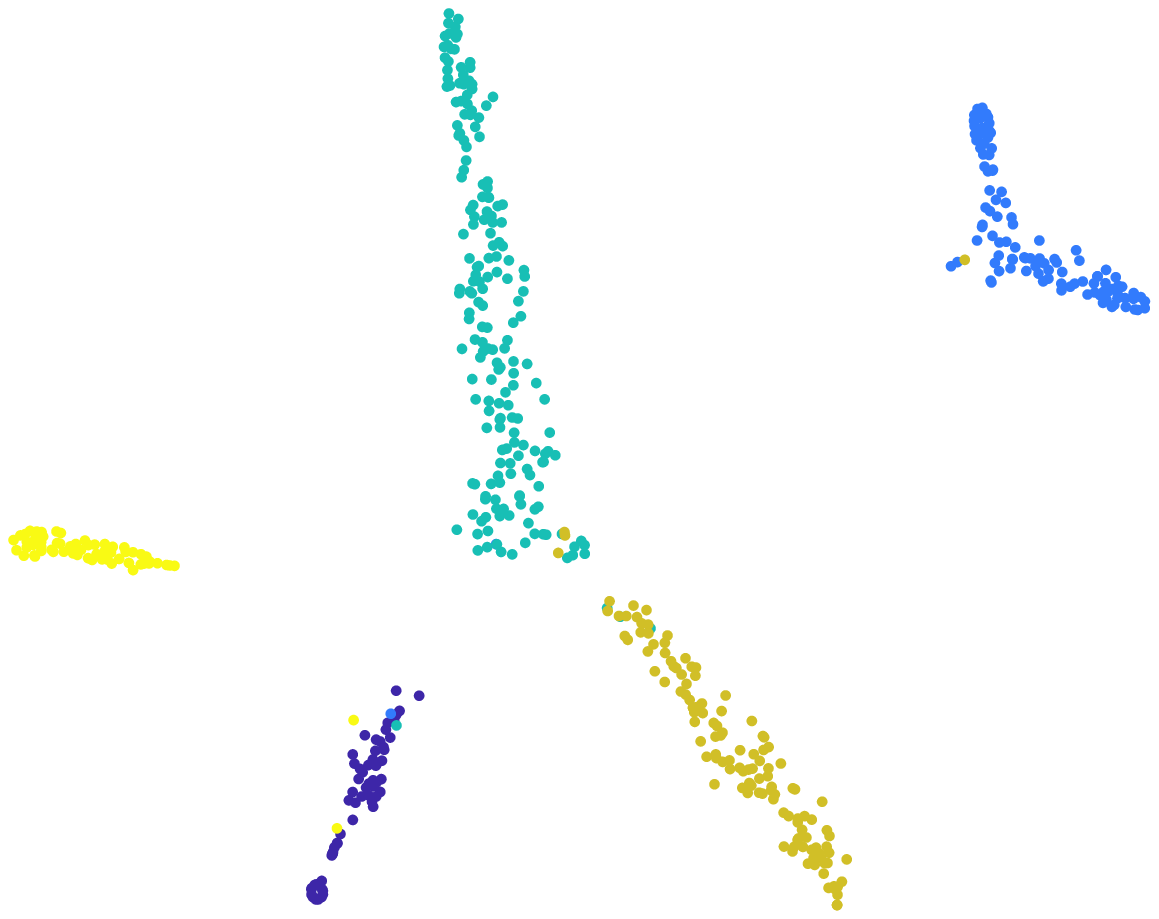}\label{bbcsport}}
	\subfigure[Flower17~(proposed)]{\includegraphics[width=0.5\columnwidth]{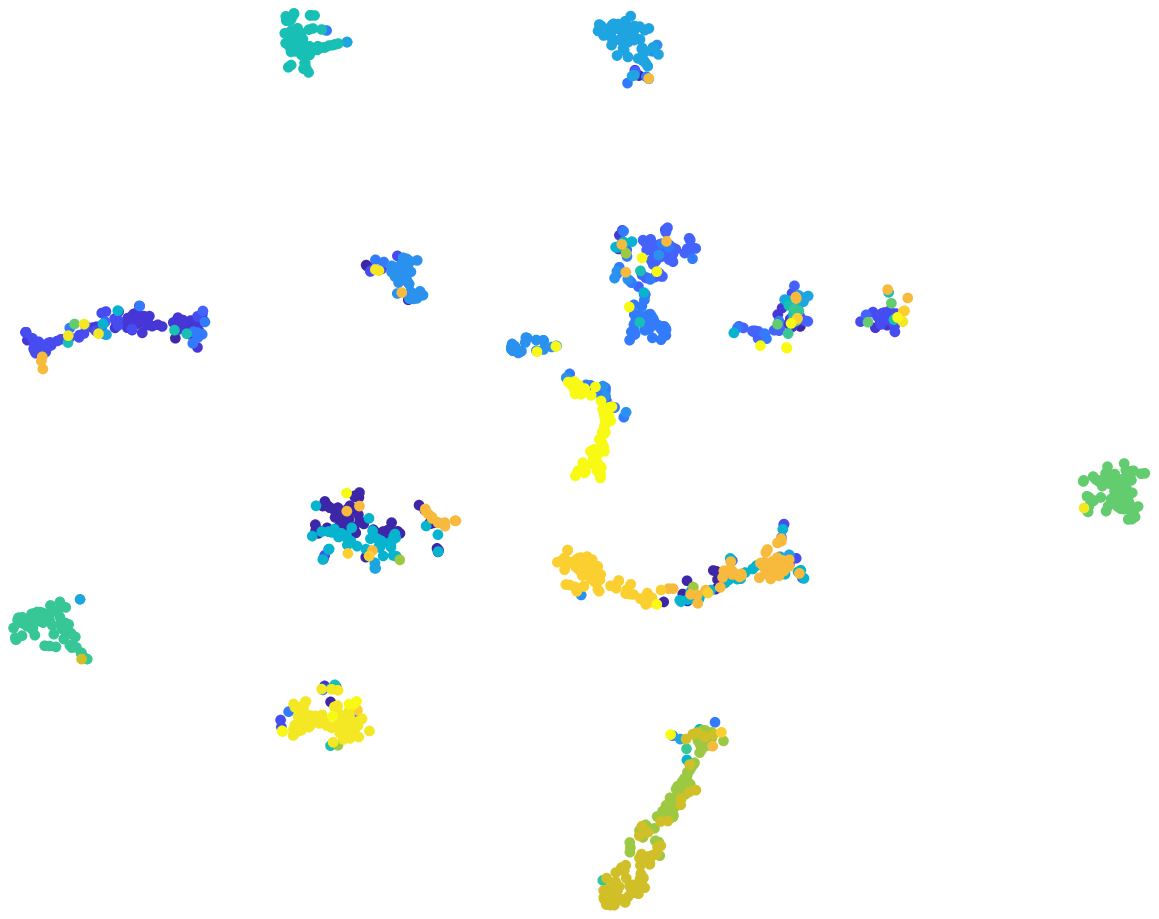}\label{flower17}}
	\subfigure[UCI-Digit~(proposed)]{\includegraphics[width=0.5\columnwidth]{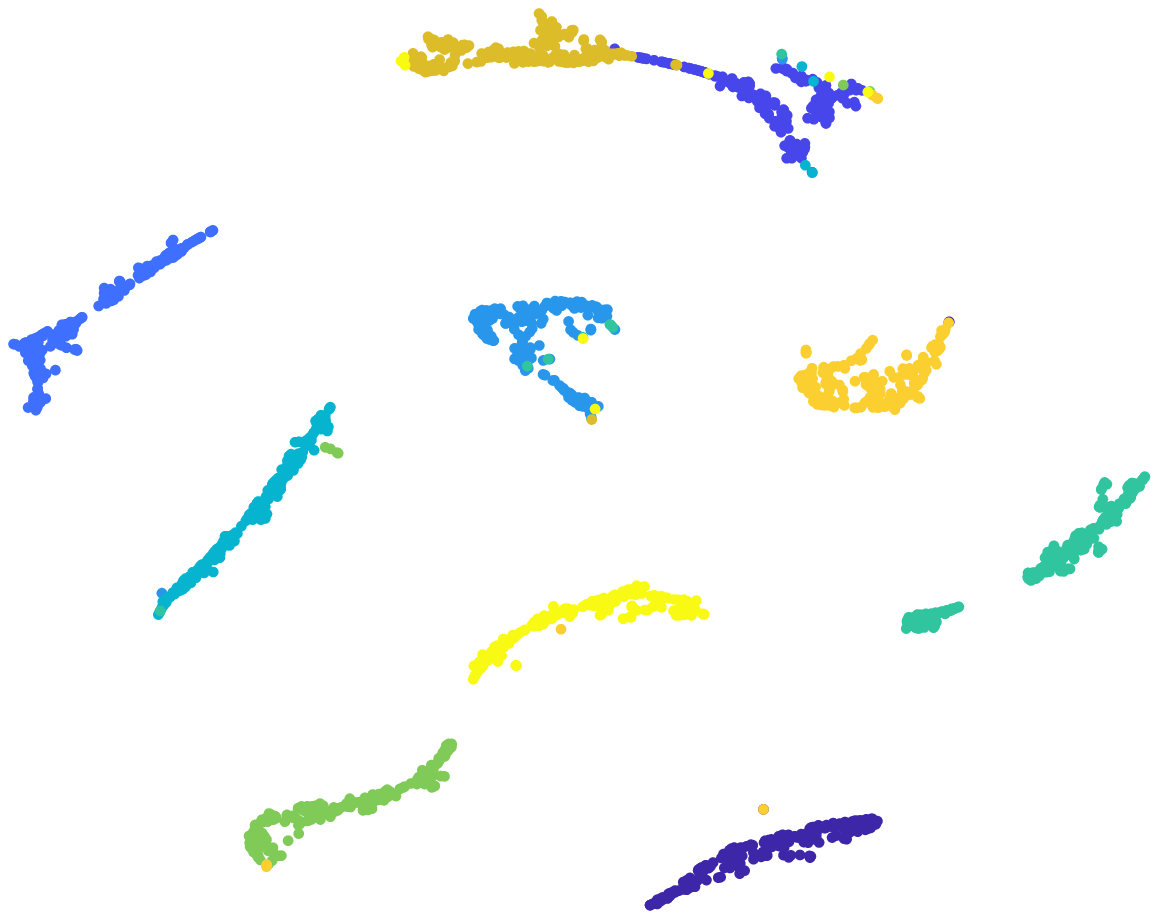}\label{UCI_DIGIT}}
	\subfigure[Mfeat~(proposed)]{\includegraphics[width=0.5\columnwidth]{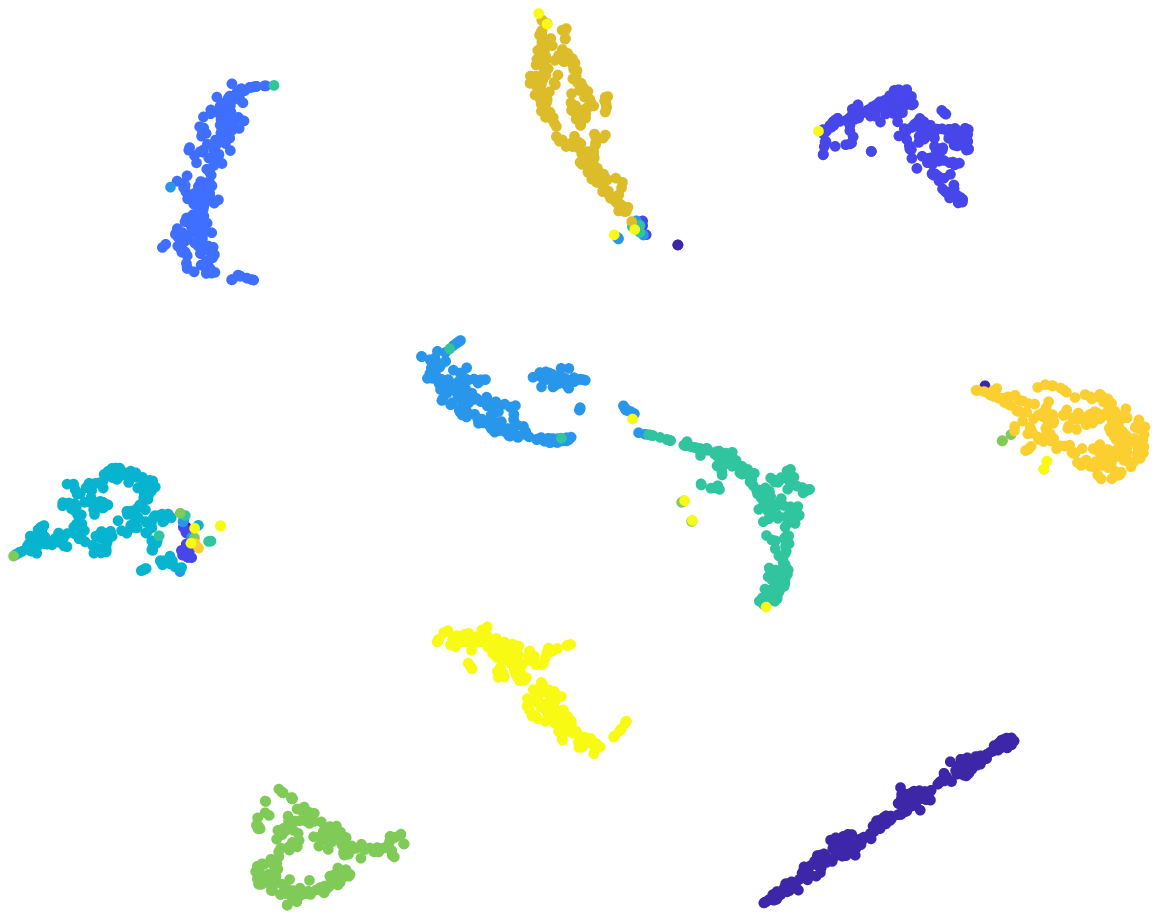}\label{mfeat}}
	\subfigure[BBCsport~(RMSC)]{\includegraphics[width=0.5\columnwidth]{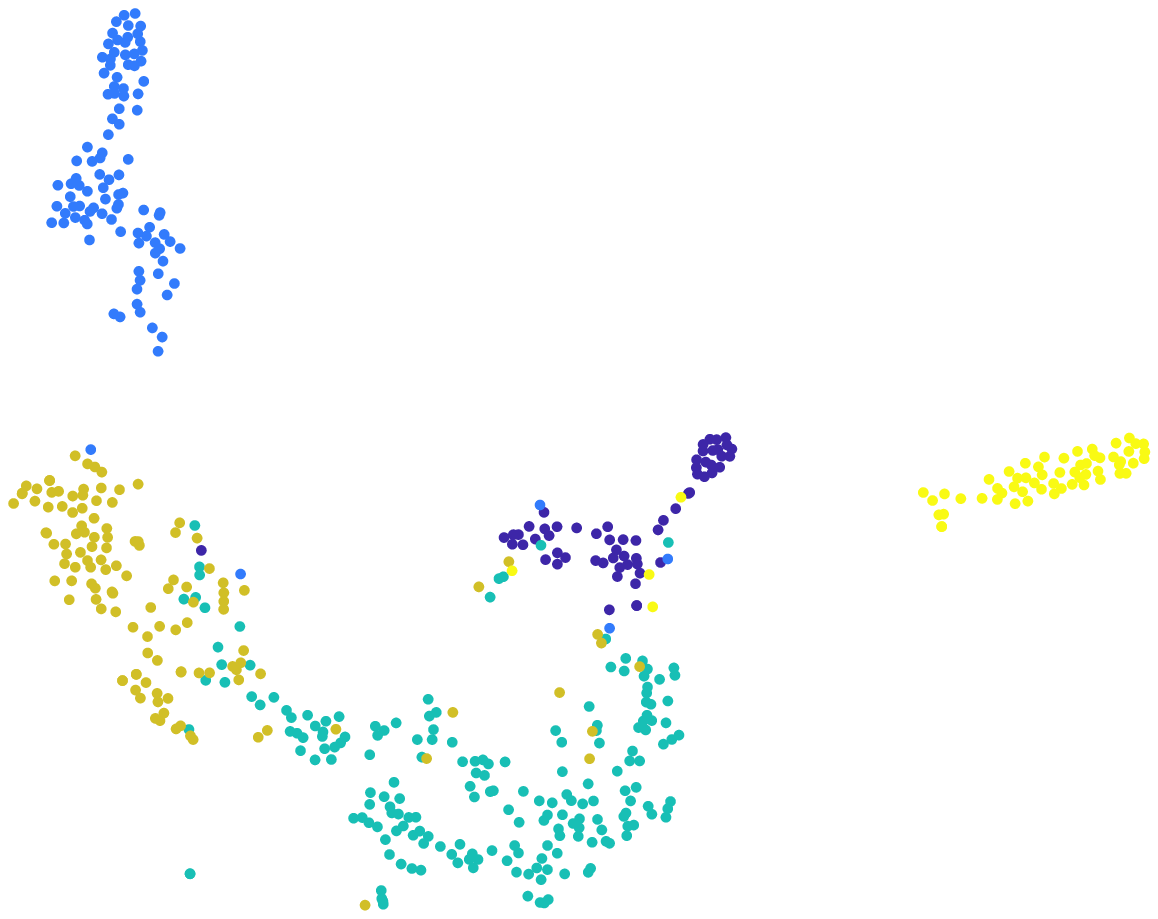}\label{bbcsport}}
	\subfigure[Flower17~(ONKC)]{\includegraphics[width=0.5\columnwidth]{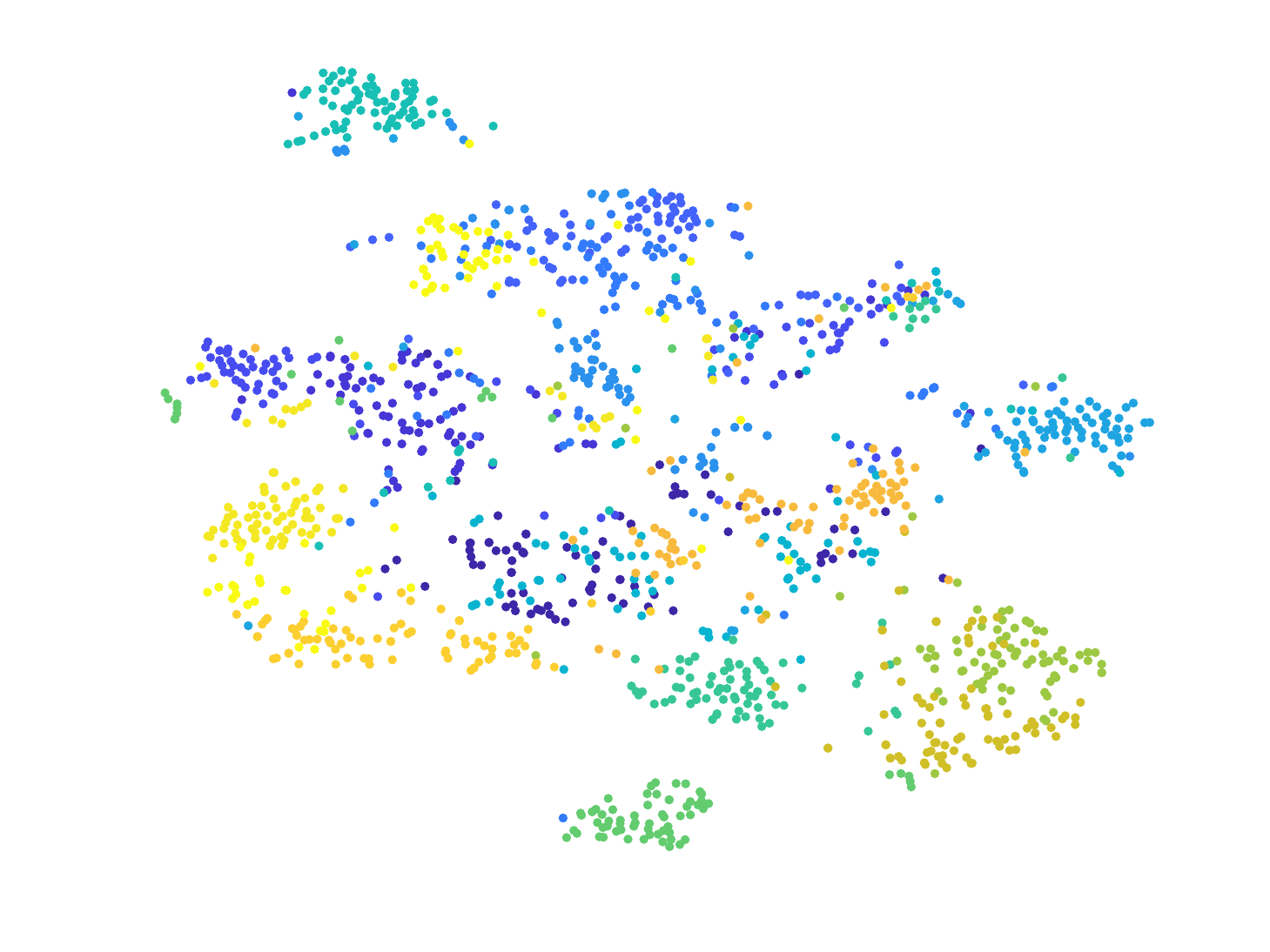}\label{flower17}}
	\subfigure[UCI-Digit~(MLAN)]{\includegraphics[width=0.5\columnwidth]{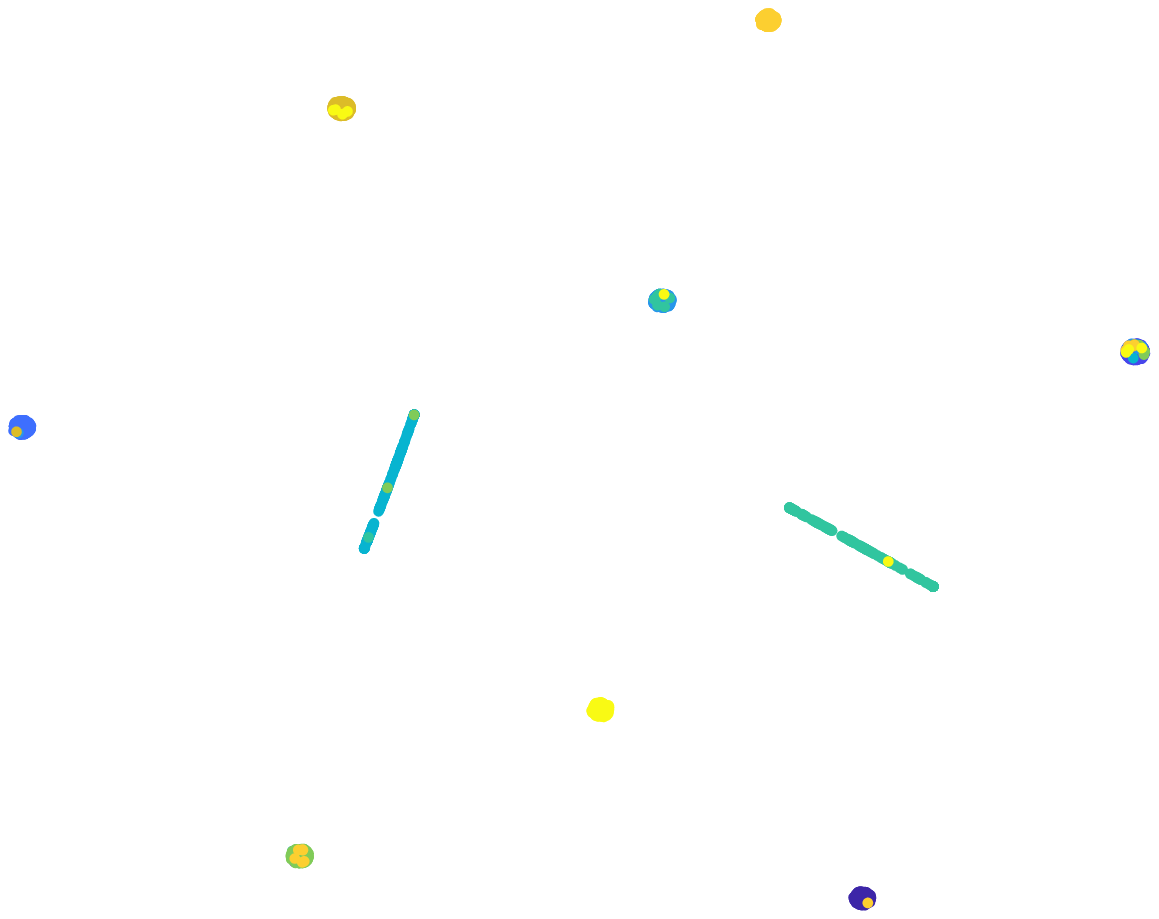}\label{UCI_DIGIT}}
	\subfigure[Mfeat~(ONKC)]{\includegraphics[width=0.5\columnwidth]{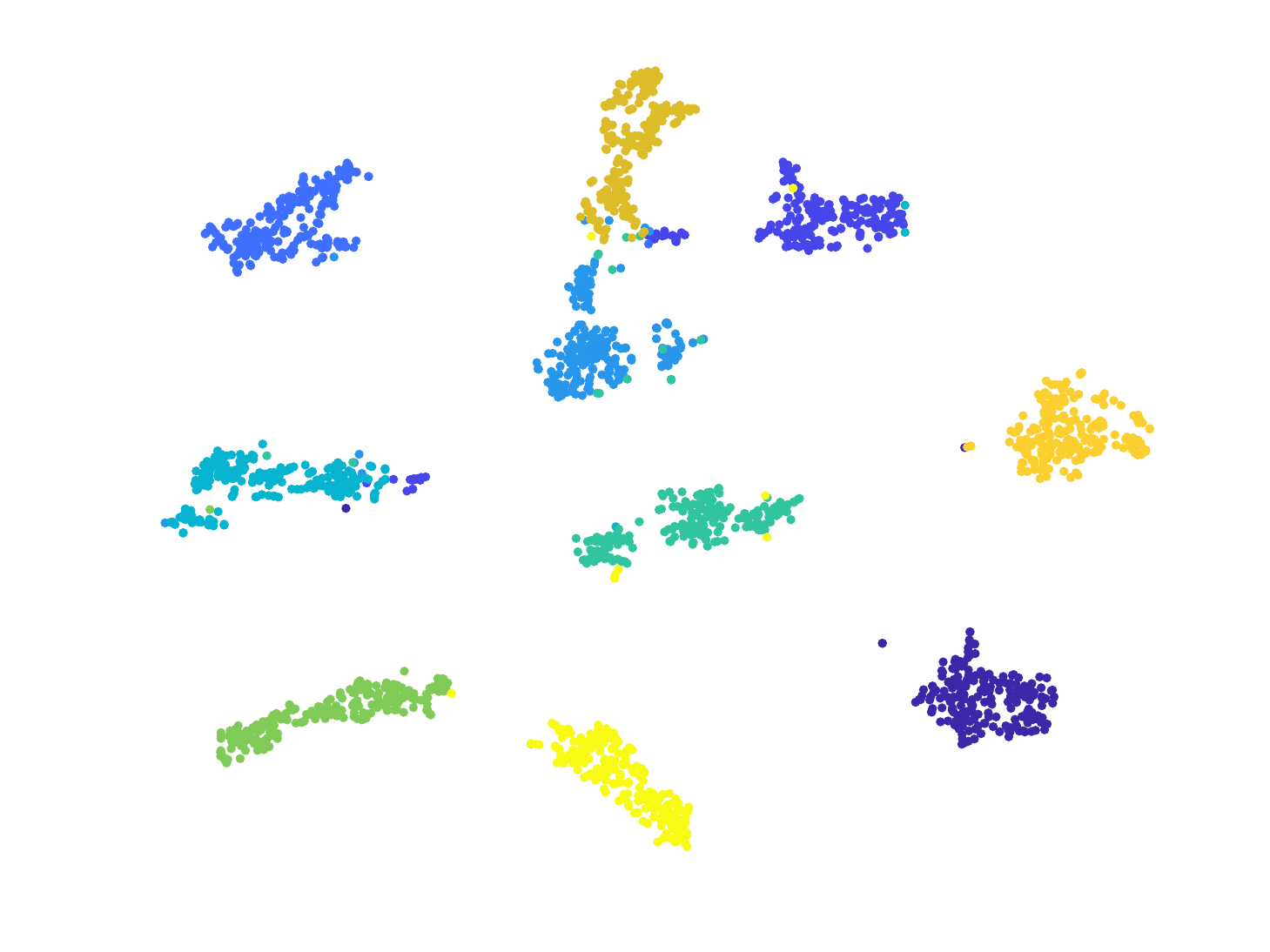}\label{mfeat}}
	\vspace{-1pt}
	\caption{Clustering structure visualization through the t-SNE algorithm \cite{van:2008}. In this figure, the results on four datasets, i.e., BBCSport, Flower17, UCIDight, Mfeat are illustrated. Among the sub-figures, the first row represents the results of our proposed algorithm, while the second row represents the best cluster structure other than our proposed algorithm.}
	\label{Visualization}
\end{figure*}

\subsection{Parameter Sensitivity and Convergence}
\textbf{Parameter Sensitivity}. The proposed ONMSC-LF introduces three hyper-parameters, i.e., the average view balancing coefficient $\lambda_1$, the diversity balancing coefficient $\lambda_2$, and the neighbor number $N$ for affinity matrix construction. To test the sensitivity of the proposed algorithm against these three parameters, we fix one parameter and tune the other in a large range. Parameter sensitivity w.r.t. $\lambda_1$ and $\lambda_2$ are displayed in Fig. \ref{parameter_sensitivity}. The performance variation against $N$ is illustrated in Fig. \ref{neighbor_sensitivity}. In this figure, the performance variation w.r.t. different $N$ is compared with the second-best state-of-the-art algorithms on eight datasets. From these figures, we observe that i) all the parameters are effective in improving the algorithm performance; ii) the proposed algorithm is stable against the these parameters that it achieves good performance in a wide range of parameter settings. iii) the proposed algorithm is relatively sensitive to the neighbor numbers.

\begin{figure*}[htbp]
\vspace{-15pt}
	\setlength{\abovecaptionskip}{0pt}
	\setlength{\belowcaptionskip}{0pt}
	\centering 
	\subfigure[BBC~($\lambda_1$,$\lambda_2$)]{\includegraphics[width=0.666\columnwidth]{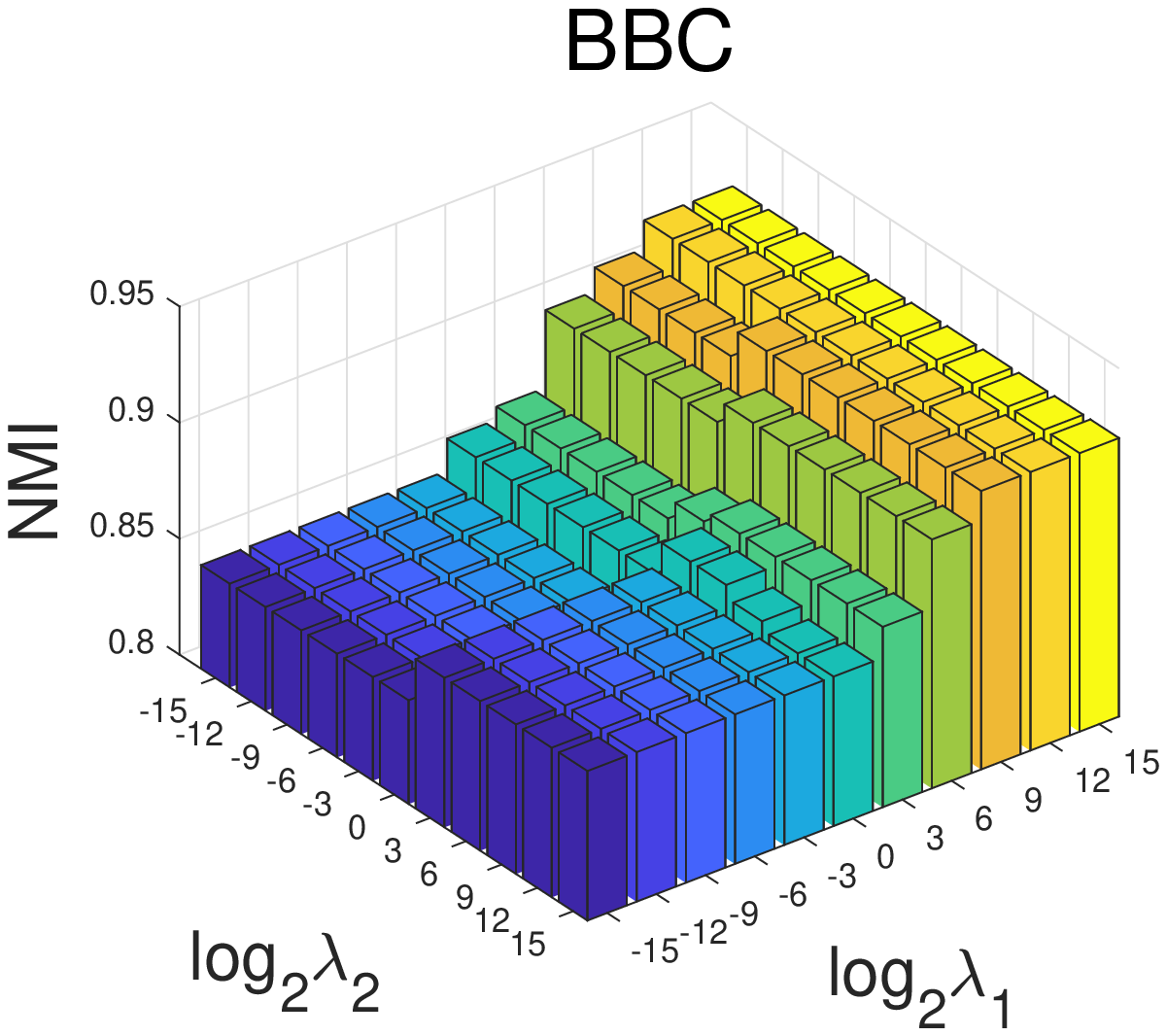}\label{BBC_all}}
	\subfigure[ProteinFold~($\lambda_1$,$\lambda_2$)]{\includegraphics[width=0.666\columnwidth]{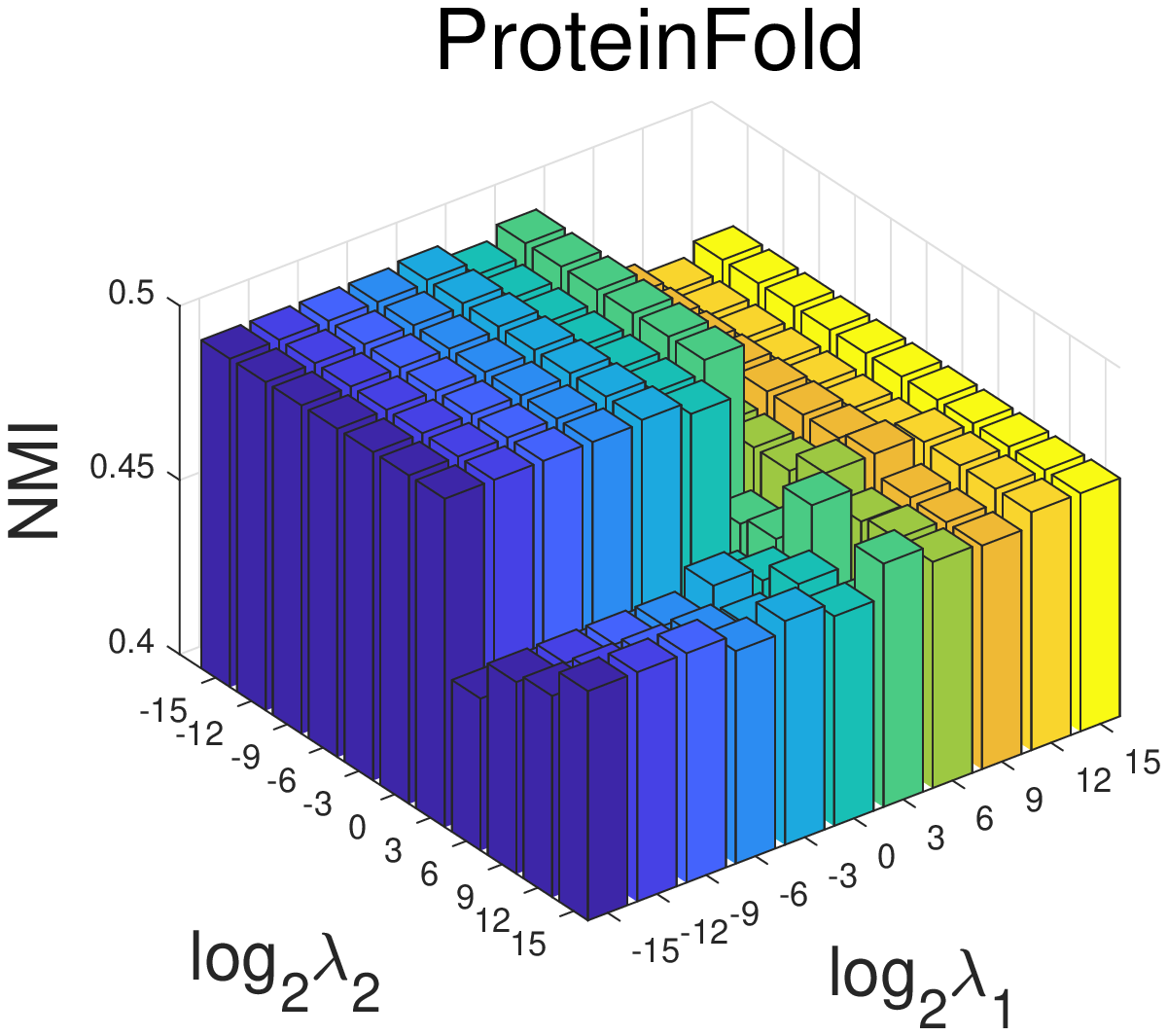}\label{ProteinFold_all}}
	\subfigure[Flower17~($\lambda_1$,$\lambda_2$)]{\includegraphics[width=0.666\columnwidth]{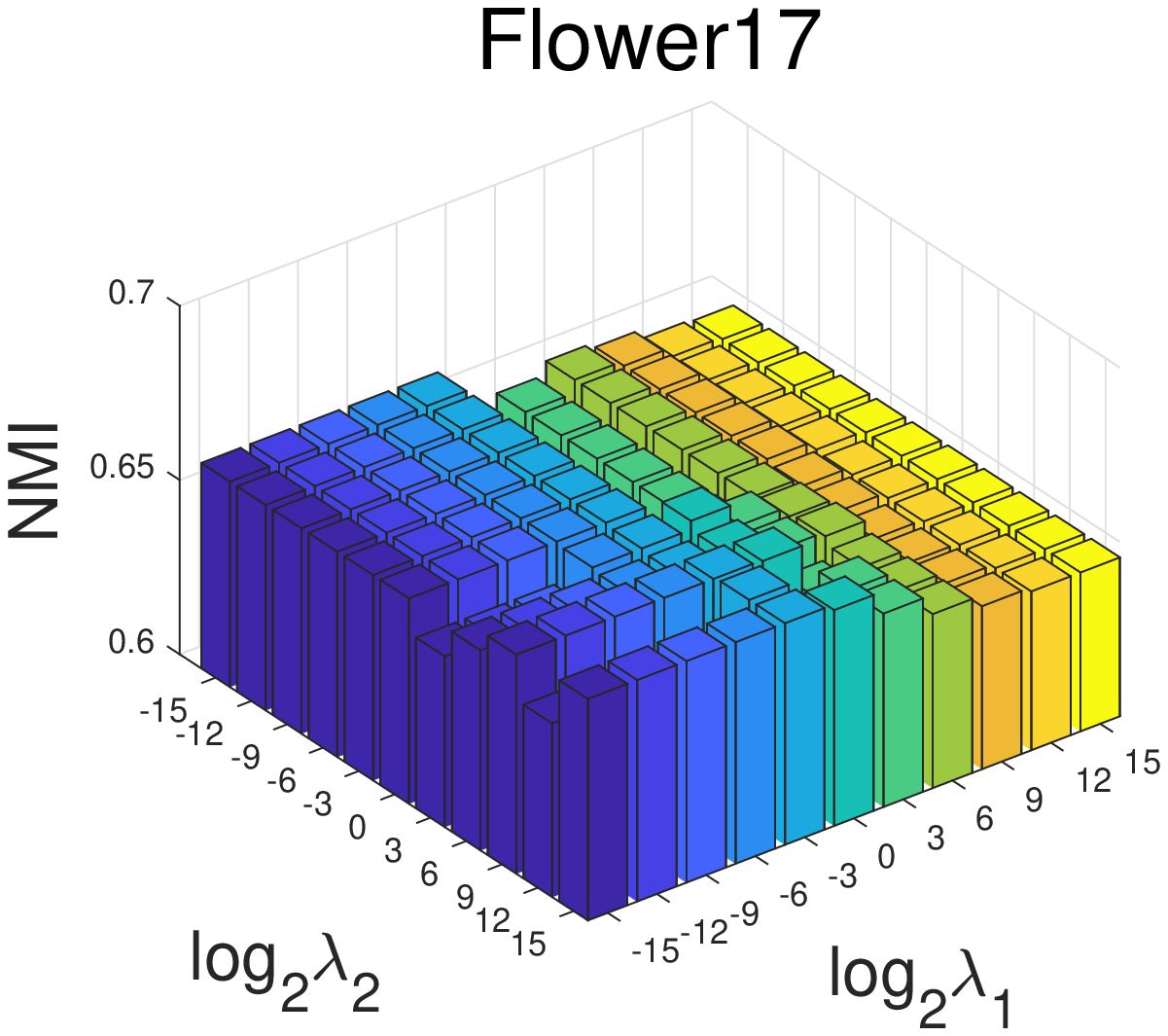}\label{flower17_all}}
	\subfigure[UCI-Digit~($\lambda_1$,$\lambda_2$)]{\includegraphics[width=0.666\columnwidth]{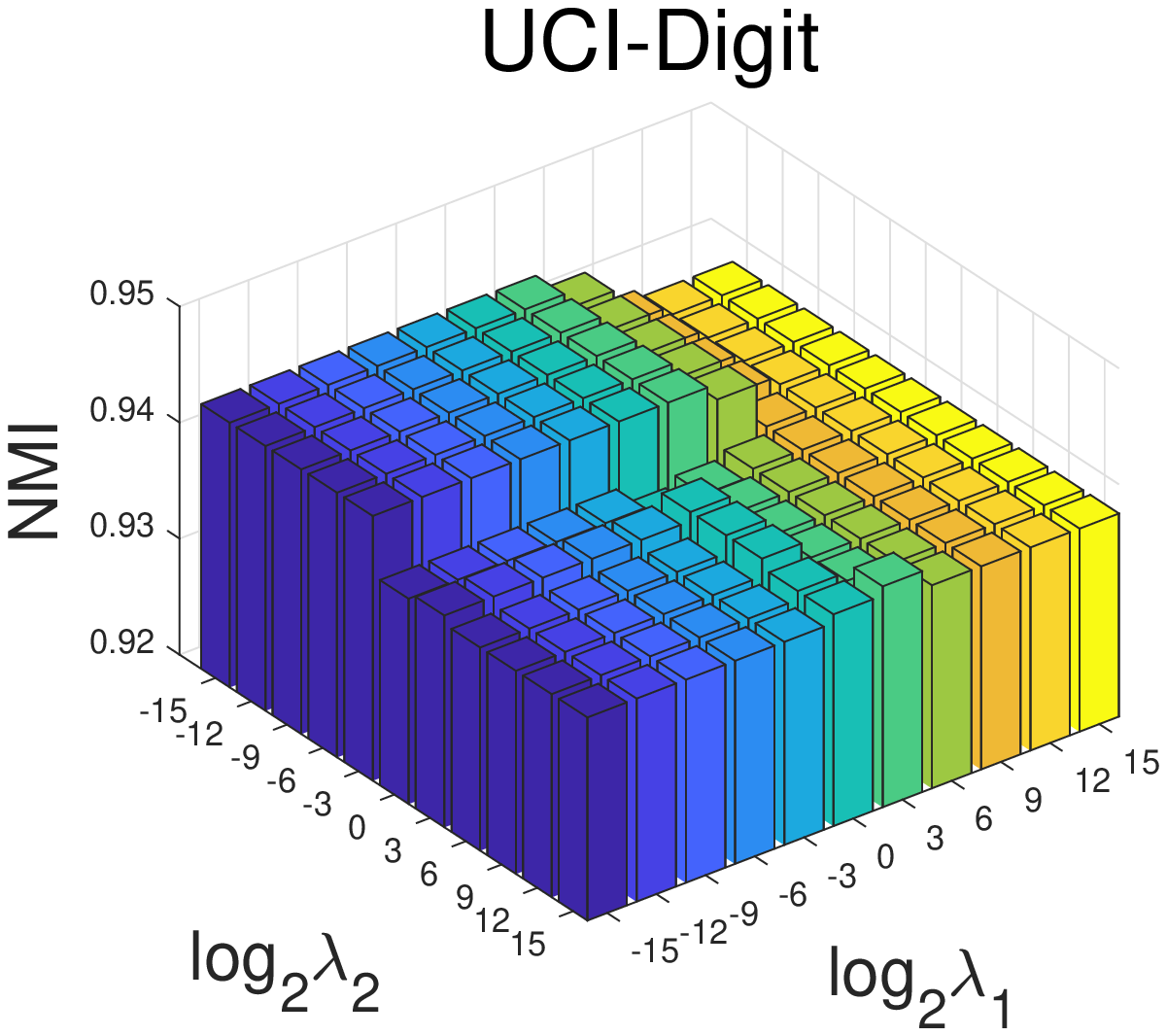}\label{UCI_DIGIT_all}}
	\subfigure[Flower102~($\lambda_1$,$\lambda_2$)]{\includegraphics[width=0.666\columnwidth]{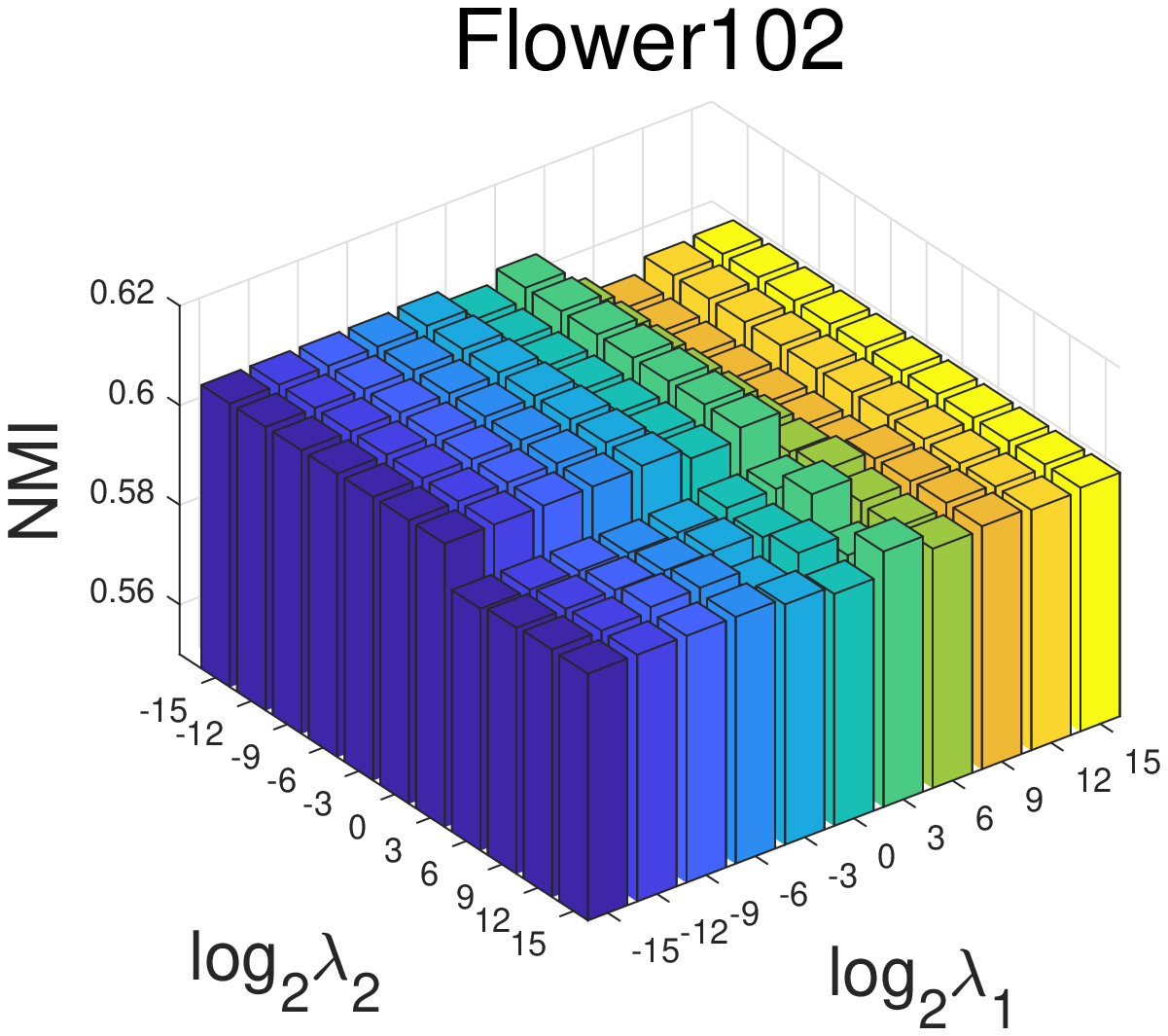}\label{flower102_all}}
	\subfigure[YALE~($\lambda_1$,$\lambda_2$)]{\includegraphics[width=0.666\columnwidth]{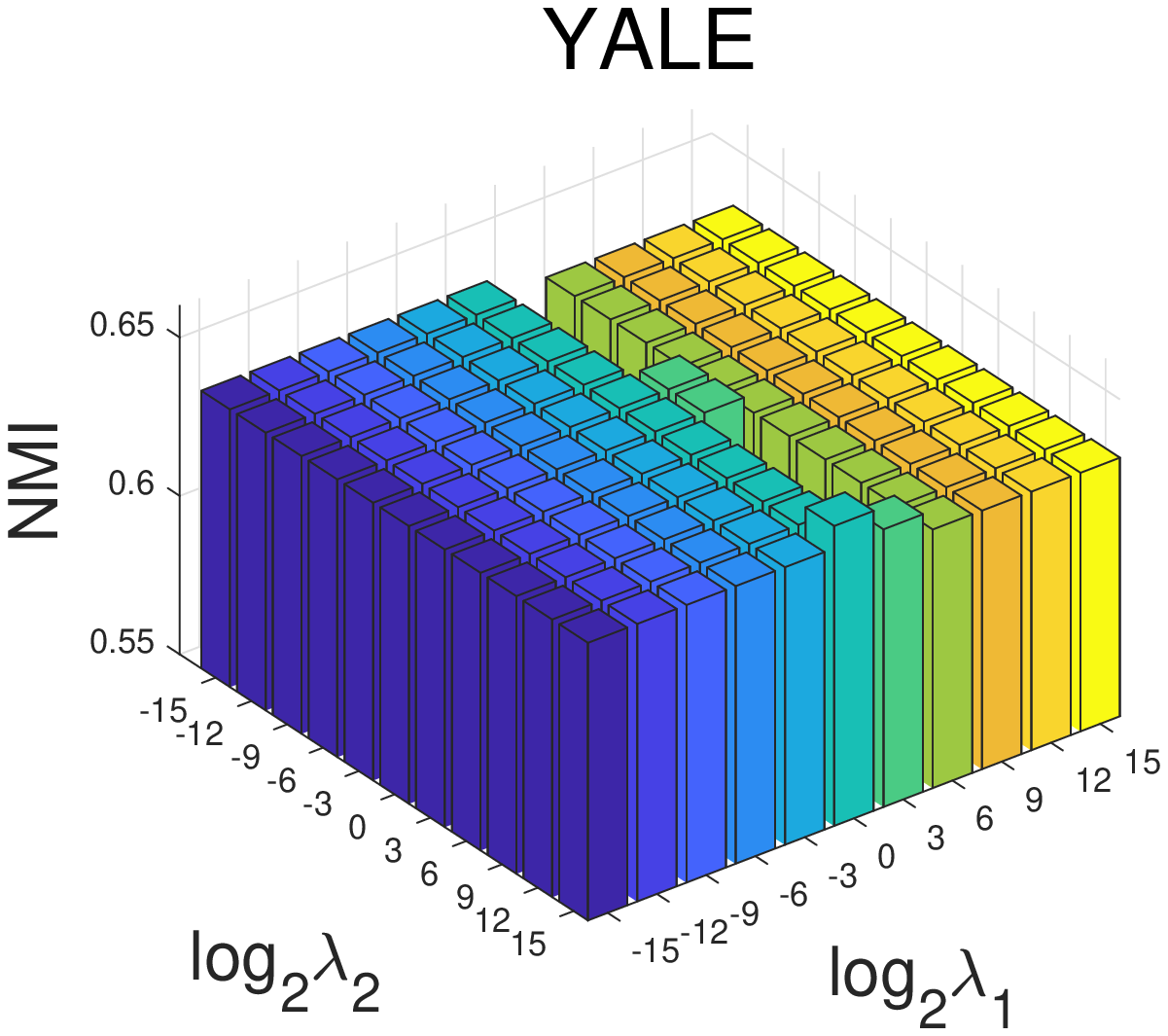}\label{yale_all}}
	\caption{Illustration of parameter sensitivity against the two hyper-parameters $\lambda_1$ and $\lambda_2$. In the experiment, both $\lambda_1$ and $\lambda_2$ are tested in the range $[2^{-15}, 2^{-12}, \cdots, 2^{15}]$. Among the figures, sub-figure a) - f) are corresponding to the performance variation of NMI on BBCSport, ProteinFold, Flower17, UCI-Digit, Mfeat, Nonpl, Flower102, and YALE respectively.}\label{parameter_sensitivity}
\vspace{-10pt}
\end{figure*}

\begin{figure*}[htbp]
\vspace{-25pt}
	\setlength{\abovecaptionskip}{0pt}
	\setlength{\belowcaptionskip}{0pt}
	\centering 
	\subfigure[BBC~(Neighbor))]{\includegraphics[width=0.666\columnwidth]{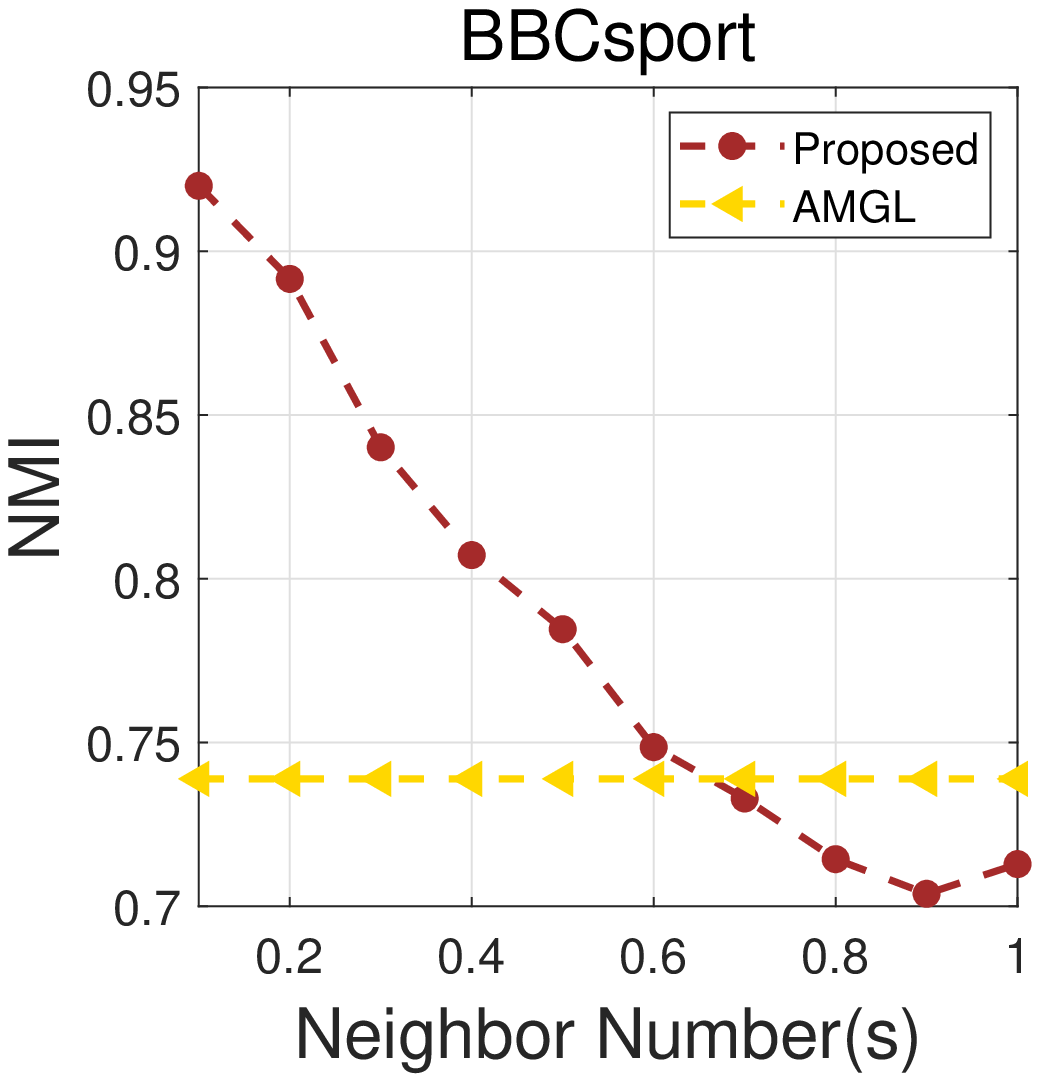}\label{BBC_NN}}
	\subfigure[ProteinFold~(Neighbor)]{\includegraphics[width=0.666\columnwidth]{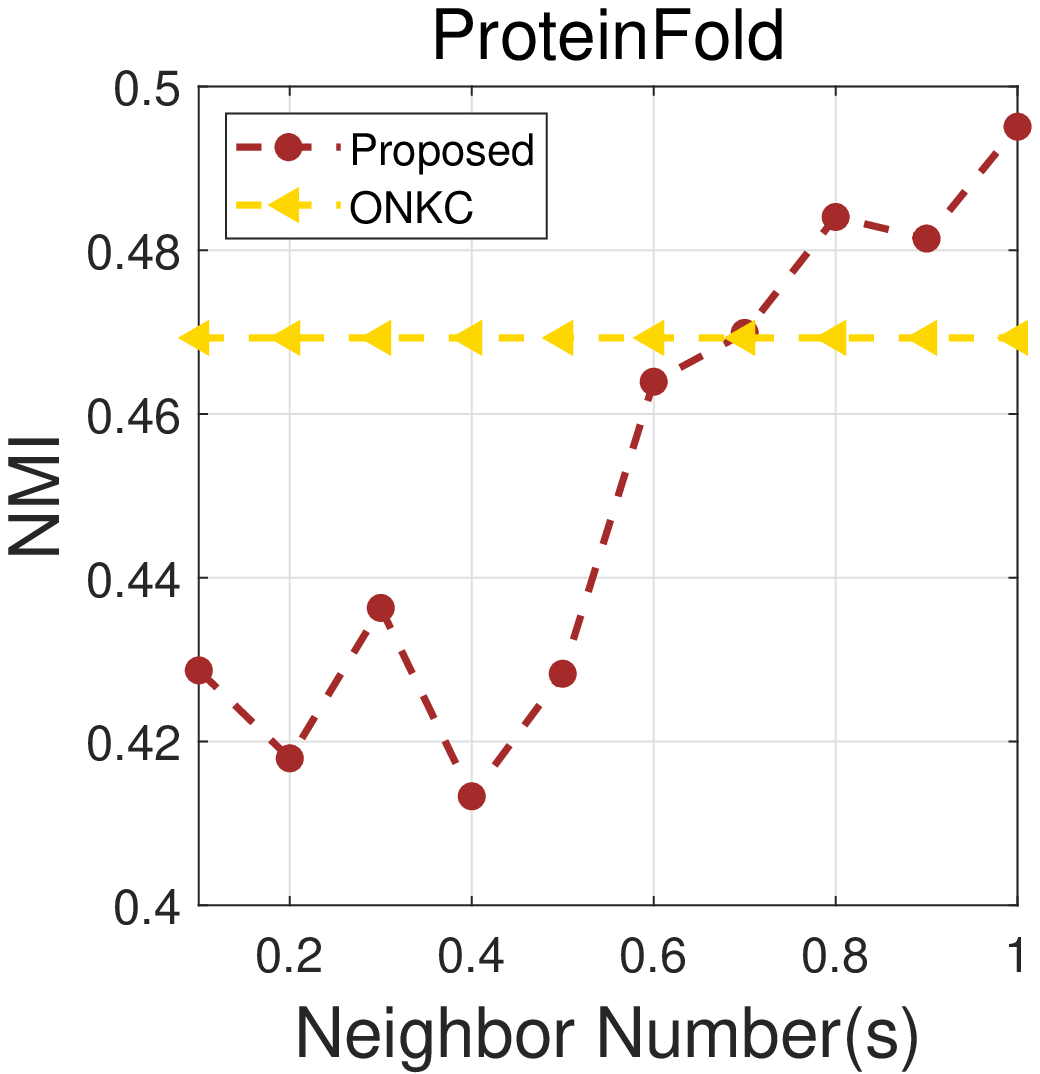}\label{ProteinFold_NN}}
	\subfigure[Flower17~(Neighbor)]{\includegraphics[width=0.666\columnwidth]{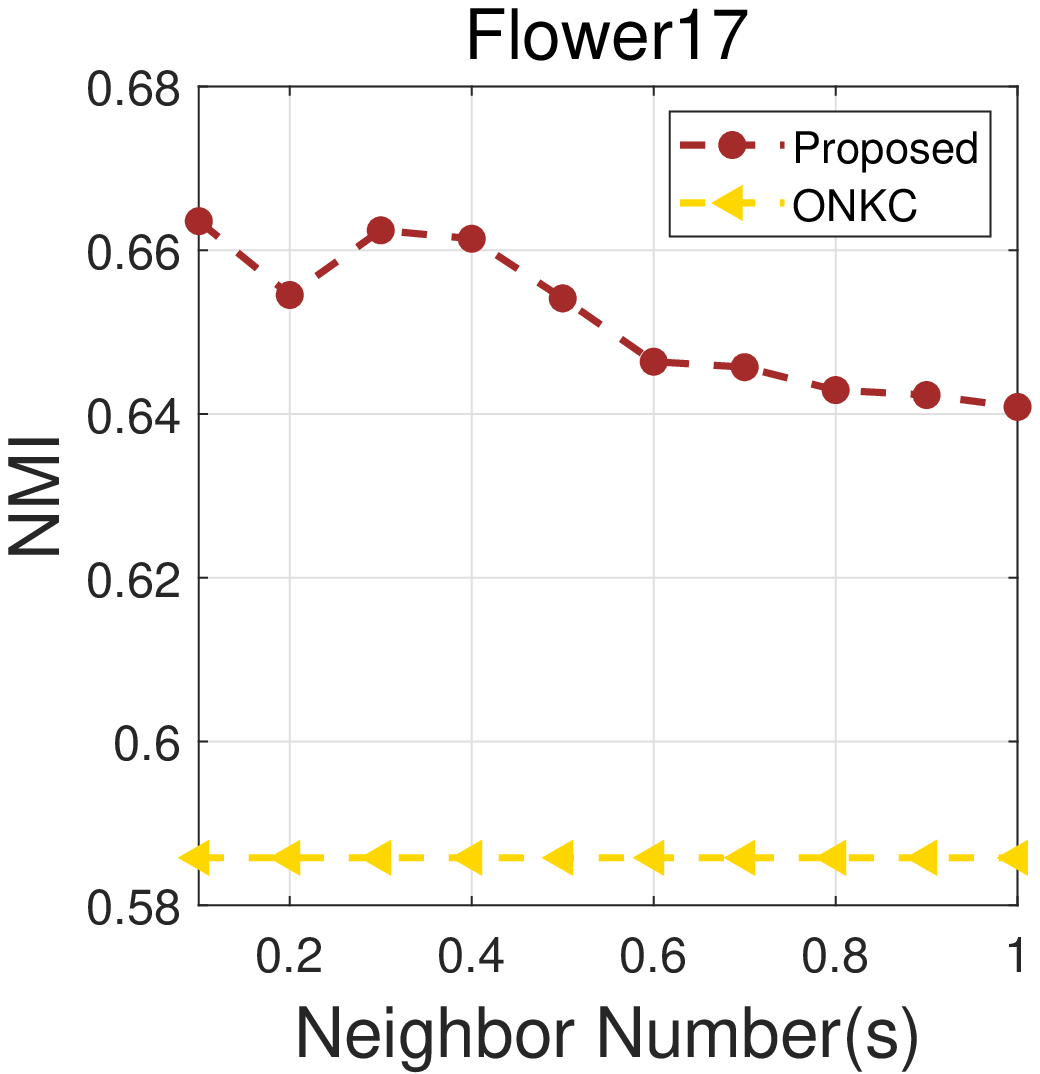}\label{flower17_NN}}
	\subfigure[UCI-Digit~(Neighbor)]{\includegraphics[width=0.666\columnwidth]{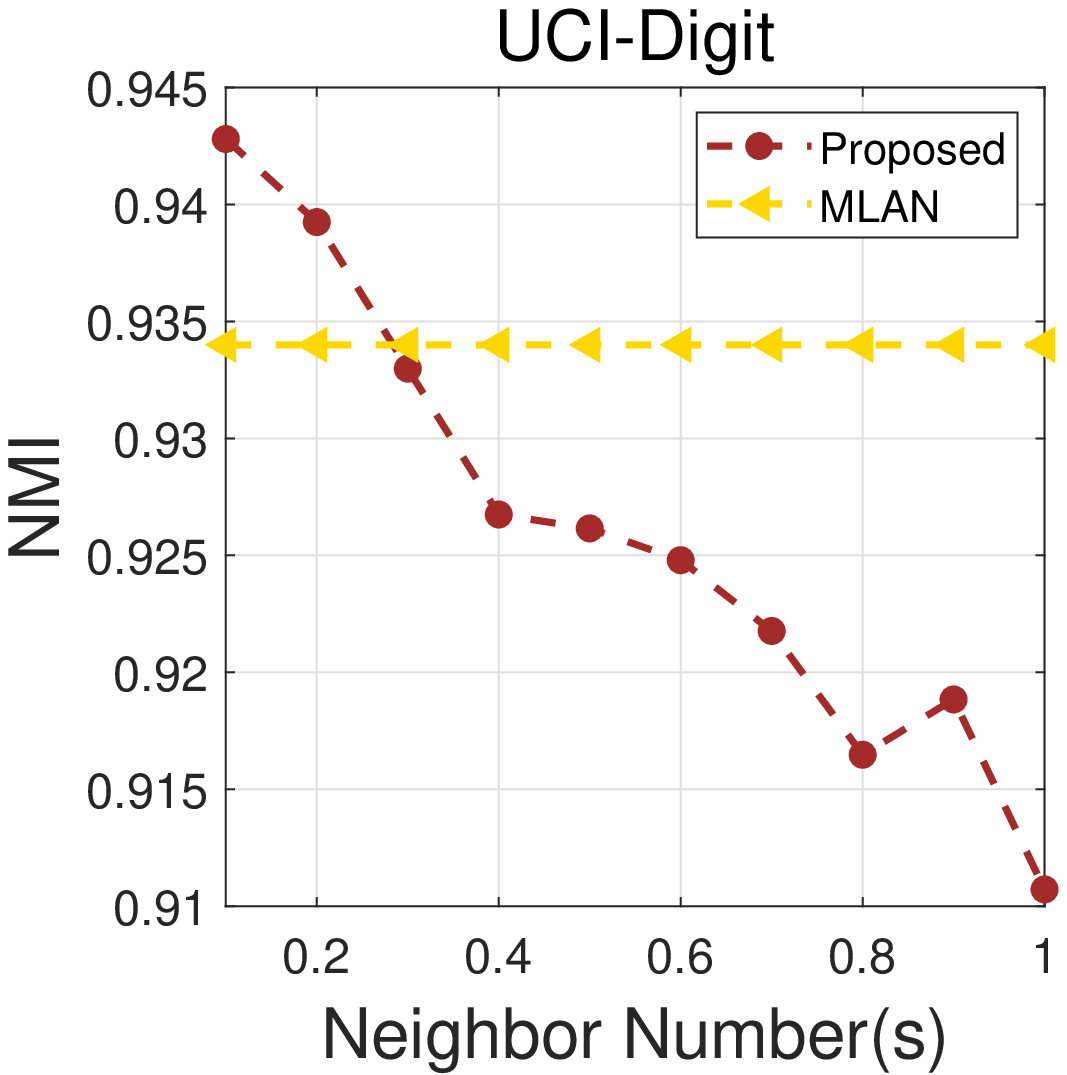}\label{UCI_DIGIT_NN}}
	\subfigure[Flower102~(Neighbor)]{\includegraphics[width=0.666\columnwidth]{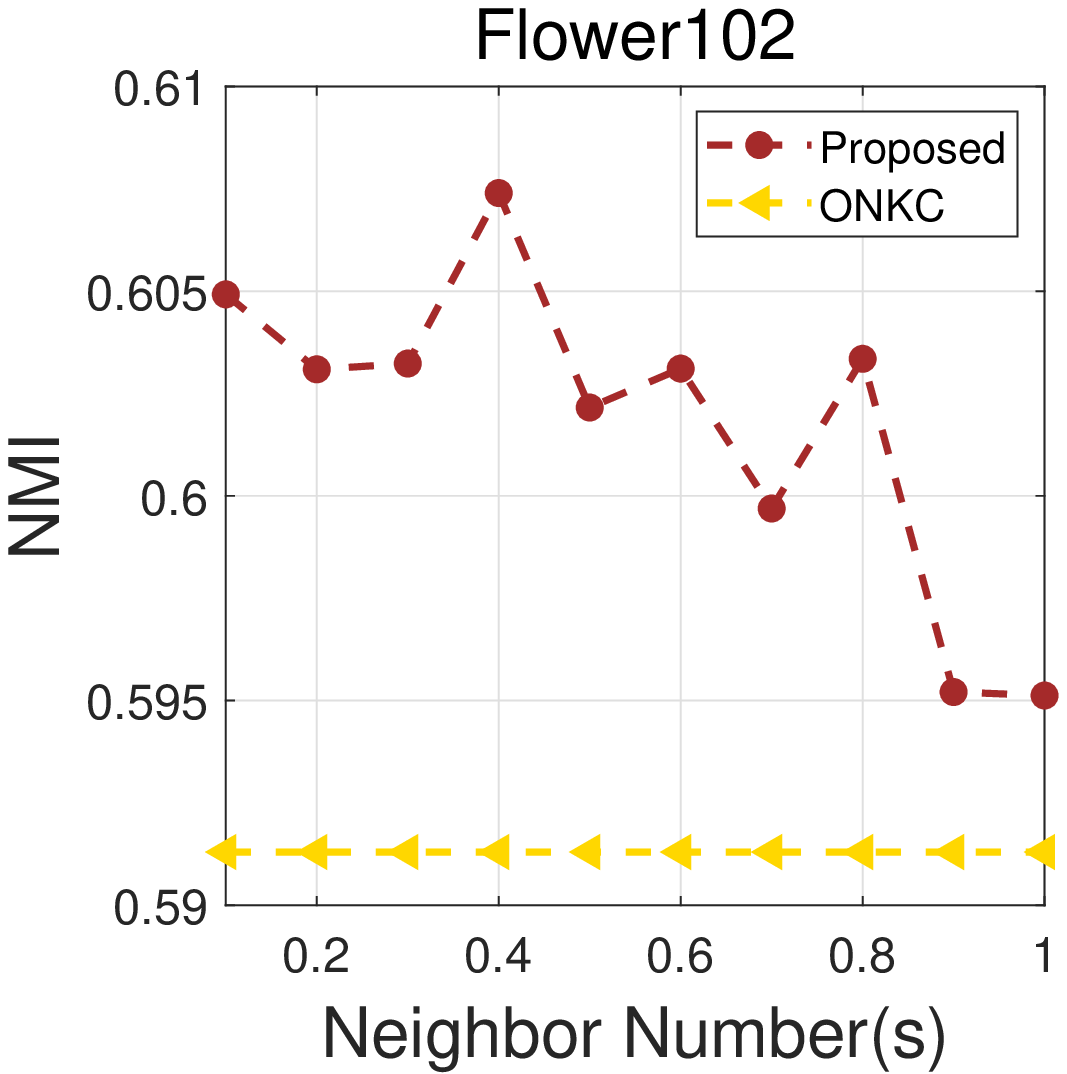}\label{flower102_NN}}
	\subfigure[YALE~(Neighbor)]{\includegraphics[width=0.666\columnwidth]{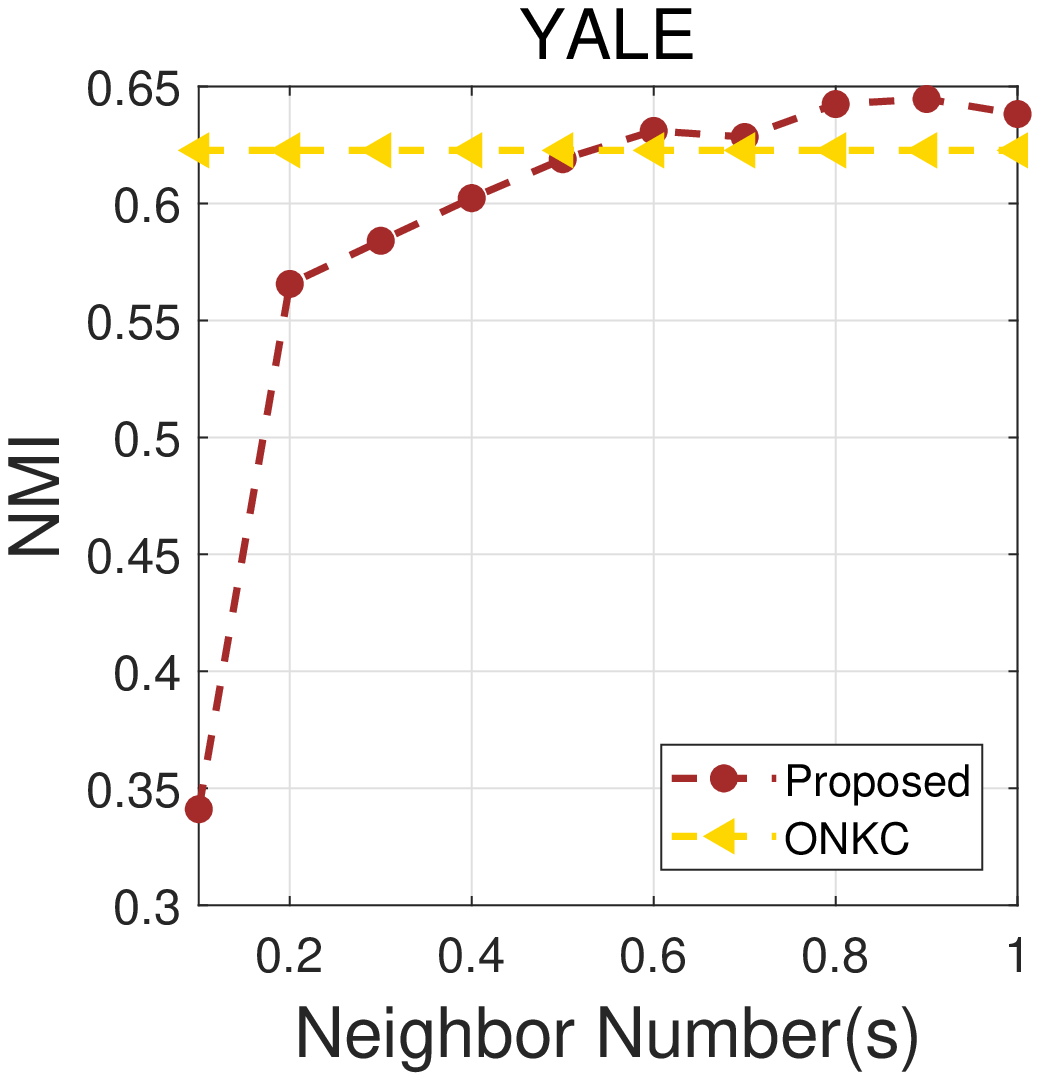}\label{yale_NN}}
	\caption{Illustration of parameter sensitivity against neighbor number $N$. Sub-figure a) - f) are corresponding to the performance variation of NMI on BBCSport, ProteinFold, Flower17, UCI-Digit, Flower102, and YALE, respectively. The brown curves record the results of the proposed algorithm, while the yellow lines are corresponding to the second-best performance on the corresponding dataset.}\label{neighbor_sensitivity}
\vspace{-25pt}
\end{figure*}

\begin{figure*}[htbp]
	\setlength{\abovecaptionskip}{0pt}
	\setlength{\belowcaptionskip}{0pt}
	\centering
	\subfigure[BBCSport]{\includegraphics[width=0.47\columnwidth]{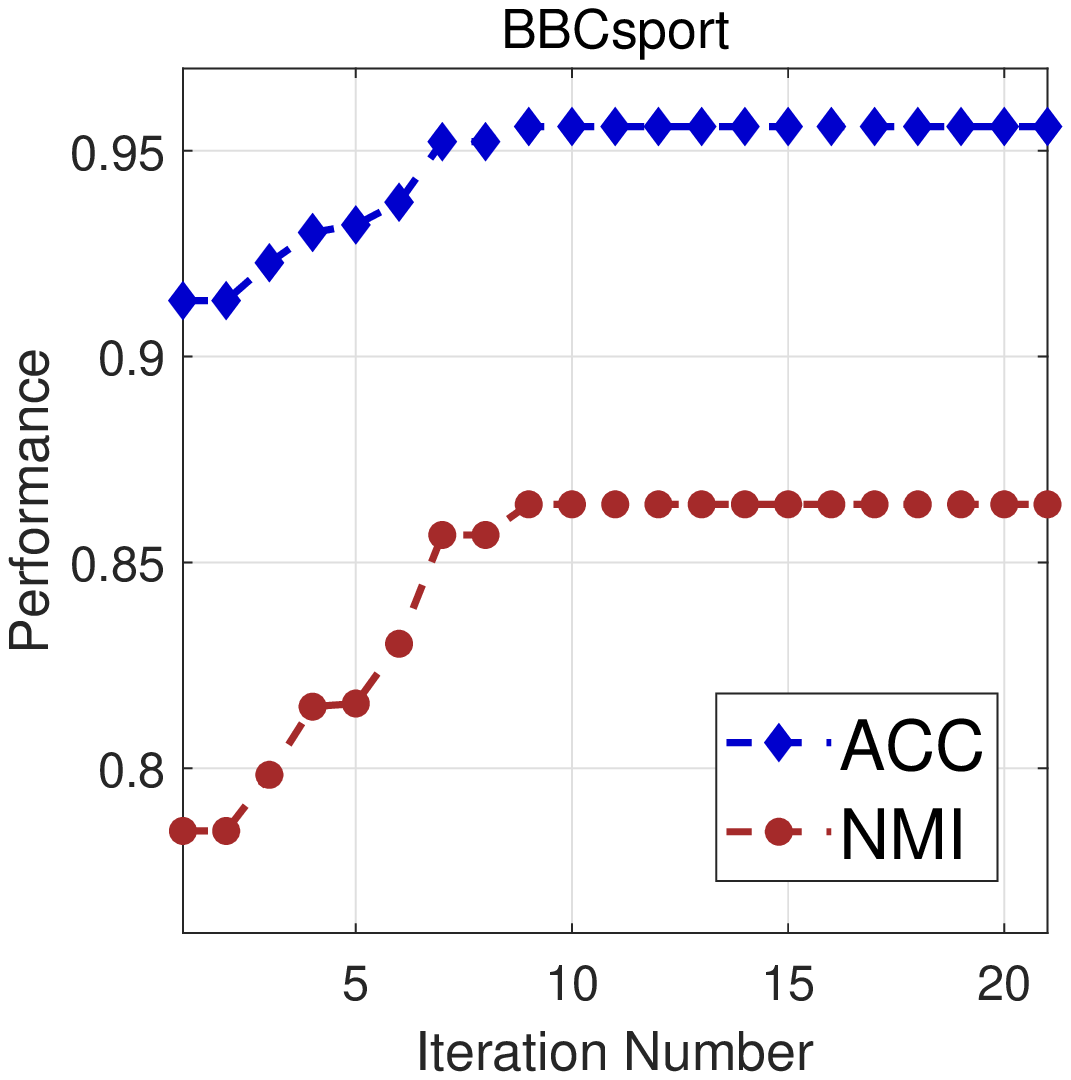}\label{bbc_result}}
	\subfigure[BBCSport]{\includegraphics[width=0.47\columnwidth]{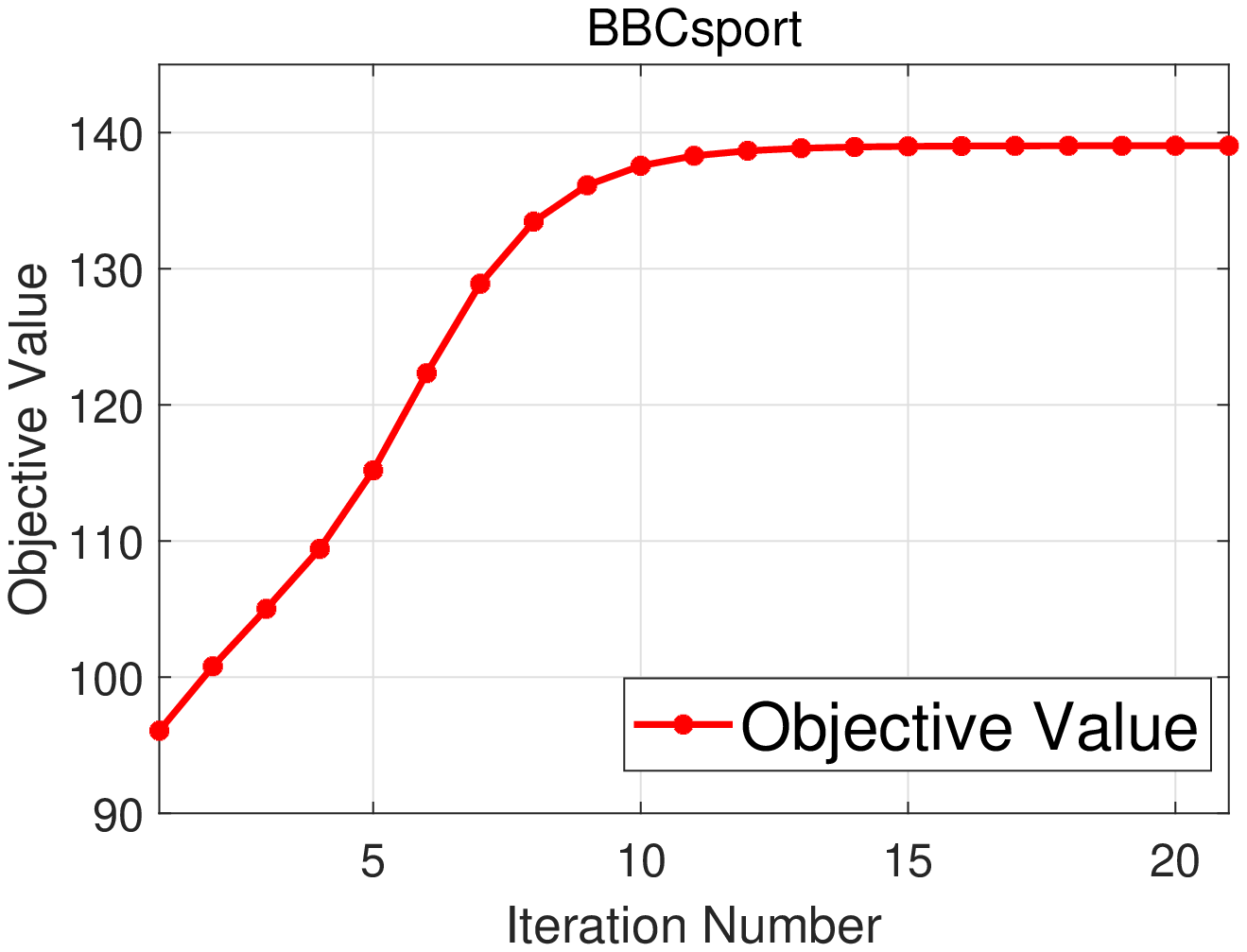}\label{bbc_loss}}
	\subfigure[ProteinFold]{\includegraphics[width=0.47\columnwidth]{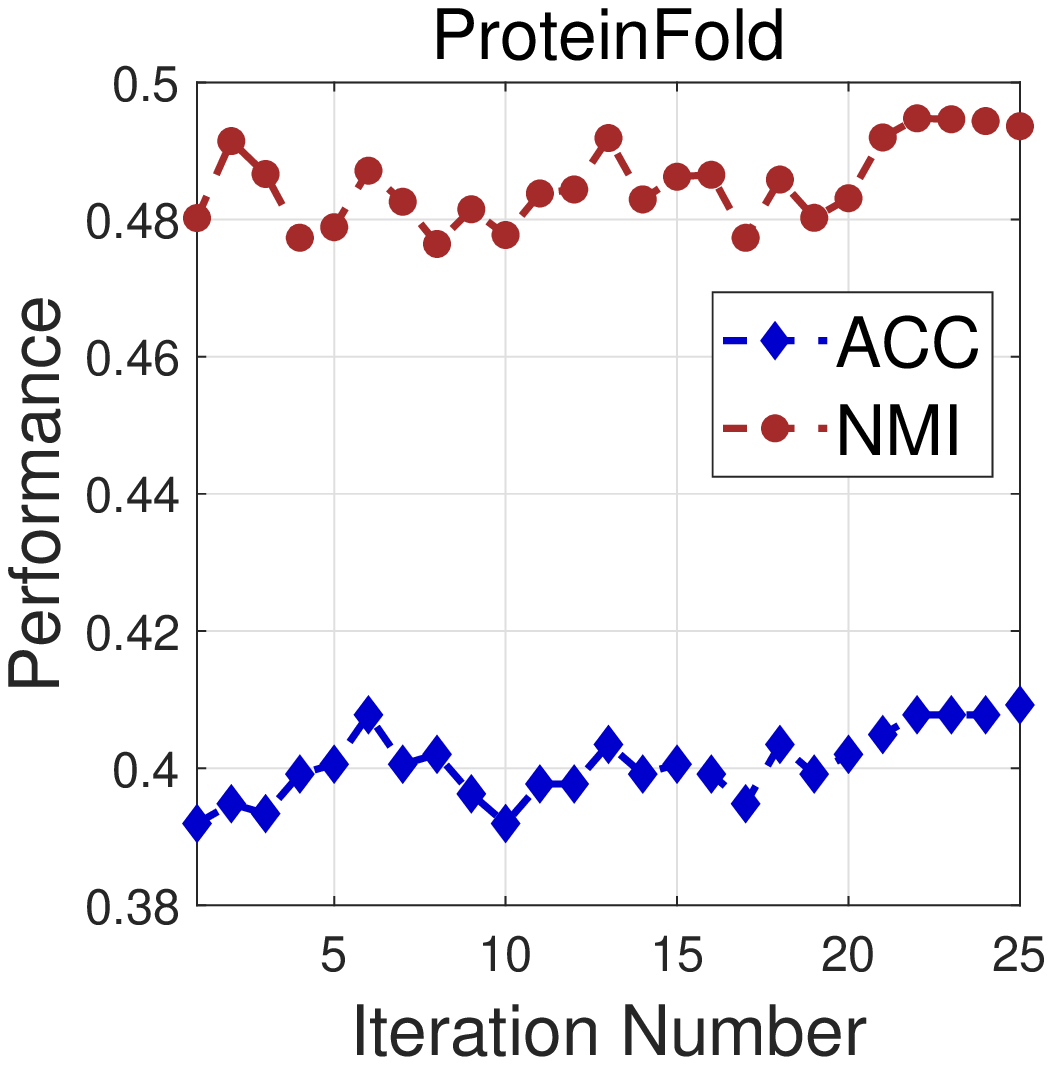}\label{proteinfold_result}}
	\subfigure[ProteinFold]{\includegraphics[width=0.47\columnwidth]{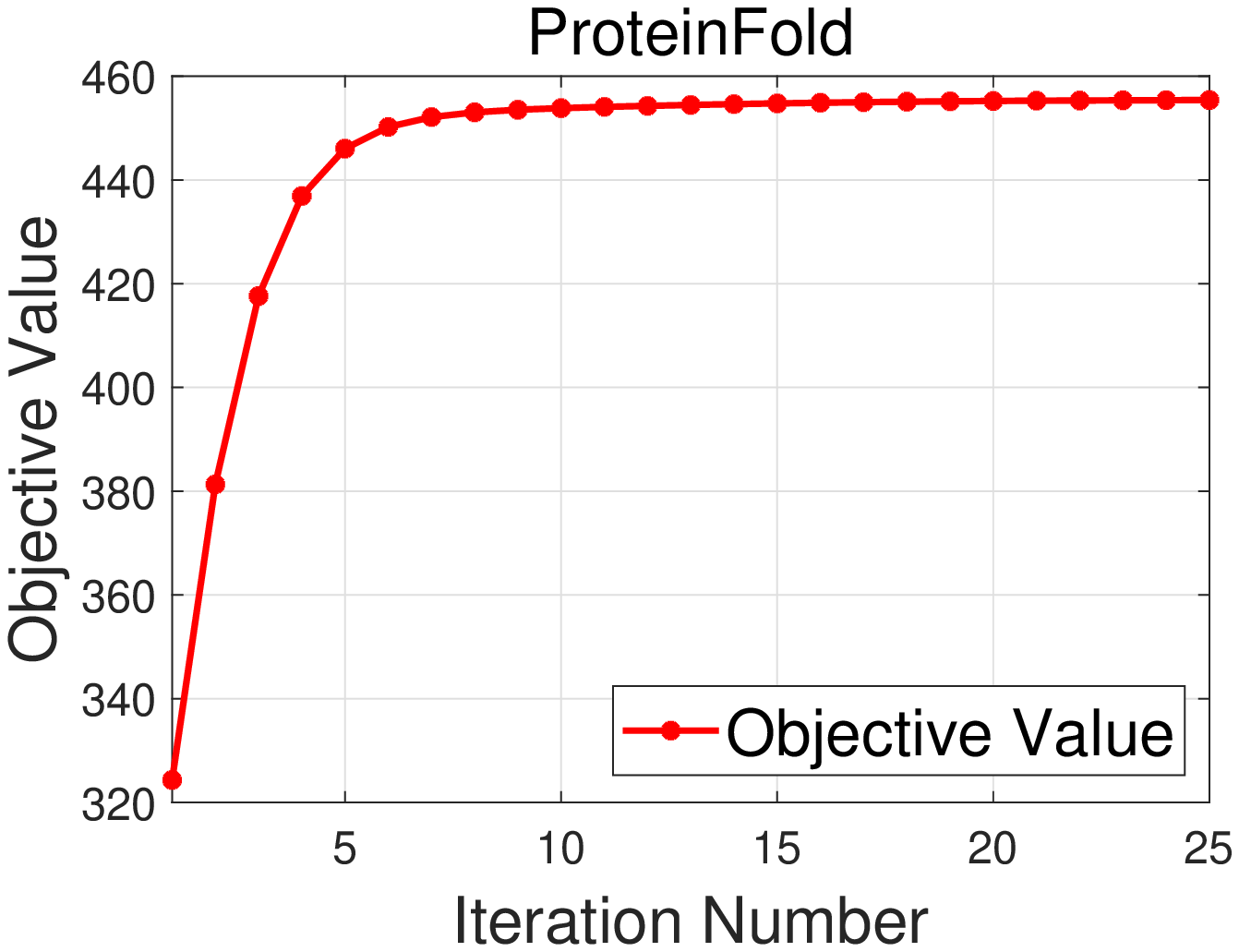}\label{proteinfold_loss}}
	\subfigure[Flower17]{\includegraphics[width=0.47\columnwidth]{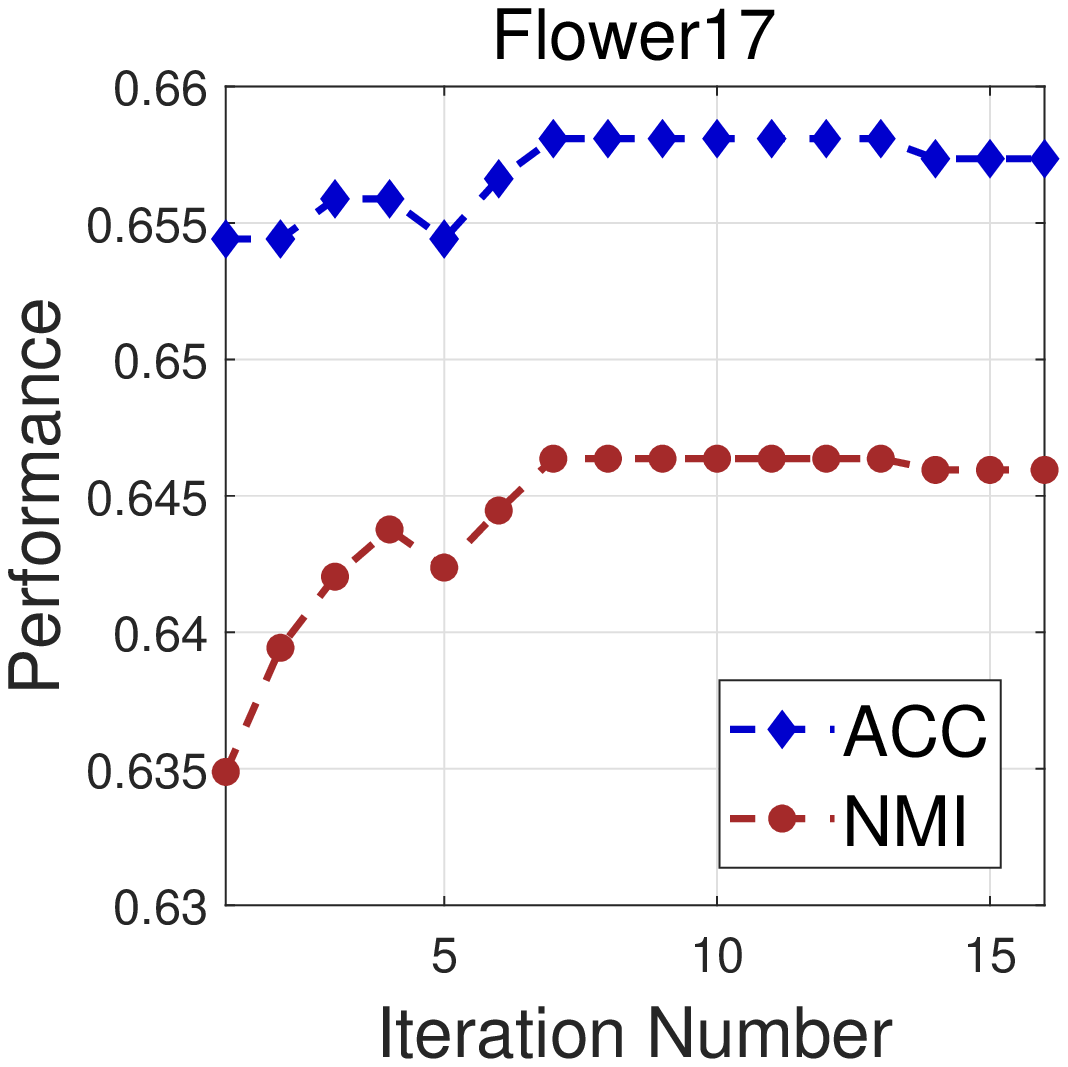}\label{flower17_result}}
	\subfigure[Flower17]{\includegraphics[width=0.47\columnwidth]{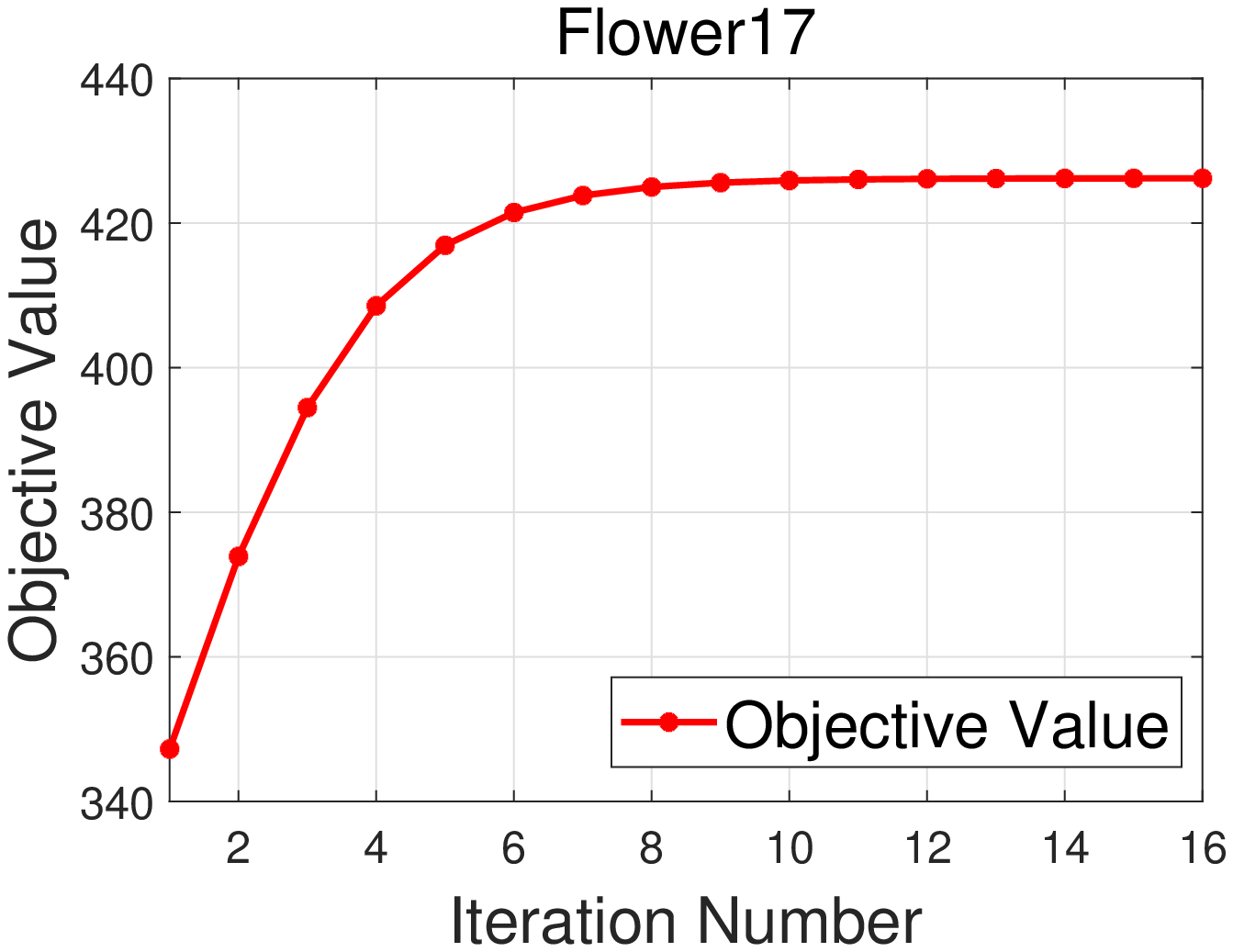}\label{flower17_loss}}
	\subfigure[UCI-Digit]{\includegraphics[width=0.47\columnwidth]{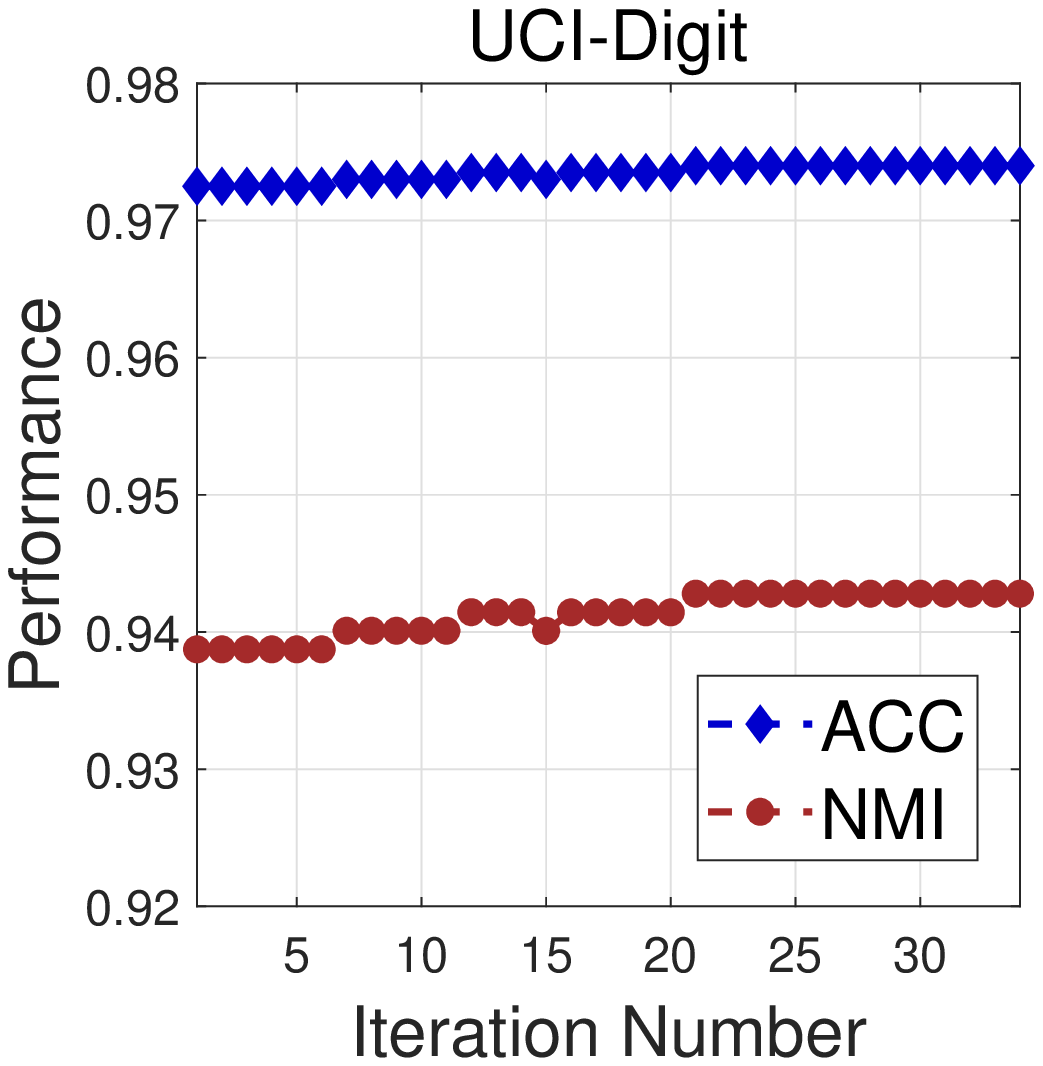}\label{UCI_DIGIT_result}}
	\subfigure[UCI-Digit]{\includegraphics[width=0.47\columnwidth]{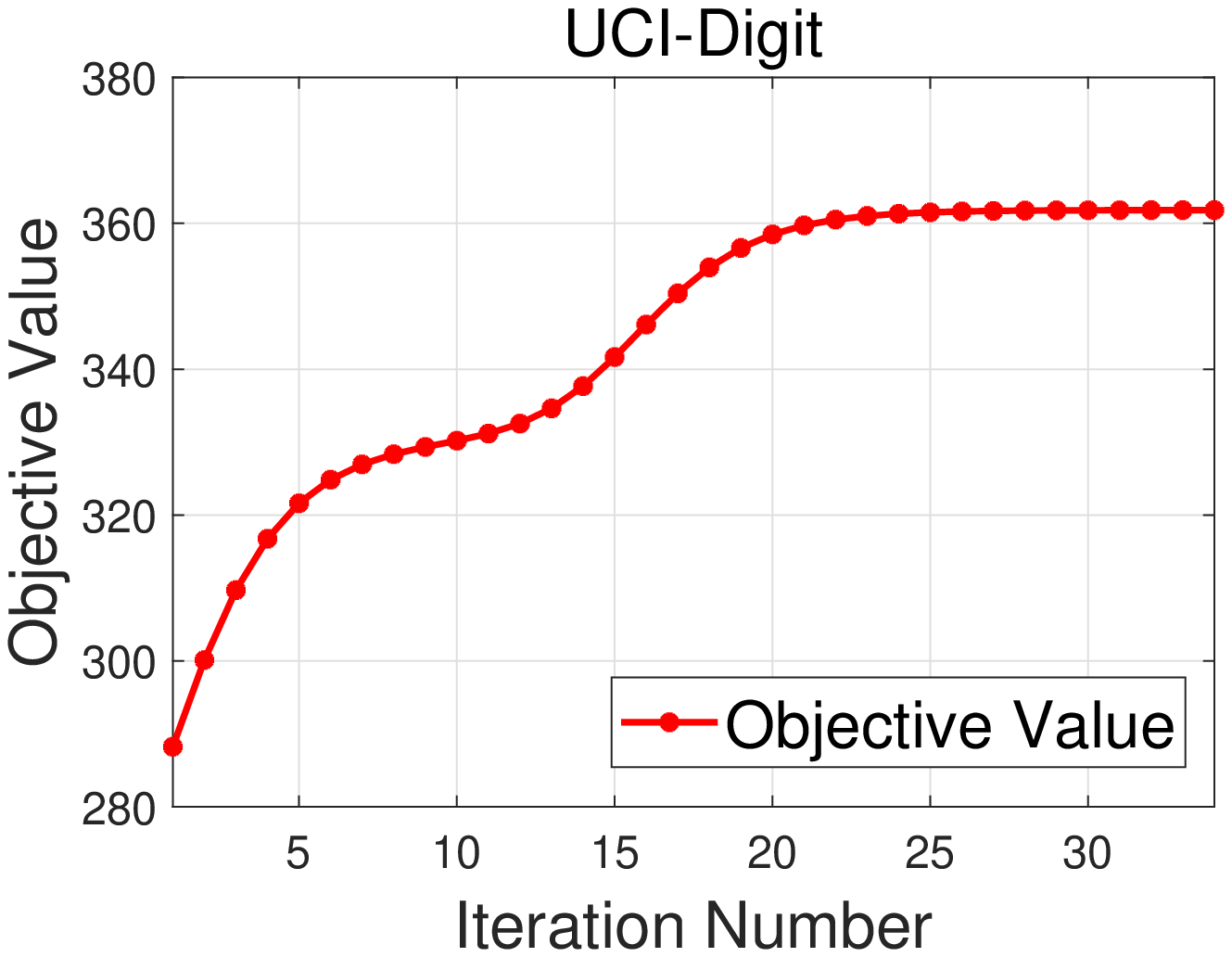}\label{UCI_DIGIT_loss}}
	\subfigure[Flower102]{\includegraphics[width=0.47\columnwidth]{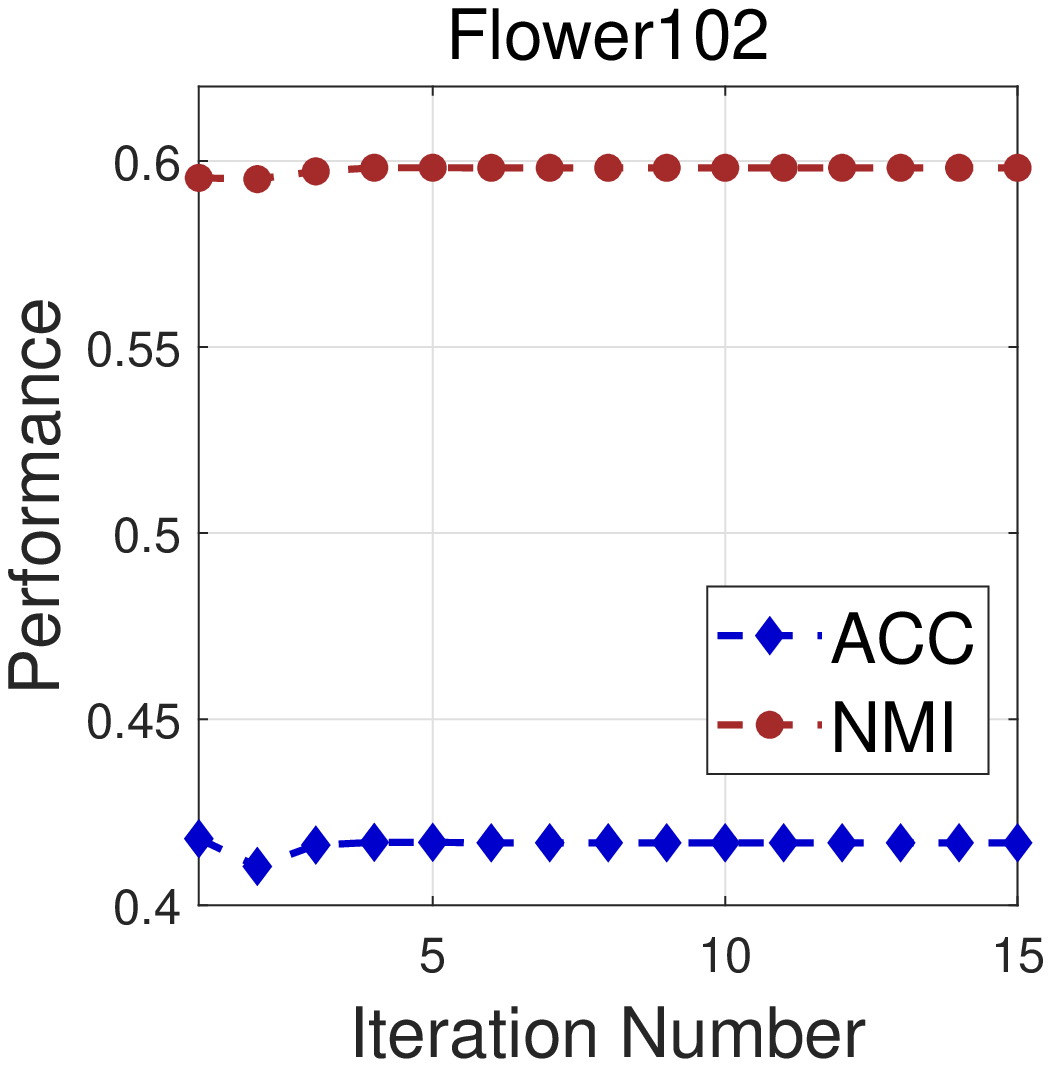}\label{flower102_result}}
	\subfigure[Flower102]{\includegraphics[width=0.47\columnwidth]{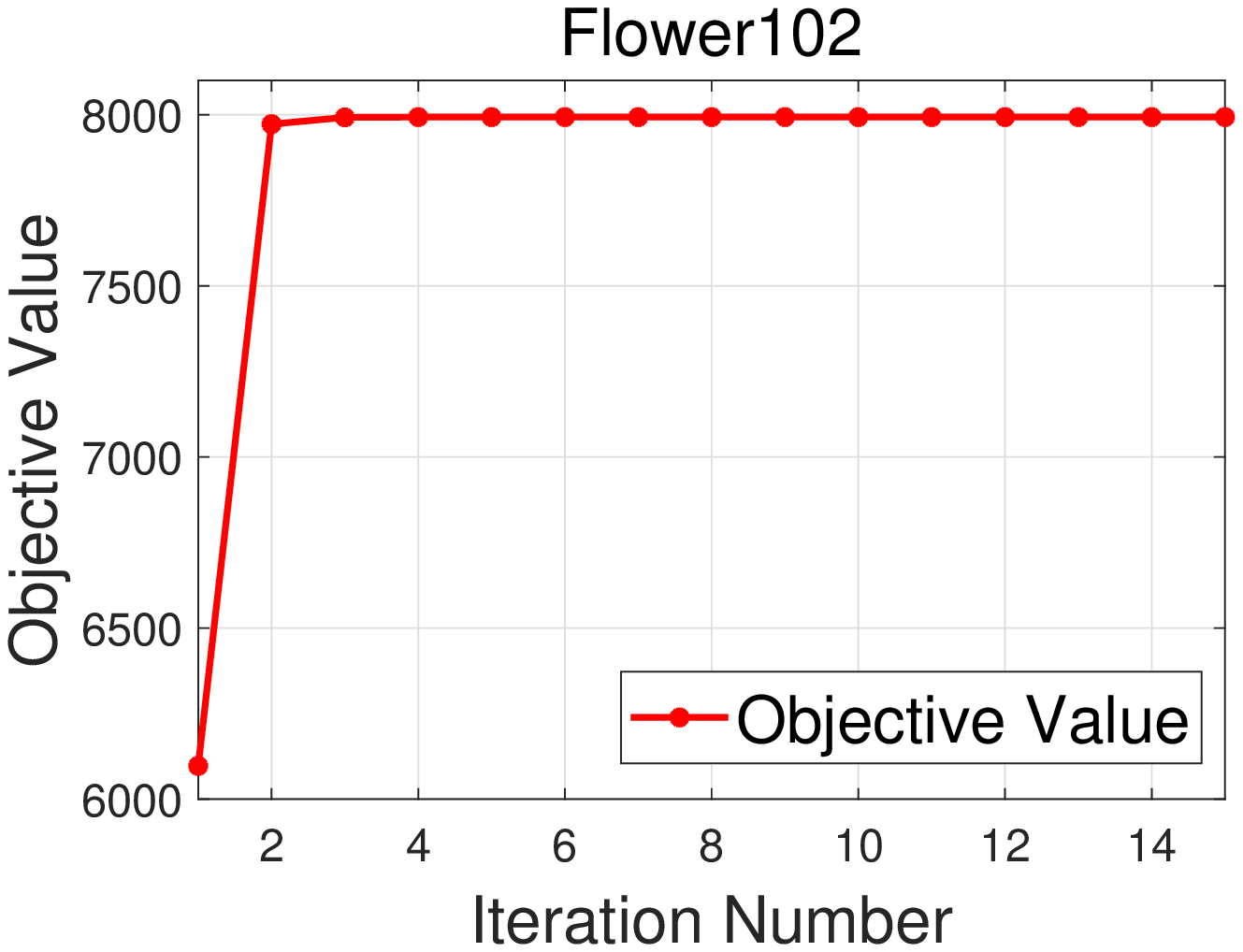}\label{flower102_loss}}
	\subfigure[YALE]{\includegraphics[width=0.47\columnwidth]{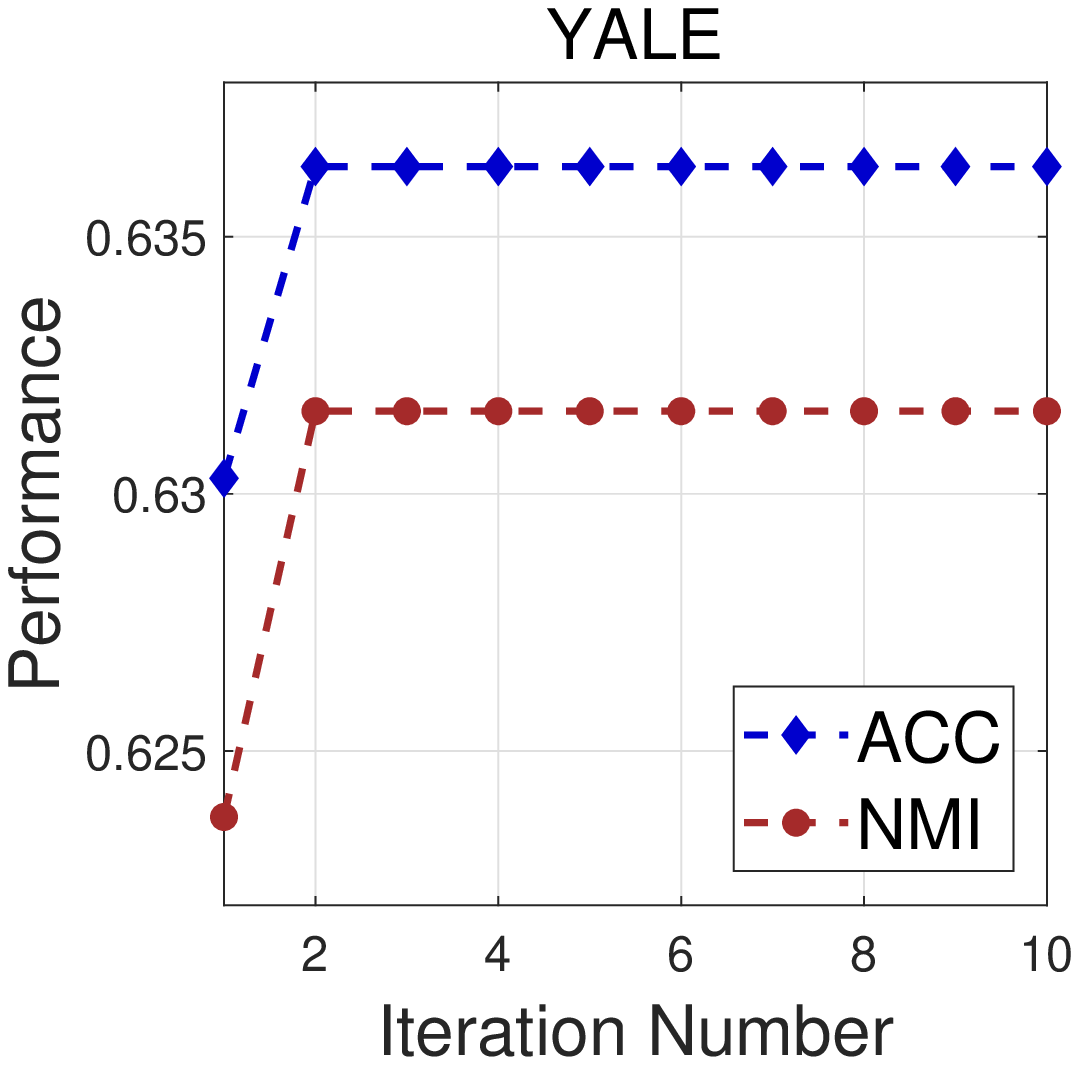}\label{YALE_result}}
	\subfigure[YALE]{\includegraphics[width=0.47\columnwidth]{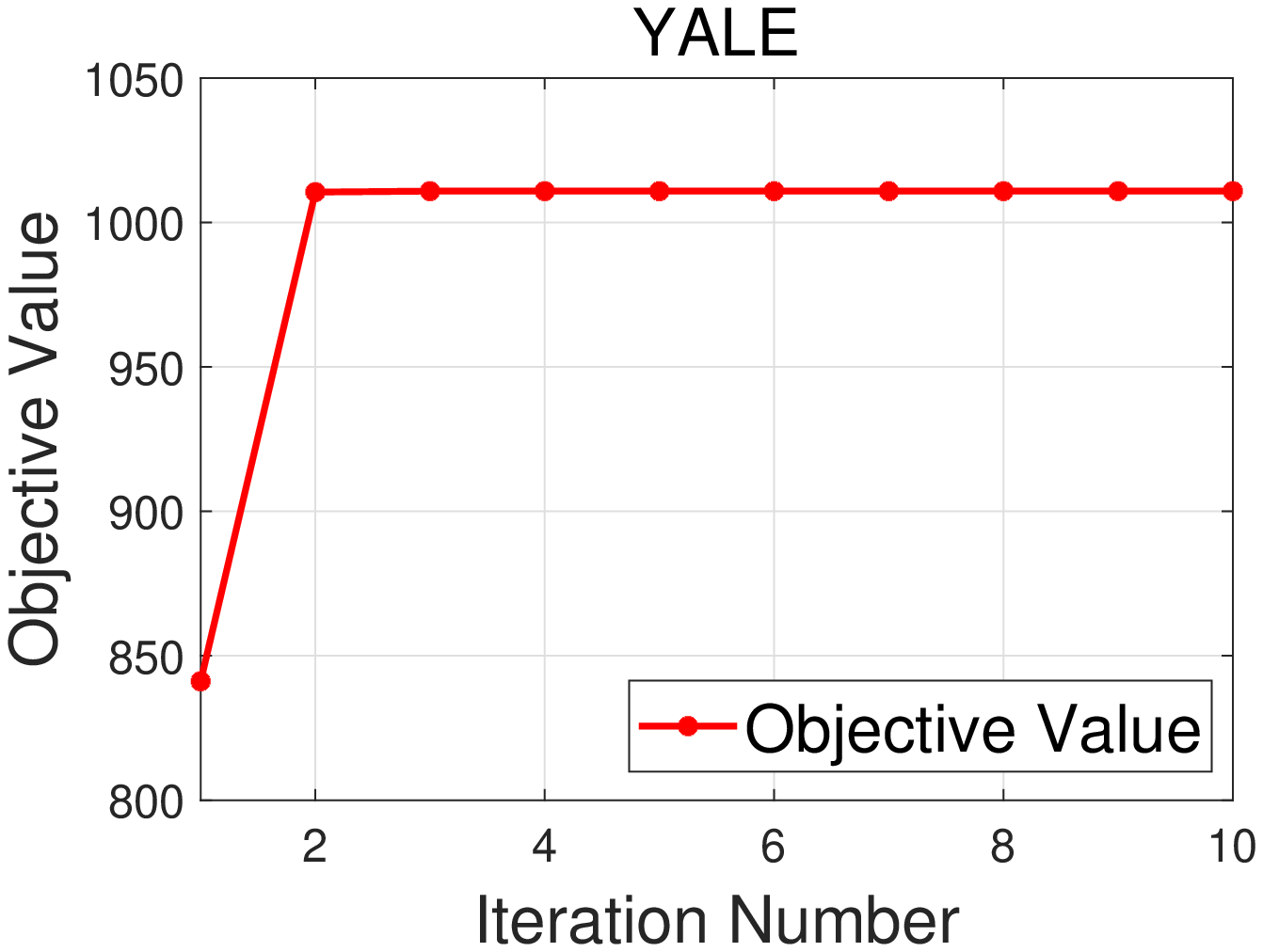}\label{YALE_loss}}
	\caption{Illustration of performance variation and algorithm convergence. In these figures, the first and the third columns are corresponding to the performance evolution as the iteration increases. The blue cures are corresponding to the clustering accuracy and the brown curves are corresponding to the NMI. The second and the fourth columns represent the variation of the objective values. Results on BBCSport, ProteinFold, Flower17, UCI-Digit, Flower102 and YALE datasets are reported.}\label{algorithm_convergence}
	\vspace{-20pt}
\end{figure*}

\textbf{Algorithm Convergence}. Six examples of the objective values of our algorithm at each iteration are shown in Fig. \ref{algorithm_convergence}. As observed from a) to l), the objective function monotonically increases and reaches convergence quickly. In a), g) and k), we can find that the clustering performance reaches the best results, as the objective function value reaches convergence. Because the best value of the objective function may not lead to the best clustering performance, the results in c), e), and i) have a slight variation but gradually increases with the increasing of iterations. In general, all the figures above have clearly demonstrated the effectiveness of the learned consensus
matrix $\H^*$. Furthermore, as shown in other figures, the objective function does monotonically increase at each iteration, and it usually converges within 30 iterations.

\subsection{Scale the Algorithm to Large Datasets}
Being able to work appropriately on large scale datasets is an essential criterion to the practicality of a MVSC algorithm. To show the effectiveness of our proposed method, we further conduct an experiment on the MNIST dataset\footnote{http://yann.lecun.com/exdb/mnist/}. To construct the dataset, we first adopt three deep neural networks, i.e., VGG19~\cite{simonyan2014very}, DenseNet121~\cite{huang2017densely}, and ResNet101~\cite{he2016deep}, which are pre-trained on the ImageNet\footnote{http://www.image-net.org/} dataset as feature extractors in three different views. Then, we conduct Algorithm \ref{alg:algorithm2} and Algorithm \ref{alg:algorithm3} to obtain the cluster indicating matrix of each order and of each view. In this step, to test the influence of the anchor number of the Nystr\"om algorithm we tune the sampled anchor points number from [10 50 100 500 1000 6000], and report the corresponding results in Fig. \ref{MNIST}. 
We also visualize the clustering structure of four representative digits, i.e., 1, 3, 5, 9. From Fig. \ref{MNIST}-(a) we observe that as the number of sampled columns increases, the performance becomes better and quickly reaches a plateau at the number of 50 anchor points. This result indicates that our algorithm can achieve favourable performance with efficient approximation of the cluster indicating matrices. The best ACC, NMI, and purity are {\bf94.09}, {\bf89.9}, and {\bf94.09}, respectively. From Fig. \ref{MNIST}-(b) we can obviously see that the cluster structure represented by our algorithm well reveal the underlying manifold of the corresponding dataset.

In addition, we compare the proposed algorithm with several state-of-the-art large-scale multi-view clustering algorithms, including:
(1) single best $k$-means {\bf (SB-KM)}: which performs standard $k$-means on every single view separately and reports the best performance, (2) robust multi-view $k$-means clustering {\bf (RMKMC)} \cite{Cai2013}: which is a robust large-scale multi-view $k$-means clustering algorithm, (3) large-scale multi-view subspace clustering {\bf(LMVSC)} \cite{Kang2020}: which replaces the full reconstruction matrix with an anchor-based reconstruction matrix for efficient subspace clustering. The ACC, NMI, purity and runtime of the above algorithms are reported in Table \ref{result_large_scale}. According to the results, our method shows the best performance and competitive running time when compared with the state-of-the-art large-scale multi-view clustering methods.

\begin{figure}[htbp]
	\setlength{\abovecaptionskip}{0pt}
	\setlength{\belowcaptionskip}{0pt}
	\centering 
	\subfigure[Performance variation.]{\includegraphics[width=0.49\columnwidth]{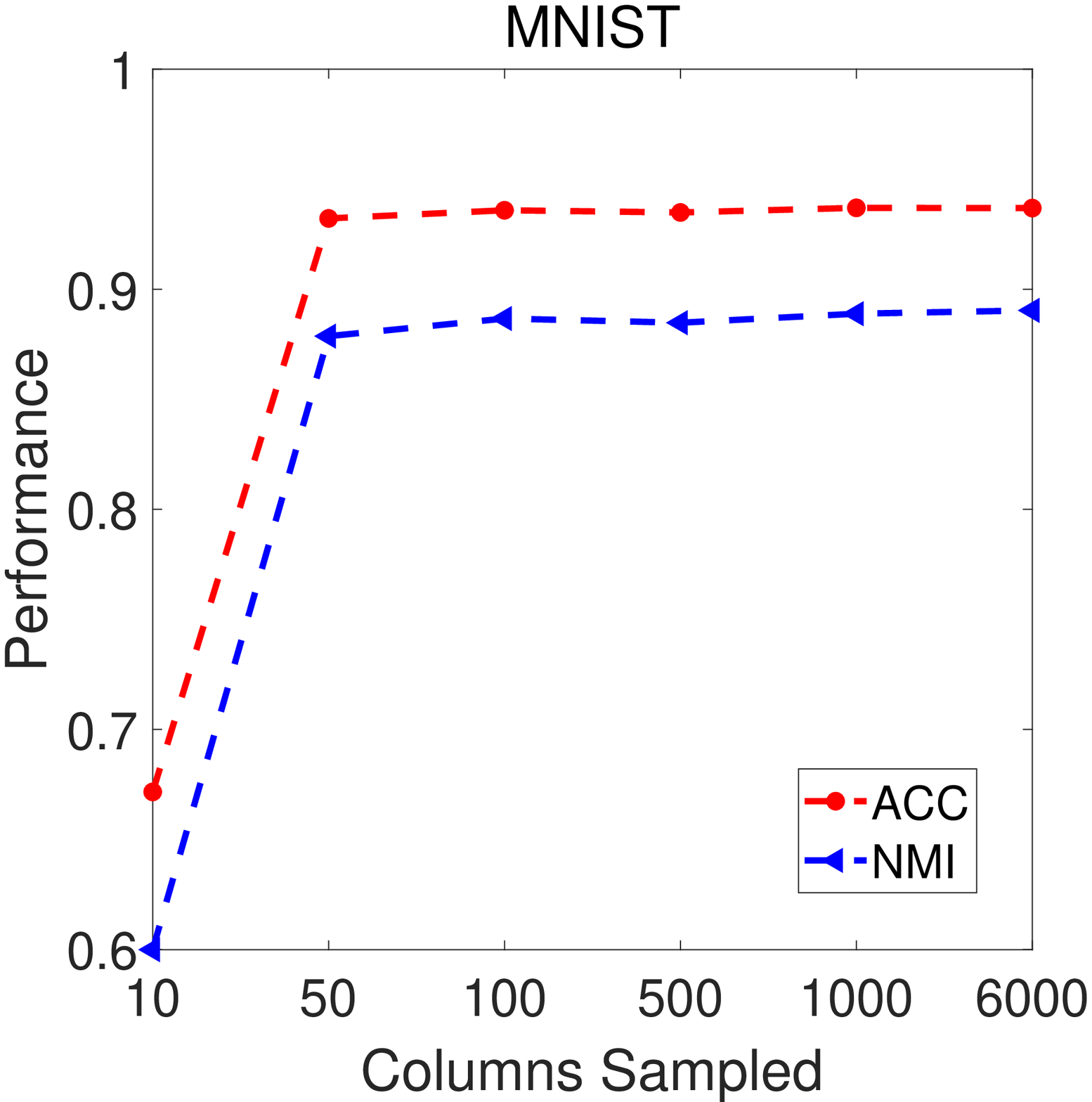}\label{MNIST performance}}
	\subfigure[Clustering results visualization.]{\includegraphics[width=0.49\columnwidth]{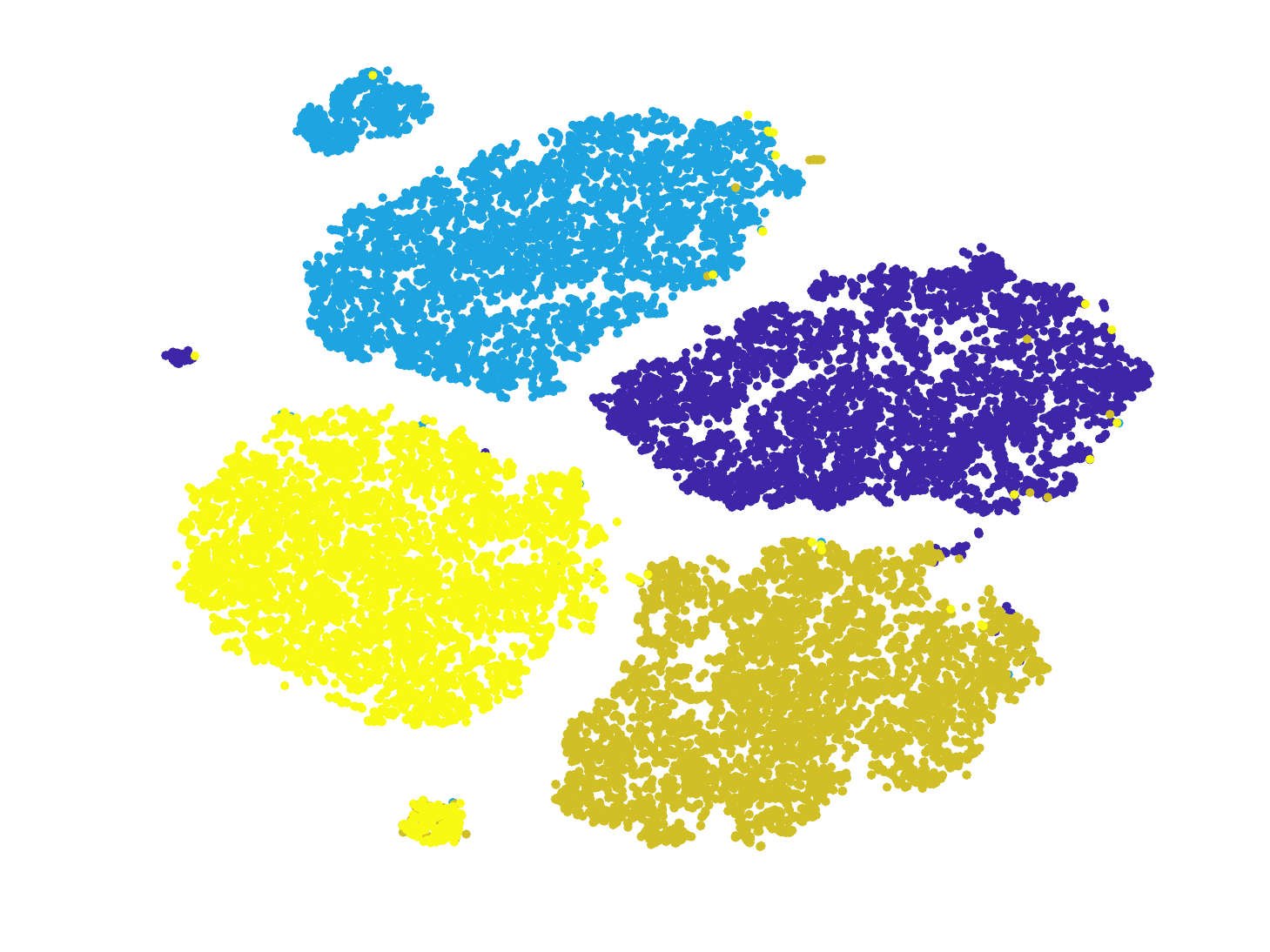}\label{MNIST_visual}}
	\caption{Results illustration on the MNIST dataset. (a) represents the performance variation against the selected anchor point number of the Nystr\"om algorithm. (b) illustrates the intrinsic cluster structure of the proposed algorithm.}
	\label{MNIST}
\end{figure}

\begin{table}[htbp]
	\centering
	\caption{ACC, NMI, purity comparison of different large-scale clustering algorithms on MNIST dataset.} \label{result_large_scale}
	\vspace{-10pt}
	\begin{tabular}{|c|c|c|c|c|}
		\hline 
		& SB-KM & RMKMC & LMVSC & Proposed\tabularnewline
		\hline 
		ACC & 72.84 &72.72  &88.56  &\bf 94.09 \tabularnewline
		\hline 
		NMI & 61.79 &64.91  &77.14  &\bf 89.9\tabularnewline
		\hline 
		Purity & 72.85 & 75.30 &88.58  &\bf 94.09\tabularnewline
		\hline 
		Time(s) & 318.8 &2935.4  & 69.4  & 27.4\tabularnewline
		\hline 
	\end{tabular}
\end{table}

\section{Conclusion} \label{con}
This paper proposes an optimal neighborhood multi-view spectral clustering (ONMSC) algorithm and its late fusion extension. In the proposed algorithms, the early fusion version (ONMSC-EF) enlarges the searching space of optimal Laplacian matrix from the linear combination of the first-order base Laplacian matrices to the neighborhood of both the first-order and high-order Laplacian combinations. The late fusion version (ONMSC-EF) introduces a late fusion learning mechanism and drastically reduces both the computational and storage complexity of the proposed algorithm. As a consequence, the scalability of the proposed algorithm is largely improved. 
A three-step algorithm with proved convergence is designed to solve the resulting optimization problem. Comprehensive experimental results demonstrate the effectiveness and the superior performance of our proposed algorithm. In the future, we plan to extend our algorithm to a more general framework and use it as a platform to revisit existing multi-view spectral clustering algorithms.

\section*{Acknowledgment}
This work was supported by the Natural Science Foundation of China (project no. 61773392 and 61672528).

\bibliographystyle{IEEEtran}
\bibliography{reference}

\vspace{-40pt}
\begin{IEEEbiography}[{\includegraphics[width=1in,height=1.25in,clip,keepaspectratio]{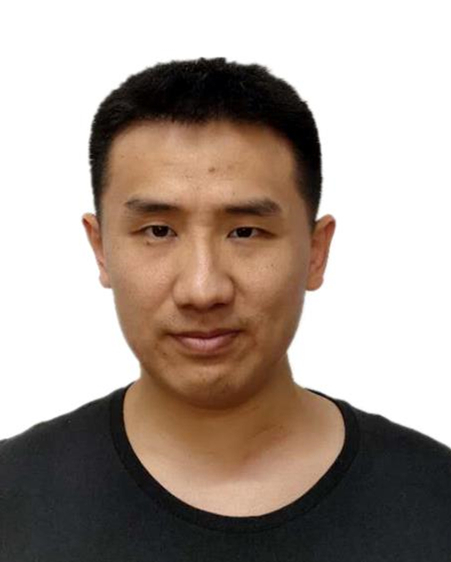}}]{Weixuan Liang} is a graduate student in National University of Defense Technology (NUDT), China. His current research interests include kernel learning, unsupervised multiple-view learning, scalable clustering and deep unsupervised learning.
\end{IEEEbiography}
\vspace{-60pt}

\begin{IEEEbiography}[{\includegraphics[width=1in,height=1.25in,clip,keepaspectratio]{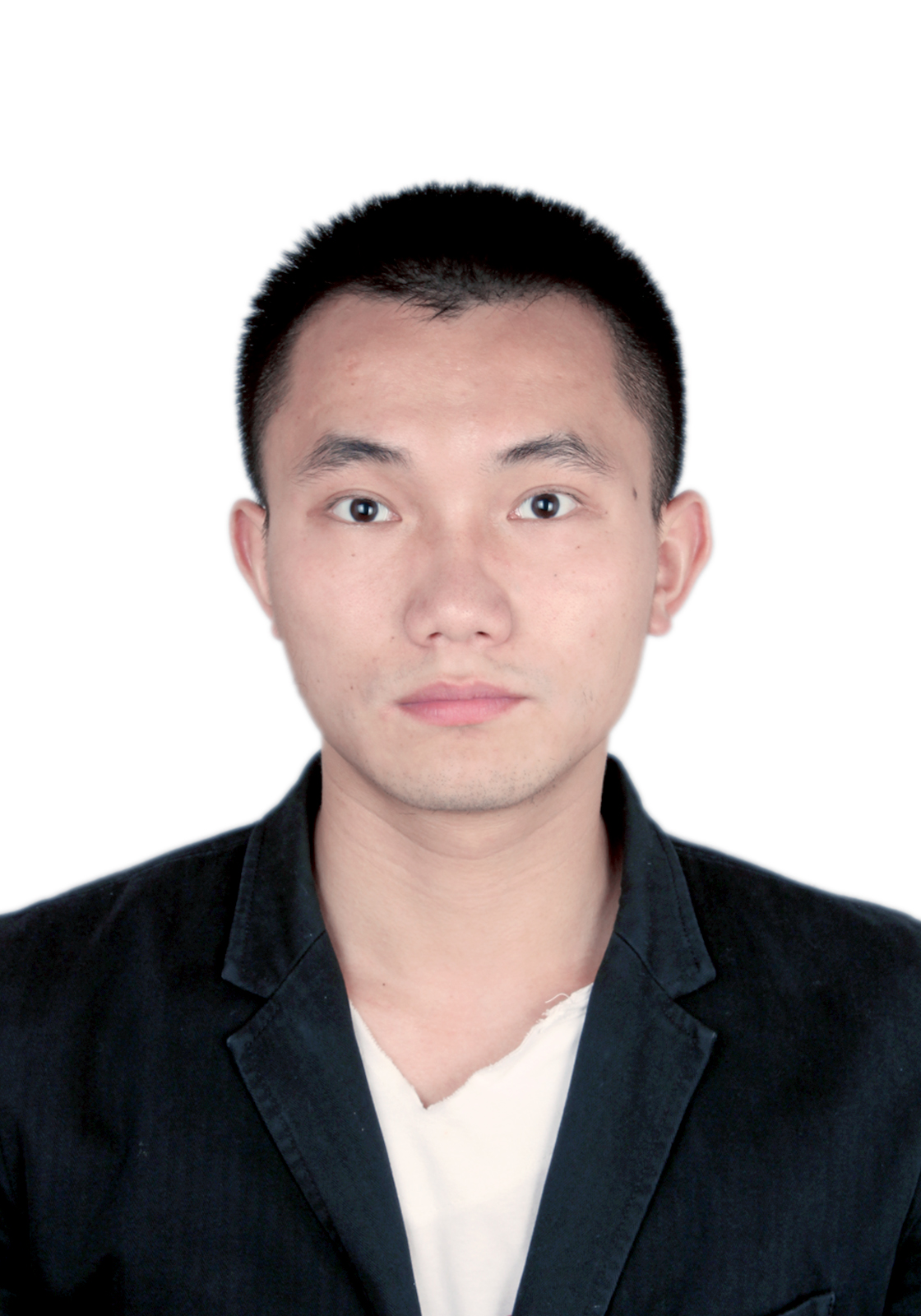}}]
{Sihang Zhou} received his PhD degree in computer science from National University of Defense Technology (NUDT), China in 2019. He received his M.S. degree in computer science from the same school in 2014 and his bachelor's degree in information and computing science from the University of Electronic Science and Technology of China (UESTC) in 2012. He is now a lecturer of College of Intelligence Science and Technology, NUDT. His current research interests include machine learning, pattern recognition, and medical image analysis.
\end{IEEEbiography}
\vspace{-40pt}

\begin{IEEEbiography}[{\includegraphics[width=1in,height=1.25in,clip,keepaspectratio]{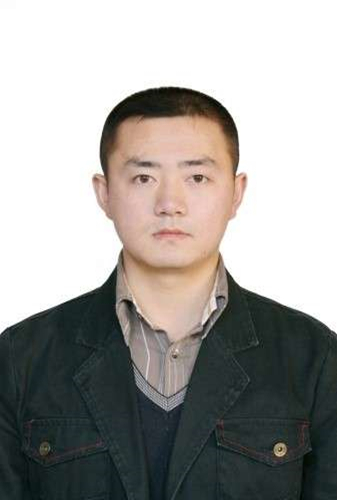}}]{Jian Xiong} received the B.S. degree in engineering, and the M.S. and Ph.D. degrees in management from National University of Defense Technology, Changsha, China, in 2005, 2007, and 2012, respectively. He is an Associate Professor with the School of Business Administration, Southwestern University of Finance and Economics. His research interests include data mining, multiobjective evolutionary optimization, machine learning, multiobjective decision making, project planning, and scheduling. Dr. Xiong has published 20+ peer-reviewed papers, including those in highly regarded journals and conferences such as IEEE T-PAMI, IEEE T-KDE, IEEE T-EC, etc.
\end{IEEEbiography}
\vspace{-40pt}

\begin{IEEEbiography}[{\includegraphics[width=1in,height=1.25in,clip,keepaspectratio]{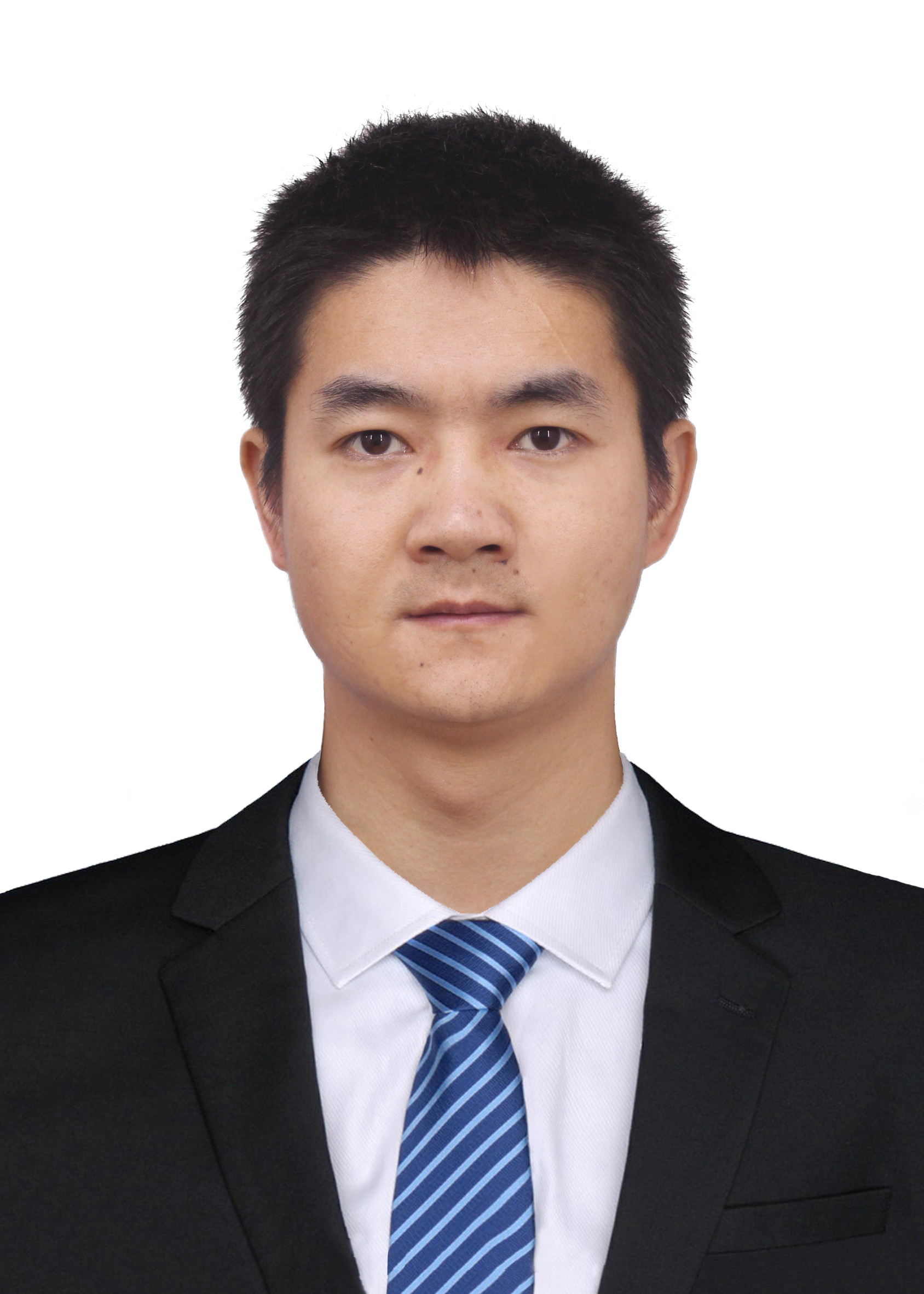}}]{Xinwang Liu} received his PhD degree from National University of Defense Technology (NUDT), China. He is now Assistant Researcher of School of Computer, NUDT. His current research interests include kernel learning and unsupervised feature learning. Dr. Liu has published 60+ peer-reviewed papers, including those in highly regarded journals and conferences such as IEEE T-PAMI, IEEE T-KDE, IEEE T-IP, IEEE T-NNLS, IEEE T-MM, IEEE T-IFS, NeurIPS, ICCV, CVPR, AAAI, IJCAI, etc.
\end{IEEEbiography}
\vspace{-40pt}

\begin{IEEEbiography}[{\includegraphics[width=1in,height=1.25in,clip,keepaspectratio]{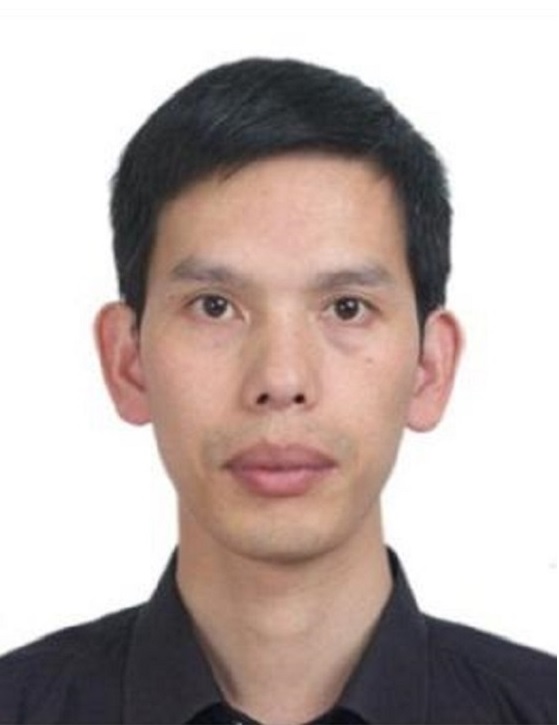}}]{En Zhu} received his PhD degree from National University of Defense Technology (NUDT), China. He is now Professor at School of Computer Science, NUDT, China. His main research interests are pattern recognition, image processing, machine vision and machine learning. Dr. Zhu has published 60+ peer-reviewed papers, including IEEE T-CSVT, IEEE T-NNLS, PR, AAAI, IJCAI, etc. He was awarded China National Excellence Doctoral Dissertation.
\end{IEEEbiography}
\vspace{-30pt}

\begin{IEEEbiography}[{\includegraphics[width=1in,height=1.25in,clip,keepaspectratio]{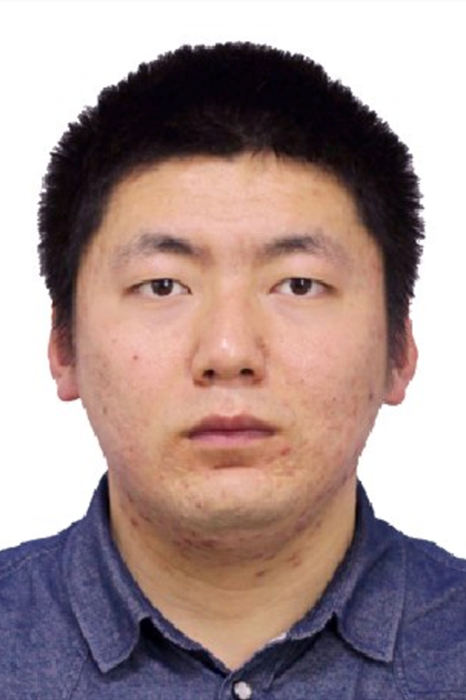}}]
{Siwei Wang} is a graduate student in National University of Defense Technology (NUDT), China. His current research interests include kernel learning, unsupervised multiple-view learning, scalable clustering and deep unsupervised learning. He has published several peer-reviewed papers such as AAAI, IJCAI, etc. He served as reviewers for of AAAI20, IJCAI20 and IEEE TCYB, IEEE TNNLS, etc.
\end{IEEEbiography}
\vspace{-30pt}

\begin{IEEEbiography}[{\includegraphics[width=1in,height=1.25in,clip,keepaspectratio]{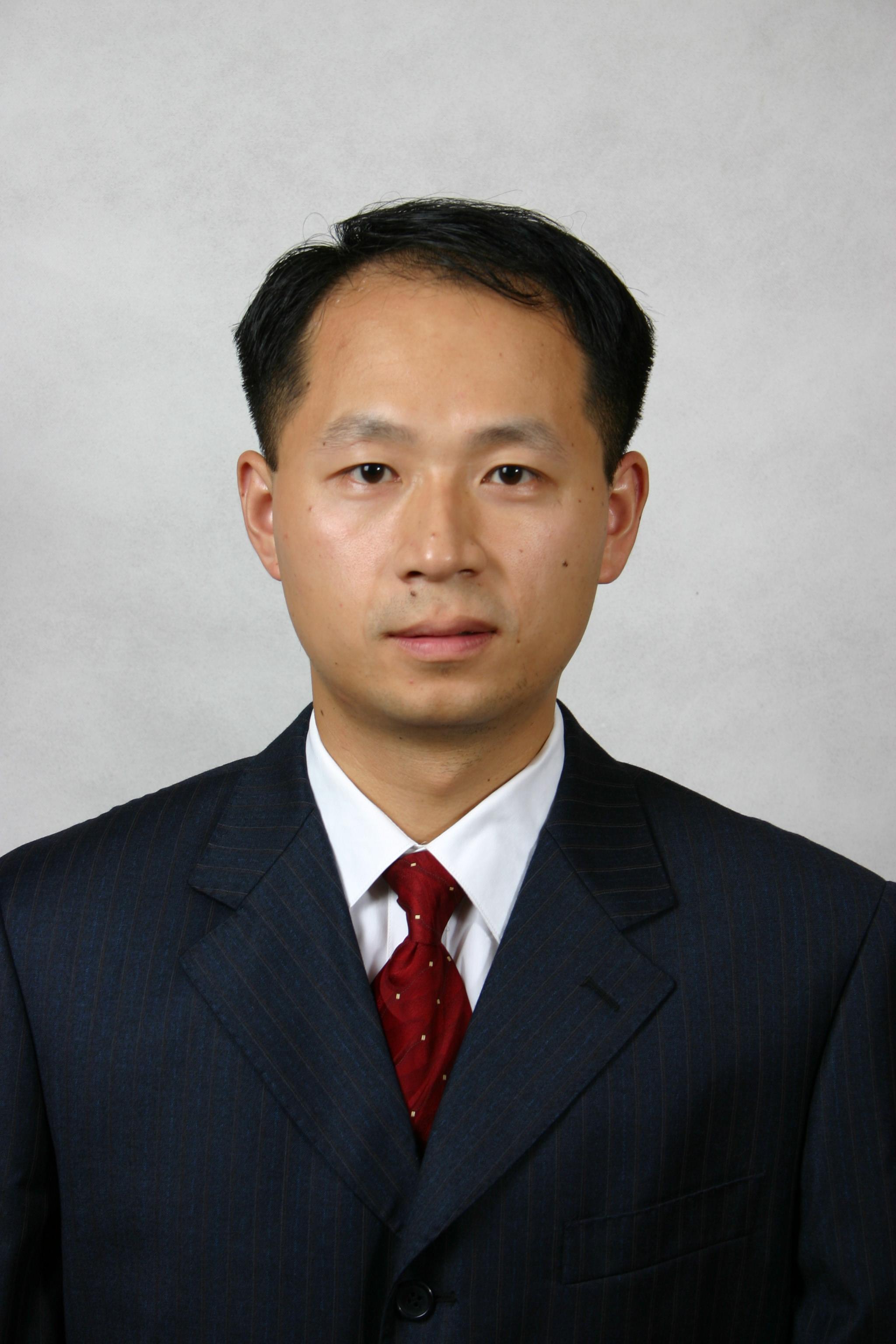}}]{Zhiping Cai}received the B.Eng., M.A.Sc., and Ph.D degrees in computer science and technology from the National University of Defense Technology (NUDT), China, in 1996, 2002, and 2005, respectively. He is a full professor in the College of Computer, NUDT. His current research interests include network security and big data. He is a senior member of the CCF and a member of the IEEE. His doctoral dissertation has been rewarded with the Outstanding Dissertation Award ofthe Chinese PLA.
\end{IEEEbiography}
\vspace{-30pt}

\begin{IEEEbiography}[{\includegraphics[width=1in,height=1.25in,clip,keepaspectratio]{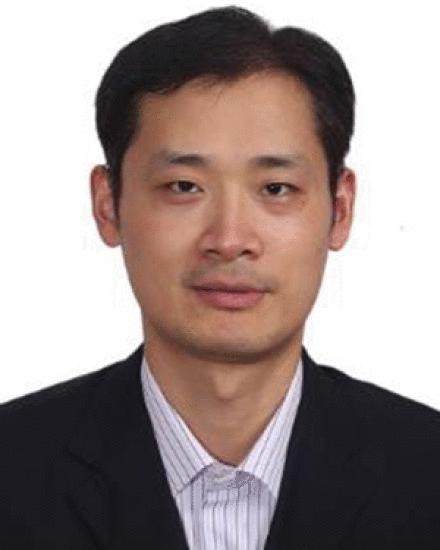}}]{Xin Xu (M'07-SM'12)} received the B.S. degree in electrical engineering from the Department of Automatic Control, National University of Defense Technology (NUDT), Changsha, China, in 1996, and the Ph.D. degree in control science and engineering from the College of Mechatronics and Automation, NUDT, in 2002. He has been a Visiting Professor with The Hong Kong Polytechnic University, the University of Alberta, the University of Guelph, and the University of Strathclyde, U.K. He is currently a Full Professor with the Institute of Unmanned Systems, College of Intelligence Science and Technology, NUDT. He has co-authored more than 160 papers in international journals and conferences and four books. His research interests include intelligent control, reinforcement learning, approximate dynamic programming, machine learning, robotics, and autonomous vehicles. He is a member of the IEEE CIS Technical Committee on Approximate Dynamic Programming and Reinforcement Learning and the IEEE RAS Technical Committee on Robot Learning. He received the National Science Fund for Outstanding Youth in China and the second-class National Natural Science Award of China. He has served as an Associate Editor or Guest Editor for Information Sciences, International Journal of Robotics and Automation, IEEE Transactions on Systems, Man, and Cybernetics: Systems, Intelligent Automation and Soft Computing, the International Journal of Adaptive Control and Signal Processing and Acta Automatica Sinica.
\end{IEEEbiography}

\end{document}